%% file: main.tex
\documentclass[twoside]{article}
\usepackage[margin=.9in]{geometry}
\usepackage{microtype}
\usepackage{graphicx}
\usepackage{subfigure}
\usepackage{booktabs} 
\usepackage{natbib}
\usepackage{url}
\usepackage{algorithm}
\usepackage{algorithmic}
\usepackage{amsmath,amsthm,amsfonts,amssymb}
\usepackage{amsmath}
\usepackage{bbm}
\usepackage{color}
\usepackage{mathrsfs}
\usepackage{enumitem}
\usepackage{bm}
\usepackage{multirow}
\usepackage{booktabs}
\usepackage{makecell}
\usepackage{graphicx}
\usepackage{subfigure}
\usepackage{caption}
\usepackage{thmtools}
\usepackage{times}
\usepackage[dvipsnames]{xcolor}
\include{notation}

\title{\textbf{Parallelizing Contextual Bandits}}
\author{
Jeffrey Chan$\thanks{Equal contribution. Part of this work was conducted at UC Berkeley.}$\\
UC Berkeley\\
\small{\texttt{chanjed@berkeley.edu}}
\and
Aldo Pacchiano$\footnotemark[1]$\\
Microsoft Research\\
 \small{\texttt{apacchiano@microsoft.com}}
\and
Nilesh Tripuraneni$\footnotemark[1]$\\
Google Research\\
 \small{\texttt{nilesh\_tripuraneni@berkeley.edu}} \hspace{-1em}
\and
Yun S. Song\\
UC Berkeley\\
 \small{\texttt{yss@berkeley.edu}}
\and
Peter Bartlett\\
UC Berkeley\\
 \small{\texttt{peter@berkeley.edu}}
\and
Michael I. Jordan\\
UC Berkeley\\
 \small{\texttt{jordan@cs.berkeley.edu}}
}
\date{}

\usepackage{natbib}
\usepackage{graphicx}

\begin{document}

\maketitle

\input{abstract}
\input{intro}

\input{related_work}

\input{prelims}
\input{par_eluder}

\input{par_lin_algs}

\input{lb}

\input{experiments}

\input{conclusion}
\input{acknowledgements}

\appendix
\newpage
\input{app_par_nonlin_algs}
\input{app_par_lin_algs}
\input{app_lb}
\input{app_experiments}

\bibliographystyle{plainnat}
\bibliography{references}
\end{document}

%% file: notation.tex
\newcommand{\deff}{d_{\textsc{eff}}}
\newcommand{\tf}{\tilde{f}}
\newcommand{\dimeluder}[2]{\textsc{dim}_{\textsc{E}}(#1, #2)}
\newcommand{\dimdeeluder}[3]{\textsc{dim}_{\textsc{DE}}(#1, #2, #3)}
\newcommand{\Sigpi}{\mSigma_{\pi}}
\newcommand{\rate}{\textsc{rate}}
\newcommand{\thresh}{\frac{10^5}{P} \frac{L^2 \pi_{\max}^2}{\pi_{\min}^4} \log(\frac{4d}{\delta})}
\newcommand{\cG}{\mathcal{G}}
\newcommand{\tV}{\tilde{\V}}
\newcommand{\tOmega}{\tilde{\Omega}}

\newcommand{\thetastar}{\boldsymbol{\theta}^{\star}}
\newcommand{\fstar}{f_{\star}}
\newcommand{\falpha}{f_{\alpha}}

\newcommand{\hbtheta}{\hat{\btheta}}
\newcommand{\btheta}{\boldsymbol{\theta}}
\newcommand{\ttheta}{\tilde{\btheta}}

\newcommand{\snr}{\textsc{snr}}
\newcommand{\cP}{\mathcal{P}}

\newcommand{\mR}{\mathbb{R}}

\newcommand{\tlr}{\tilde{r}}

\newcommand{\halpha}{\hat{\balpha}}

\newcommand{\KL}{\text{KL}}
\newcommand{\xstar}{\mathbf{x}_{\star}}
\newcommand{\balpha}{\bm{\alpha}}

\usepackage[colorlinks,citecolor=blue,urlcolor=black,linkcolor=blue,pdfborder={0 0 0}]{hyperref}

\usepackage[capitalize]{cleveref}  
\crefname{nlem}{Lemma}{Lemmas}
\crefname{nprop}{Proposition}{Propositions}
\crefname{ncor}{Corollary}{Corollaries}
\crefname{nthm}{Theorem}{Theorems}
\crefname{exa}{Example}{Examples}
\crefname{assumption}{Assumption}{Assumptions}
\crefname{equation}{}{}

\theoremstyle{plain}
\newtheorem{theorem}{Theorem}

\newtheorem{corollary}{Corollary}
\newtheorem{assumption}{Assumption}
\newtheorem{lemma}{Lemma}
\newtheorem{proposition}[theorem]{Proposition}

\theoremstyle{definition}
\newtheorem{condition}{Condition}
\newtheorem{definition}{Definition}
\newtheorem{example}{Example}
\newtheorem{remark}{Remark}

\newcommand{\vect}[1]{\ensuremath{\mathbf{#1}}}

\newcommand{\mat}[1]{\ensuremath{\mathbf{#1}}}

\newcommand{\argmin}{\mathop{\rm argmin}}
\newcommand{\argmax}{\mathop{\rm argmax}}

\newcommand{\Ind}[1]{\mathbbm{1}{#1}}

\newcommand{\tr}{\mathrm{tr}}
\renewcommand{\det}{\mathrm{det}}

\newcommand{\poly}{\mathrm{poly}}

\newcommand{\abs}[1]{|{#1}|}
\newcommand{\norm}[1]{\|{#1} \|}

\newcommand{\mE}{\mathbb{E}}
\renewcommand{\Pr}{\mathbb{P}}

\newcommand{\tlO}{\tilde{O}}
\newcommand{\tlOmega}{\tilde{\Omega}}

\newcommand{\A}{\mat{A}}
\newcommand{\B}{\mat{B}}
\newcommand{\Bone}{\hat{\mat{B}}}

\newcommand{\I}{\mat{I}}

\newcommand{\V}{\mat{V}}

\newcommand{\W}{\mat{W}}
\newcommand{\X}{\mat{X}}

\newcommand{\mSigma}{\mat{\Sigma}}

\renewcommand{\v}{\vect{v}}

\newcommand{\x}{\vect{x}}
\newcommand{\y}{\vect{y}}

\newcommand{\cC}{\mathcal{C}}
\newcommand{\cD}{\mathcal{D}}
\newcommand{\cE}{\mathcal{E}}
\newcommand{\cA}{\mathcal{A}}
\newcommand{\cF}{\mathcal{F}}

\newcommand{\cN}{\mathcal{N}}

\newcommand{\cH}{\mathcal{H}}

\newcommand{\cX}{\mathcal{X}}
\newcommand{\cR}{\mathcal{R}}

\newcommand{\bv}{\mathbf{v}}

%% file: abstract.tex
\begin{abstract}
    Standard approaches to decision-making under uncertainty focus on sequential exploration of the space of decisions. However, 
    \textit{simultaneously} proposing a batch of decisions, which leverages available resources for parallel experimentation, has the potential to rapidly accelerate exploration. We present a family of (parallel) contextual bandit algorithms applicable to problems with bounded eluder dimension whose regret is nearly identical to their perfectly sequential counterparts---given access to the same total number of oracle queries---up to a lower-order ``burn-in" term. We further show these algorithms can be specialized to the class of linear reward functions where we introduce and analyze several new linear bandit algorithms which explicitly introduce diversity into their action selection.
    Finally, we also present an empirical evaluation of these parallel algorithms in several domains, including materials discovery and biological sequence design problems, to demonstrate the utility of parallelized bandits in practical settings. 
\end{abstract}

%% file: intro.tex
\section{Introduction}
    Traditional approaches to adaptive design and search often require resource-intensive modeling and significant time to execute. As the cost of experimentation has steadily dropped, it is the latter that has become a primary constraint in many domains.  For example, within the field of biology, advances in synthesis and high-throughput sequencing have shifted the focus from top-down mechanistic modeling towards iterative, data-driven search algorithms \citep{baker2018mechanistic}.
    In protein engineering, the directed evolution approach \citep{arnold1998design}---which mimics evolution via iterative rounds of measurement of modified sequences---has been successful in improving therapeutic antibodies \citep{hawkins1992selection} and changing substrate specificity \citep{shaikh2008teaching}. 
    Within the realm of information technology, learning in real-world interactive tasks is often well-suited to adaptive, data-driven search; examples include recommender systems which generate movie suggestions for users on the basis of binary feedback for previously recommended movies \citep{bietti2018contextual}. An important property of many of these examples is that learning must be done via bandit feedback. Given the state of a system (i.e., a particular protein sequence or user) the corresponding value or feedback signal is only observed for that particular state.
    

    In these settings of limited feedback, there is a critical need to use adaptive approaches to  \textit{intelligently} query the design space. 
    Bandit methods and Bayesian optimization have provided such adaptivity in a variety of applications, providing a principled way to navigate the exploration-exploitation tradeoff in the design space. It can still be infeasible, however, to explore these spaces efficiently, in particular as measured by the overall \textit{real time} needed for a sequential algorithm to run.
    
 A clear path forward involves making use of parallelism, in particular the massive parallelism that is provided by modern computing hardware.  Indeed, in many applications of interest, it is often feasible to perform parallel measurements (or batch queries) simultaneously. In protein engineering, batch queries can be as large as $10^6$ sequences while each query takes months to measure \citep{sinai2020primer}. Similarly, large-scale recommender systems can often make multiple, concurrent interactions with different users \citep{agarwal2016making}.
    Parallelism provides a mechanism to utilize further hardware for progress, while not increasing the amount of real time needed to learn. The drawback of making batch queries at a fixed time is that queries within a batch will necessarily be less informed than sequential queries---since their choice cannot benefit from the results of other experiments within the same batch.
    
    It is thus of value to characterize the utility of parallelism in adaptive design and search.  In the current paper, we tackle this problem in a simplified setting that allows us to study some of the fundamental challenges that characterize the problem---that of contextual multi-armed bandits. In particular, we ask: can parallelism provide the same benefit as sequentiality through time in this class of adaptive decision-making problems? Perhaps surprisingly, and in accord with  work in related settings, we show that the answer is often yes. We consider a setting in which $P$ distinct processors/machines  perform simultaneous queries in batches, $\{\x_{t,p} \}_{p=1}^{P}$, of size $P$, over $T$ distinct rounds, to a noisy reward oracle $r_{t,p} \approx f_{\star}(\x_{t, p})  + \epsilon_{t,p}$. Importantly, the reward feedback for all $P$ processors in a given round is only observed after the entire batch of all $P$ actions has been selected. In this setting, we consider the price of parallelism with respect to a single, perfectly sequential agent querying the same reward oracle over $TP$ rounds. Our results show that up to a burn-in term independent of time, the worst-case regret of parallel, contextual bandit algorithms can nearly match their perfectly sequential counterpart. Informally, under the standard bandit setting where 
    $d_{\textsc{eff}}$ is the effective dimension (i.e. the eluder dimension \citep{russo2013eluder}) of the function class and $d$ the dimension of the action-space, our results show that (parallel) variants of optimistic linear bandit algorithms can achieve the following aggregate regret:
\begin{align*}
   \tlO(P \cdot \kappa + \sqrt{d_{\textsc{eff}} d} \sqrt{TP}),
\end{align*}
where $\kappa$ is a term capturing the geometric complexity of the sequence of contexts. For \textit{arbitrary} sequences of contexts we have that $\kappa = \tlO(d_{\textsc{eff}})$, but for sequences of contexts with additional geometric structure we provide sharper instance-dependent bounds for $\kappa$ in the case $f$ is linear. The second term in the former regret matches that of a purely sequential agent playing for $TP$ rounds, $\tlO(\sqrt{d_{\textsc{eff}} d} \sqrt{TP})$. The first term is a burn-in term that represents the price of parallelism; it is subleading whenever $T \gtrsim \tlOmega(P \kappa^2/(d d_{\textsc{eff}}))$. However, we note for many applications it will be the case that $P \gg T$ in the regimes of interest---so understanding the scaling of $\kappa$ is an important question. Our work is motivated by the class of  design problems in which \textit{statistical sample complexity} is the primary object of interest.  In between batch queries at different time steps, the cost of inter-processor communication and computation is often negligible. In this setting we make the following contributions:
\begin{itemize}
    \item We introduce a family of parallelized bandit algorithms applicable in the general setting of contextual bandits (which encompasses potentially changing and infinite action sets), building on the UCB algorithm
    which is applicable to general, nonlinear reward functions that have bounded eluder dimension \citep{russo2013eluder}.
    \item We present a unified treatment of several classic, optimistic algorithms, such as UCB and Thompson sampling, in the specialized case the reward function is approximately linear which allows a fine-grained understanding of the interaction of the scales of the underlying parameters, model misspecification, parallelism, and geometry of the context set. In particular, our analysis highlights the scale of the burn-in term that is proportional to $P$. Although asymptotically negligible as $T \to \infty$, this term can be significant in the regime where we have $TP \gg \poly(d)$, but $P \gg T$---which is realistic for many real applications.
    \item We synthesize lower bounds on regret for the misspecified, linear bandit problem in the parallel setting, confirming that our algorithms are optimal up to logarithmic factors as the time horizon tends to $T \to \infty$. 
    \item We present a comprehensive experimental evaluation of our parallel algorithms on a suite of synthetic and real  problems. Through our theory and experiments, we show the importance of 
    explicitly introducing diversity into the action selections of parallel bandit algorithms since it often leads to practical performance gains.
\end{itemize}

%% file: related_work.tex
\subsection{Related Work}

 Parallel sequential decision-making problems often occur in high-throughput experimental design for synthetic biology applications. For protein engineering, \cite{romero2013navigating} employ a batch-mode Gaussian Process Upper Confidence Bound(GP-UCB) algorithm, while \cite{belanger2019} use the batched Bayesian optimization via parallelized Thompson sampling algorithm proposed in \cite{kandasamy2018parallelised}. \cite{angermueller2020population} propose a portfolio optimization method layered atop an ensemble of optimizers. Finally, \cite{sinai2020adalead} utilize a heuristic, batched greedy evolutionary search algorithm for design. 
 
Contextual linear bandit algorithms are also often used in dynamic settings on large, multi-user platforms for problems such as ad placement, web search, and recommendation \citep{swaminathan2015counterfactual, bietti2018contextual}. As important as such domains are in real-world applications, systematic empirical research has lagged, however, due to the paucity of publicly available nonstationary data sources.

Several investigations of the utility of parallelism in sequential decision-making problems have been conducted in the framework of Gaussian process (GP)-based Bayesian optimization. \citet{desautels2014parallelizing} study a Bayesian framework and show that a lazy GP-bandit algorithm initialized with uncertainty sampling can achieve a parallel regret nearly matching the corresponding sequential regret, up to a ``burn-in" term independent of time. Later, \citet{kathuria2016batched} showed that diversity induced by determinantal point processes (DPP) can induce additional, useful exploration in batched/parallel Bayesian optimization in a comparable setting. \citet{kandasamy2018parallelised} establish regret bounds for Bayesian optimization parallelized via Thompson sampling algorithm, obtaining qualitatively similar theoretical results to \citet{kathuria2016batched, desautels2014parallelizing}. 
    
A related line of work studies bandit learning under delayed reward feedback. \citet{chapelle2011empirical} provide an early, empirical investigation of Thompson sampling under fixed, delayed feedback. Several theoretical works establish regret bounds for bandit algorithms under known and unknown \textit{stochastic} delayed feedback models~\citep{joulani2013online,vernade2017stochastic,vernade2020linear,zhou2019learning}. Here the reward information is delayed by a random time interval from the time the action is proposed, which is distinct from our setting in which the batching of the queries results in a fixed time delay. In the setting of finite-armed contextual bandits~\cite{dudik2011efficient}, Delayed Policy Elimination algorithm satisfies a regret bound of the form $\tilde{\mathcal{O}}(\sqrt{m}(\sqrt{T}+\tau))$, where $m$ denotes the number of actions, $T$ the problem horizon and $\tau$ the delay. However, these guarantees require i.i.d.\ stochastically generated contexts and require access to a cost-sensitive classification oracle for their elimination-based protocol.

A more closely related line of work studies distributed bandit learning under various models of limited communication between agents (which can make decisions in parallel). \citet{hillel2013distributed} and \citet{szorenyi2013gossip} study the regret of distributed arm-selection algorithms under restricted models of communication between parallel agents in the setting of multi-armed bandits. \citet{korda2016distributed} and \citet{wang2019distributed}  provide distributed confidence-ball algorithms for linear bandits in peer-to-peer networks with a focus on limiting communication complexity. However, the algorithms in \citet{korda2016distributed} and \citet{wang2019distributed}  both require \textit{intra-round communication} of rewards, which is incompatible with the batch setting that we study in this work. 

The most closely related line of work to ours falls often falls under the category of batched bandit algorithms -- which mirrors the parallel setting but often with the explicit goal of limiting intra-round communication to the minimum number of communication steps. Several works address this multi-armed batch setting \citep{perchet2016batched, esfandiari2019batched, gao2019batched}. A more closely related line of works address the problem in the setting of linear bandits, where \citet{ruan2021linear}, improving upon the results of \citet{han2020sequential}, provides a near-complete characterization
of the minimum number of policy switches needed for stochastic contextual bandit problem
with linear rewards by using an elimination-based algorithm based on a distributional optimal design. In the setting of adversarial
contexts, they show at least $\Omega((d \log T)/ \log(d \log T))$ policy switches are needed to match the purely sequential rate. These works only focus on the setting of linear rewards. \citet{gu2021batched} extends the NeuralUCB algorithm to the batch setting which handles nonlinear rewards. The key idea here is to use the Neural Tangent Kernel  (NTK) linearization to approximate the reward function as  linear in the overparameterized limit to relate the analysis to the linUCB setting; this results in a regret dependence on the effective dimension of the induced NTK matrix. 

Our work provides a class of algorithms and unified regret analysis for the nonlinear bandit problem with general contexts in the parallel setting. In this setting we aim to design theoretically sound and practically useful algorithms for parallel bandits. Our work differs from this previous work in several ways. The first set of results (\cref{thm:nonlin_regret_bound}) we provide applies to nonlinear reward functions that have bounded eluder dimension -- which leverages a quite different approach then \citet{gu2021batched}. \citet{gu2021batched} relies on the NTK approximation to linearize the nonlinear reward function in the overparameterized limit, while our analysis use a graph matching technique to bound the distributional eluder dimension in the parallel setting\footnote{A technique we believe may also be of independent interest in the reinforcement learning setting.}. Second, we phrase our results in the setting of fixed parallelism $P$ (i.e. equivalent to a fixed batch size $B$ with a pre-specified uniform grid in time for allowed communication rounds in the literature) with no explicit goal of minimizing the rounds of communication -- since we are motivated by classes of problems where $P$ is often apriori fixed. However, in this setting our results for linear reward functions are closely related to the results of \citet{han2020sequential, ruan2021linear}. Here, our regret bounds for ParLinUCB match existing optimality results for fixed batch size schemes with general contexts which can only synchronize at a prespecified grid of times in \citet[Theorem 1]{han2020sequential}\footnote{To identify our results in their work we can take $T \to TP$ and $M = T$. Note their results for upper and lower bounds apply to adversarial contexts and require $T \geq d^2$.}. Our purpose is to demonstrate how these practical linear algorithms, although not theoretically superior to existing linear results, are direct analogs of our general nonlinear results but with the ability to adapt to the geometric structure in the context sets. Our analysis for these algorithms can also be adapted to adaptive batch size scheme as in \citet{gu2021batched} although we do not pursue that direction here (sometimes referred to as the rare policy switch model). Moreover, our results also hold in the slightly more general setting where the reward function misspecified and $\epsilon$-close to a linear reward function which is not addressed in these previous works. Finally, since we are motivated by the practical performance of our algorithms, we also introduce several new linear bandit algorithms. The first set of these algorithms have lazy update rules to explicitly introduce diversity in the selection of contexts which we show is empirically useful, while maintaining equivalent theoretical performance. Similarly, our unified analysis which connects all the linear bandit algorithms provided in this paper, also provides regret bounds for the linear Thompson sampling  (TS) algorithm in the parallel setting which is not addressed in these works. As we show, the TS algorithm has excellent empirical performance as in related settings. Note that in concurrent work \citet{karbasi2021parallelizing}, provides guarantees for batched thompson sampling in the multi-armed setting and also in the contextual bandits when restricted to settings with a (global) finite set of arms using elimination-style algorithms. Our approach uses different techniques to handle the generic, parallel setting with adversarial contexts.

%% file: prelims.tex
\subsection{Preliminaries}

Throughout, we use bold lower-case letters (e.g., $\x$) to refer to vectors and bold upper-case letters to refer to matrices (e.g., $\X$).  The norm $\Vert \cdot \Vert$ appearing on a vector or matrix refers to its $\ell_2$ norm or spectral norm, respectively. We write $\norm{\x}_{\mSigma} = \sqrt{\x^\top \mSigma \x}$ for positive semi-definite matrix $\mSigma$. $\langle \x, \y \rangle$ is the Euclidean inner product. Generically, we will use ``hatted" vectors and matrices (e.g., $\halpha$ and $\Bone$) to refer to (random) estimators of their underlying population quantities. We also use the bracketed notation $[n] = \{1, \hdots, n \}$. We will use $\gtrsim$, $\lesssim$, and $\asymp$ to denote greater than, less than, and equal to up to a universal constant and use $\tlO$ to denote an expression that hides polylogarithmic factors in all problem parameters. Our use of $O$, $\Omega$, and $\Theta$ is otherwise standard. 

Formally, we consider the (parallel) contextual linear bandit setting where at each round $t$, the $p$-th bandit learner receives a context $\cX_{t,p} \subset \mathbb{R}^d$ and a master algorithm $\cA$ commands each learner to select an action $\x_{t, p} \in \cX_{t,p}$ on the basis of all the past observations. Given an (unobserved) function $f(\cdot)$ each learner simultaneously receives a noisily generated reward:
\begin{align}
    r_{t, p} = \fstar(\x_{t,p}) + \xi_{t, p}, \label{eq:nonlin_bandit}
\end{align}
where $\xi_{t,p}$ is an i.i.d. noise process and $\fstar(\x_{t,p})$ is approximately linear (i.e., $f(\x_{t,p}) \approx \x_{t, p}^\top \thetastar$ for some unobserved $\thetastar$).
The goal of the master algorithm/processors is to utilize its access to the sequence of rewards $r_{t,p}$ and joint control over the sequence of action selections $\x_{t, p}$ (which depends on the past sequence to events) to minimize the \textbf{parallel regret}:
\begin{equation*}
    \cR(T, P) = \sum_{p=1}^{P} \left( \sum_{t=1}^{T} \fstar(\x_{t, p}^\star) - \fstar(\x_{t, p}) \right),
\end{equation*}
where $\x_{t,p}^\star = \arg \max_{\x \in \cX_{t, p}} f(\x)$. We also introduce the notion of the \textbf{best regret} across processors:
\begin{equation*}
    \cR_{*}(T) = \min_{p \in [P]}  \left( \sum_{t=1}^{T} \fstar(\x_{t, p}^\star) - \fstar(\x_{t, p}) \right),
\end{equation*}
which captures the performance of the best processor in hindsight. We note the following relationship which follows immediately by definition of the aforementioned regrets:
\begin{remark}
The \textbf{parallel regret} and \textbf{best regret} satisfy
\begin{align*}
    \cR_{*}(T) \leq \frac{\cR(T, P)}{P}.
\end{align*}
\end{remark}
In the special case that there is a fixed (finite) context for all time across all processors (i.e., $\cX_{t,p}=\cX$), as is the case for design problems, the
\textbf{simple regret} is useful:
\begin{equation*}
    \cR_{s}(T,P) =   \fstar(\x^\star) - \fstar(\x_{T+1,1}).
\end{equation*}
The simple regret
captures the suboptimality of a choice $\x_{T+1}$, given by a next-step policy $\pi(\cdot)$ at the $(T+1)$st time against the single best choice $\xstar \in \cX$.
\begin{remark}
    There exists a randomized next-step policy $\pi(\cdot)$ (depending on the sequence of $\x_{t,p}$) at the $(T+1)$st timestep such that the simple regret satisfies
    \begin{align*}
        \mE_{\pi}[\cR_{s}(T,P)] \leq \mE \left[ \frac{\cR(T,P)}{TP} \right],
    \end{align*}
    when there is a fixed global context $\cX$ for all $t \in [T]$, $p \in [P]$\footnote{This follows from a similar reduction from sequential regret to simple regret applicable to multi-armed bandits in \citet[Proposition 33.2]{lattimore_szepesvari_2020}.}.
\end{remark}


For our analysis, we make the following standard assumptions on the bandit problem in \cref{eq:nonlin_bandit}.
\begin{assumption}[Subgaussian Noise]
\label{assump:noise}
    The noise variables $\xi_{t,p}$ are $R$-subgaussian for all $t \in [T]$ and $p \in [P]$.
\end{assumption}

\begin{assumption}[Bounded Functions]
\label{assump:functions}
For all  $\x \in \mathcal{X}_{t, p}$, $t \in [T]$, $p \in [P]$, and all $f \in \cF$ the function values are bounded by a known upper bound:
\begin{equation*}
    \abs{f(\x)} \leq B, \quad \forall \x \in \mathcal{X}_{t,p}, \forall f \in \cF.
\end{equation*}
Further the function $\fstar$ is epsilon-close in the sup-norm to its best-approximant in $\cF$ which we denote $\tf$:
\begin{align}
    \tf = \arg \inf_{f \in \cF} \norm{f-\fstar}_{\infty}
\end{align}
\end{assumption}

\begin{assumption}[Bounded Covariates]
\label{assump:data}
For all contexts $\mathcal{X}_{t, p}$, $t \in [T]$, $p \in [P]$, the actions are norm-bounded by a known upper bound:
\begin{equation*}
    \| \x \| \leq L, \quad \forall \x \in \mathcal{X}_{t,p}.
\end{equation*}
\end{assumption}

\begin{assumption}[Almost-Linear Rewards]
\label{assump:param}
The function $f(\cdot)$ is $\epsilon$-close to linear in that for all contexts $\cX_{t,p}$ and for all $\x \in \cX_{t,p}$, there exists a parameter $\thetastar$ such that
\begin{align*}
    \abs{f(\x)-\x^\top \thetastar} \leq \epsilon .
\end{align*}
Further, this underlying parameter $\thetastar$ satisfies the norm bound:
\begin{equation*}
    \| \thetastar \| \leq S.
\end{equation*}
\end{assumption}
In the context of the above conditions we define the signal-to-noise ratio as:
\begin{equation}
    \snr = \left( \frac{LS}{R} \right)^2,
\end{equation}
in analogy with the classical setting of offline linear regression. Note that while we allow $\epsilon$ to be arbitrary, our guarantees are only non-vacuous when $\epsilon$ is suitably small. Thus $f(\x)$ should be thought of a function that is close to its best approximant in $\cF$. We also note that for the linear algorithms in this paper, our methods can easily be generalized using finite-dimensional feature expansions and random feature approximations \citep{rahimi2007random} to increase the flexibility of our model class. 

Throughout the following sections our regret upper bounds are stated with high probability. That is, we claim $\cR(T,P) \leq \rate$, with probability at least $1-O(\delta)$, where $\rate$ has at most $O(\log(\frac{1}{\delta}))$ dependence on $\delta$. However, under \cref{assump:data,,assump:param}, the total regret can always be trivially bounded as $O(LS TP)$ for example. Thus, our high-probability regret bounds can be easily converted to upper bounds in expectation at the cost of only logarithmic factors by setting $\delta \propto (\frac{1}{TP})^2$.


%% file: par_eluder.tex
\section{Parallelizing Contextual Bandits}

\begin{algorithm}[!bt]
\caption{Parallel UCB}\label{algo:par_nonlin_ucb}
\begin{algorithmic}[1]
\renewcommand{\algorithmicrequire}{\textbf{Input: }}
\renewcommand{\algorithmicensure}{\textbf{Output: }}
\REQUIRE $R, S, L, \epsilon$.
\FOR{$t = 1:T$}
    \STATE Compute $\hat{f}_t$ and $\beta_{t}$.
    \FOR{ $p=1:P$}
    \STATE Given $\cX_{t,p}$, compute $\x_{t, p} \leftarrow \arg \max_{\x \in \cX_{t, p}} \sup_{f \in \cF_{t-1,0}} f(\x)$ for $\cF_{t-1,0}$ as in \cref{eq:conf_set_nonlin}.
    \ENDFOR
    \FOR{$p = 1:P$} 
        \STATE Query $\x_{t,p}$ to receive reward $r_{t,p}$
     \hspace{1em}\smash{$\left.\rule{0pt}{1\baselineskip}\right\}\ \mbox{Executed In Parallel}$}
        \ENDFOR
\ENDFOR
\end{algorithmic}
\end{algorithm} 

\begin{algorithm}[!bt]
\caption{Parallel Lazy UCB}\label{algo:par_nonlin_lazy_ucb}
\begin{algorithmic}[1]
\renewcommand{\algorithmicrequire}{\textbf{Input: }}
\renewcommand{\algorithmicensure}{\textbf{Output: }}
\REQUIRE $R, S, L, \epsilon$.
\FOR{$t = 1:T$}
    \STATE Compute $\hat{f}_t$ and $\beta_{t}$.
    \FOR{ $p=1:P$}
    \STATE Given $\cX_{t,p}$, compute $\x_{t, p} \leftarrow \arg \max_{\x \in \cX_{t, p}} \sup_{f \in \cF_{t,p}} f(\x)$ for $\cF_{t,p}$ as in \cref{eq:conf_set_nonlin}.
    \ENDFOR
    \FOR{$p = 1:P$} 
        \STATE Query $\x_{t,p}$ to receive reward $r_{t,p}$
     \hspace{1em}\smash{$\left.\rule{0pt}{1\baselineskip}\right\}\ \mbox{Executed In Parallel}$}
        \ENDFOR
\ENDFOR
\end{algorithmic}
\end{algorithm} 

Here we consider a (parallel) optimistic, UCB-like algorithm which is applicable to general bandit problems which achieves nearly perfect parallel speed-up asymptotically. Our applies to function classes which have a bounded eluder dimension--which includes linear functions and generalized linear models amongst others \citep{russo2013eluder}. \cref{algo:par_nonlin_ucb,,algo:par_nonlin_lazy_ucb} rely on the \textit{optimism} principle. Both algorithms maintain a confidence set containing the true function $\fstar$ (or a close approximation to it) with high probability, using a least-squares estimate $\hat{f}_t \approx \fstar$. They then select the best action/arm with respect to an \textit{optimistic} estimate of what the true function might be. The critical difference between \cref{algo:par_nonlin_ucb,,algo:par_nonlin_lazy_ucb} is that within the same batch of contexts (i.e. in the loop $1:P$) \cref{algo:par_nonlin_ucb} uses the same rule to pick arms across the different $P$ processors. On the other hand, \cref{algo:par_nonlin_lazy_ucb} explicitly encourages diversity with in a batch by adjusting the empirical norm to include the previously chosen actions within a batch -- which discourages picking similar actions in the next within-batch iteration.

We introduce several useful pieces of notation to understand these algorithms. We first define the
the action-induced empirical 2-norm as $\norm{g}_{2,t,p}^2 = \sum_{a=1}^{t-1} \sum_{b=1}^{P} g(\x_{a,b})^2 + \sum_{b=1}^{p-1} g(\x_{t,b})^2$. We also consider the empirical squared loss $L_{2, t} = \sum_{a=1}^{t-1} \sum_{p=1}^{P} (f(\x_{a,b})-r_{a,b})^2$ and the corresponding least-squares minimizer over some base function class $\cF$ as $\hat{f}_t = \argmin_{f \in \cF_{t}} L_{2, t}$. We then consider the confidence sets:
\begin{align}
\label{eq:conf_set_nonlin}
    \cF_{t,p} = \{ f \in \cF : \norm{f-\hat{f}_t}_{t, p} \leq \sqrt{\beta_{t,p}(\cF, \alpha, \delta,  \epsilon)} \}
\end{align}
where $\beta_{t,p}(\cF, \alpha, \delta,  \epsilon)$ is defined in \cref{eq:nonlin_conf_set_main}.
We also recall the best approximant to the true function $\fstar$ in the function class $\cF$ in the sup-norm $\tf$.

Before introducing our algorithms we formally define the eluder dimension: a notion of complexity relevant to adaptive selection procedures introduced in \citet{russo2013eluder}.
\begin{definition}[Action Independence and Eluder Dimension]
Let $\epsilon > 0$ and $\{\x_i \}_{i=1}^{n} \subset \X$ be a set of actions.
\begin{itemize}
    \item An action $\x$ is $\epsilon$-dependent on $\{\x_i \}_{i=1}^{n}$ if any $f, f' \in \cF$ satisfying $\sqrt{\sum_{i=1}^n (f(\x_i)-f'(\x_i))^2} \leq \epsilon$ also satisfies $\abs{f(\x)-f'(\x)} \leq \epsilon$. An action $\x$ is $\epsilon$-independent of $\{\x_i \}_{i=1}^{n}$ with respect to $\cF$ if $a$ is not $\epsilon$-dependent on $\{\x_i \}_{i=1}^{n}$.
    \item The $\epsilon$-eluder dimension $\dimeluder{\cF}{\epsilon}$ is the length of the longest sequence of elements in $\{ \x_i \}_{i=1}^{n}$ such that for some $\epsilon' \geq \epsilon$, every element is $\epsilon'$-independent of its predecessors.
\end{itemize}
\end{definition}

In fact many of our core technical results also for the generalized notion of complexity referred to as the distributional eluder dimension introduced in \citet{jin2021bellman} which we defer to the Appendix. These results may be of independent interest since they have been used to prove regret bounds for algorithms in the full reinforcement learning setting. 
Defining, 
\begin{align}
\eta_{t}  = [(t-1)P]6 B + (t-1)P  R \sqrt{8 \log(4(tP)^2/\delta)}\end{align}
and explicitly defining the confidence set parameter as,
\begin{align}
\beta_{t,p}(\cF, \alpha_{T,P}^{\cF}, \delta, \epsilon) = 8R^2\log(N(\cF, \alpha_{T,P}^{\cF},\norm{\cdot}_{\infty})/\delta) + 2 \alpha_{T,P}^{\cF} \eta_t + 4(p-1)B^2 + 8\epsilon^2 TP \label{eq:nonlin_conf_set_main}
\end{align}
for $\alpha_{T,P}^{\cF} = \frac{B}{TP}$, we can show that,
\begin{theorem} \label{thm:nonlin_regret_bound}
Under \cref{assump:noise,,assump:functions} and defining $\deff = \dimeluder{\cF}{ \frac{B}{TP}}$, the regret of \cref{algo:par_nonlin_ucb}  satisfies,
\begin{equation}
    \cR(T, P) \leq O(BP\deff + \sqrt{\deff \beta_{T,1} TP}) \label{eq:nonlin_regret}
\end{equation}
with probability at least $1-\delta$
and the regret of  \cref{algo:par_nonlin_lazy_ucb}
satisfies,
\begin{equation*}
    \cR(T, P) \leq O(BP\deff + \sqrt{\deff \beta_{T,P} TP})
\end{equation*}
with probability at least $1-\delta$.
We make several comments on \cref{thm:nonlin_regret_bound}
\end{theorem}
\begin{itemize}
    \item Note that since $\cR(TP, 1) \leq O(\sqrt{\deff \beta_{T,1} TP})$, the second term in \cref{eq:nonlin_regret} matches the optimal regret achievable by a perfectly sequential, optimistic algorithm. \cref{thm:nonlin_regret_bound} shows that the cost of parallelism can be purely accounted in the leading term $O(BP\deff)$ which scales with the eluder dimension.
    \item \cref{thm:nonlin_regret_bound} applies uniformly to all models with bounded eluder dimension and covering numbers--including nonlinear function classes such as generalized linear models. \cref{thm:nonlin_regret_bound} shows linearity is not a barrier to achieving (asymptotically) perfect parallelism in the setting of contextual bandits.
    \item The regret bound for \cref{algo:par_nonlin_lazy_ucb} is weaker then that for \cref{algo:par_nonlin_ucb} due to the $\beta_{T,P}$ vs $\beta_{T,1}$ term, which results from inflating the confidence set to account for the lazy contributions. In the linear case we show via sharper analysis that the translation of \cref{algo:par_nonlin_lazy_ucb} need not incur this extra factor.
    \item The cost of misspecification is high. The regret contribution from the lack of a well-specified model scales with $\epsilon \sqrt{\deff} TP$. We return to this point in the case of the linear model where we show this scaling is in general unimprovable.
\end{itemize}

\cref{thm:nonlin_regret_bound} is holds generically for function class with bounded eluder dimension. Accordingly it can be instantiated in several examples with controlled eluder dimension and covering number using results from the literature \citep{russo2013eluder, osband2014model}. We highlight several of these examples below.
\begin{example}[Finite Sets]
\label{ex:finite}
Consider the case where $\cF$ any functions defined on $\cX$ (where $\cX_{t,p} \subseteq \cX$ for all $t \in [T], p \in [P]$) and $\abs{\cX}=m$. Then for any $\epsilon > 0$ $\dimeluder{\cF}{\epsilon} \leq m$ and $\log N(\cF, \epsilon, \norm{}_{\infty}/\delta) = \log m$. So the regret can be summarized as,
\begin{align*}
    \cR(T,P) \leq \tlO(BPm + R \log m \sqrt{mTP})
\end{align*}
\end{example}

\begin{example}[Linear Models]
\label{ex:lin_mod}
Consider the case where $\cF$ consists of $d$-dimensional linear models $\ttheta^\top \x$ satisfying $\norm{\ttheta} \leq S$ and $\norm{\x} \leq L$. Then since $\dimeluder{\cF}{\frac{B}{TP}} \leq \tlO(d)$ and $\log N(\cF, \frac{B}{TP}, \norm{}_{\infty}/\delta) = \tlO(d)$. So the regret can be summarized as,
\begin{align*}
    \cR(T,P) \leq \tlO(BPd + Rd \sqrt{TP} + \epsilon TP)
\end{align*}
\end{example}

\begin{example}[Quadratic Functions]
Consider the case where $\cF$ consists of $d \times d$-dimensional quadratic function $\x^\top \ttheta \x$ for $\ttheta \in \mR^{d \times d}$ satisfying $\norm{\ttheta}_F \leq S$ and $\norm{\x} \leq L$. Then since $\dimeluder{\cF}{\frac{B}{TP}} \leq \tlO(d^2)$ and $\log N(\cF, \frac{B}{TP}, \norm{}_{\infty}/\delta) = \tlO(d^2)$. So the regret can be summarized as,
\begin{align*}
    \cR(T,P) \leq \tlO(BPd^2 + Rd^2 \sqrt{TP} + \epsilon TP)
\end{align*}
\end{example}

\begin{example}[Generalized Linear Models]
Let $g(\cdot) : \mathbb{R}^d \to \mathbb{R}$ be a function with derivative $g'(\cdot) \in [a, b]$ for $a>0$ which defines the condition number $\kappa = \frac{b}{a}$. If $\cF = \{ g(\ttheta^\top \x) : \norm{\ttheta}_2 \leq S, \norm{\x}_{2} \leq L \}$. Then since $\dimeluder{\cF}{\frac{B}{TP}} \leq \tlO(d \kappa^2)$ and $\log N(\cF, \frac{B}{TP}, \norm{}_{\infty}/\delta) = \tlO(d)$, we also obtain,
\begin{align*}
    \cR(T,P) \leq \tlO(BPd + R \kappa d \sqrt{TP} + \kappa \epsilon TP)
\end{align*}

\end{example}

The general structure of each example is similar, with a burn-in term scaling with $\propto P$ and the eluder dimension, followed by the optimal regret term achievable by a purely sequential agent $\propto \sqrt{TP}$, and the misspecification cost $\propto \epsilon \sqrt{TP}$. To our knowledge, the results of \cref{thm:nonlin_regret_bound} are the first to establish the (asymptotic) optimal regret in the nonlinear setting of contextual bandits with bounded eluder dimension. The analysis, which is deferred to the Appendix, relies on the ability to argue the existence of long, parallel sequences of $\epsilon$-dependent actions/measure which relies on a  simple graph matching theorem. As our results also hold for the distributional eluder dimension, which is useful to the analysis of reinforcement learning (RL) \citet{jin2021bellman}, we believe these techniques may also enable the study of parallelism in the setting of RL.
Next, we specialize our results to the more familiar setting of linear rewards in the following sections where a more geometric understanding is provided.

%% file: par_lin_algs.tex
\section{Parallelizing Linear Bandits}

\begin{algorithm}[!bt]
\caption{Parallel LinUCB}\label{algo:par_linucb}
\begin{algorithmic}[1]
\renewcommand{\algorithmicrequire}{\textbf{Input: }}
\renewcommand{\algorithmicensure}{\textbf{Output: }}
\REQUIRE $P, T, R, S, L, \lambda, \epsilon$, \textsc{DR} Routine.
\FOR{$t = 1:T$}
    \STATE Compute $\hbtheta_{t}$, $\V_{t,1}$, and $\beta_{t}$ as in \cref{eq:conf_ellipse1} and \cref{eq:conf_ellipse2}.
    \FOR{ $p=1:P$}
    \STATE Given $\cX_{t,p}$, compute $\y_{t, p} \leftarrow \arg \max_{\x \in \cX_{t, p}} \max_{\btheta \in \cC_{t,0}} \x^\top \btheta$ for $\cC_{t,0}(\hbtheta_{t}, \V_{t,1}, \beta_t, \epsilon)$ as in \cref{eq:conf_ellipse1}.
    \ENDFOR
    \STATE Compute $\tV_{t+1,1} = \V_{t,1} + \sum_{p=1}^{P} \y_{t,p} \y_{t,p}^\top$.
    \IF{ $\tV_{t+1,1}  \preceq 2 \V_{t,1}$ } \label{algo_linlazyts:if}
        \FOR{$p = 1:P$} 
            \STATE Set $\x_{t,p} \leftarrow \y_{t,p}$ and query $\x_{t,p}$ to receive reward $r_{t,p}$
            \hspace{1em}\smash{$\left.\rule{0pt}{1\baselineskip}\right\}\ \mbox{Executed In Parallel}$}
        \ENDFOR
    \ELSE
        \STATE Set $\{ \x_{t,p} \}_{p=1}^{P} \leftarrow \textsc{DR}(\cF_{t,P})$
        \FOR{$p = 1:P$} 
            \STATE Query $\x_{t,p}$ to receive reward $r_{t,p}$
            \hspace{1em}\smash{$\left.\rule{0pt}{1\baselineskip}\right\}\ \mbox{Executed In Parallel}$}
        \ENDFOR
    \ENDIF
\ENDFOR
\end{algorithmic}
\end{algorithm} 

\begin{algorithm}[!bt]
\caption{Parallel Lazy LinUCB}\label{algo:par_lazy_linucb}
\begin{algorithmic}[1]
\renewcommand{\algorithmicrequire}{\textbf{Input: }}
\renewcommand{\algorithmicensure}{\textbf{Output: }}
\REQUIRE $P, T, R, S, L, \lambda, \epsilon$, \textsc{DR} Routine.
\FOR{$t = 1:T$}
    \STATE Compute $\hbtheta_{t}$, and $\beta_{t}$ as in \cref{eq:conf_ellipse1} and \cref{eq:conf_ellipse2}.
    \FOR{$p = 1:P$} 
        \STATE Compute $\tV_{t, p} = \V_{t,1} + \sum_{k=1}^{p-1} \y_{t,k} \y_{t,k}^\top$.
        \STATE Given $\cX_
        {t,p}$, compute $\y_{t,p} \leftarrow \arg \max_{\x \in \cX_{t,p}} \max_{\btheta \in \cC_{t,p}} \x_{t}^\top \btheta$ for $\cC_{t,p}(\hbtheta_{t}, \tV_{t,p}, \beta_t, \epsilon)$ as in \cref{eq:conf_ellipse1}.
    \ENDFOR
    \STATE Compute $\tV_{t+1,1} = \V_{t,1} + \sum_{p=1}^{P} \y_{t,p} \y_{t,p}^\top$.
    \IF{ $\tV_{t+1,1}  \preceq 2 \V_{t,1}$ } \label{algo_linlazy:if}
        \FOR{$p = 1:P$} 
            \STATE Set $\x_{t,p} \leftarrow \y_{t,p}$ and query $\x_{t,p}$ to receive reward $r_{t,p}$
            \hspace{1em}\smash{$\left.\rule{0pt}{1\baselineskip}\right\}\ \mbox{Executed In Parallel}$}
        \ENDFOR
    \ELSE
        \STATE Set $\{ \x_{t,p} \}_{p=1}^{P} \leftarrow \textsc{DR}(\cF_{t,P})$
        \FOR{$p = 1:P$} 
            \STATE Query $\x_{t,p}$ to receive reward $r_{t,p}$
            \hspace{1em}\smash{$\left.\rule{0pt}{1\baselineskip}\right\}\ \mbox{Executed In Parallel}$}
        \ENDFOR
    \ENDIF
\ENDFOR
\end{algorithmic}
\end{algorithm} 

We now consider several specialized algorithms for (parallel) linear bandits that are direct analogs of \cref{algo:par_nonlin_ucb,,algo:par_nonlin_lazy_ucb}.  These algorithms are also inspired by the \textit{optimism} principle, but can take advantage of the explicit linear structure of the reward function. Such algorithms maintain a running estimate of $\hbtheta_t \approx \thetastar$ and the empirical covariance matrix of queried observations:

$$\V_{t,p} = \lambda \I_d + \sum_{a=1}^{t-1} \sum_{b=1}^{P} \x_{a,b} \x_{a,b}^\top + \sum_{k=1}^{p-1} \x_{t,k}\x_{t,k}^\top, $$
which are used to construct a confidence ellipsoid in parameter space where $\lambda$ is a regularization parameter. The best parameter in this set of plausible parameters---which allows the estimation of a maximum hypothetical reward---is used to optimistically guide exploration in the feature space. We use the notation $\mathcal{F}_{t, p-1}$
to define the filtration of all events that have occurred up until and including the revelation of the context $\cX_{t,p}$. Despite the topical differences in these algorithms, there is a common thread which ties together their analyses in the parallel setting in our framework:\footnote{The choice of two in the upper bound here is arbitrary and can be replaced with any universal constant $c>1$ without changing our results up to constants.}
\begin{samepage}
\begin{condition}
\label{cond:critical_inequality}[Critical Covariance Inequality]
An (estimated) covariance matrix $\V_{t,p} \in \cF_{t, p-1}$ is said to satisfy the \textit{critical covariance inequality} at round $t$ for processor $p$ if
\begin{equation}
\boxed{\V_{t,1} \preceq \V_{t,p} \preceq 2 \V_{t,1}.}
\end{equation}
We refer to any round $t$ for which the aforementioned inequality does not hold for any $p \in \{ 2, \hdots, P \}$ as a \textit{doubling round}\footnote{Note that similar to \citet{gu2021batched} we could also introduce a more refined scheme to adaptively synchronize across rounds by replacing this inequality with an analogous one involving log-determinants if our goal was to minimize the number of policy switches.}. 
\end{condition}
\end{samepage}
For the standard purely sequential setting of linear bandits (i.e., $P=1$), in each interaction round we select $\x_t$, gain the information $r_t$, update our model, and iterate. In a parallel setting, on the other hand, we must jointly select $\x_{t,p}$ for all $p \in [P]$ \textit{before} seeing any additional rewards in round $t$. \cref{cond:critical_inequality} ensures that up a factor of two, there is no direction along which our estimate of the covariance changes too rapidly within a single round. If we imagine unrolling the parallel dimension $p$ sequentially across time (i.e., consider the lexigraphic ordering of pairs $(t,p)$), \cref{cond:critical_inequality} ensures the covariance estimate is  quasi-static intra-round.  The significance of this simple condition is that once we receive reward $r_{t,p}$ for $p \in [P]$ at the end of the round $t$, it is nearly as if we received the reward $r_{t,p}$ immediately after selecting $\x_{t,p}$. We later provide more intuition as how this factors into our analysis and what algorithms satisfy this property.

Finally, in the event a particular round $t$ is a doubling round, we allow our algorithms to call a \textit{doubling round routine}, $\{ \x_{t,p} \}_{p=1}^{P} \leftarrow \textsc{DR}(\cF_{t,P})$, where  $\cF_{t,P}$ is the $\sigma$-algebra containing all information regarding past contexts, rewards, and selected actions. The doubling round routine allows our algorithm to make a different choice of actions instead of the actions $\{ \y_{t,p} \}_{p=1}^{P}$  suggested by the optimistic algorithm if round $t$ is a doubling round. In many cases, this routine can simply be taken to be the \textit{identity} map and non-trivial parallelism gains are still obtained. However, in \cref{sec:rich_context} we provide an example of a nontrivial choice of doubling round routine that can exploit the geometry of the context sets for improved performance. We note that our algorithms fall under the term of the \textit{fixed batch setting} as described in the adaptivity literature -- where synchronization between processors is specified at a predefined set of intervals. 

\subsection{Linear Upper Confidence Bound (UCB) Algorithms}
\label{sec:lin_ucb}

We first show how two natural algorithms, which are parallelized variants of the classic LinUCB algorithm of \citet{abbasi2011improved}, obtain the optimal sequential regret, up to a burn-in term, when they satisfy \cref{cond:critical_inequality}. In the following we use 
\begin{align}
    \hbtheta_t = \argmin_{\btheta} \frac{1}{2} \sum_{a=1}^{t-1} \sum_{b=1}^{P} (r_{a, b}-\x_{a,b}^\top \btheta)^2 + \lambda \norm{\btheta}_2^2 \label{eq:ls_estimator}
\end{align}
to refer to the least-squares estimator using data until round $t$, and
\begin{align}
    & \cC_{t, p}(\hbtheta_{t}, \V_{t,p}, \beta_t, \epsilon) = \{ \btheta : \| \btheta - \hbtheta_t \|_{\V_{t, p}} \leq \sqrt{\beta_t(\delta)} + \sqrt{(t-1)P} \epsilon \} \\
    & \text{ for }  \V_{t,p} = \lambda \I_d + \sum_{a=1}^{t-1} \sum_{b=1}^{P} \x_{a,b} \x_{a,b}^\top + \sum_{k=1}^{p-1} \x_{t,k} \x_{t,k}^\top & \text{where } \label{eq:conf_ellipse1}
    \\
    & \sqrt{\beta_{t}(\delta)} =  R\sqrt{  \log\left( \frac{\det(\V_{t,0})}{\lambda^d \delta^2}\right) } + \sqrt{\lambda} S  \leq  R \sqrt{ d \log \left(\frac{1+ tPL^2/\lambda}{\delta} \right) } + \sqrt{\lambda} S \label{eq:conf_ellipse2}
\end{align}
to refer to a confidence ellipsoid which uses $\hbtheta_t$ as its center, but allows the matrix $\V_{t,p}$ (which includes intra-round updates) to modulate the exploration directions. As we will argue, these confidence ellipsoids satisfy the ``optimism" property in that they contain the unobserved $\thetastar$ with high probability. We note that the confidence ellipsoids also include an addition term proportional to $\epsilon$ to accommodate the nonlinearity of the objective. 

\cref{algo:par_linucb} exploits parallelism in a simple fashion. It finds the best optimistic upper bound on the reward in round $t$ for each context and allocates each of its $P$ parallel resources to querying those arms. Although this seems redundant (if for example all the context sets are equal), this strategy can still provide benefit because when \cref{algo:par_linucb} queries a common arm $\x$, $P$ times, the effective variance in the noise of the observed reward is reduced by a factor of $1/P$. As \cref{thm:regret_par_linucb} shows, even this simple parallelism-enabled noise reduction strategy can provide significant benefit.

\cref{algo:par_lazy_linucb} naturally encourages diversity in its parallel exploration strategy within a given round. Within a round $t$, \cref{algo:par_lazy_linucb} queries new actions $\x_{t,p}$ in the standard LinUCB fashion via an optimistic approach with an important caveat: while the covariance matrix is sequentially updated using the queried actions, a ``stale" mean estimate (with data from the first $t-1$ round) is used in the construction of its corresponding confidence ellipsoid since the rewards $r_{t,p}$ are not available intra-round. This update strategy exploits a key property of the linear regression estimator used to construct the confidence ellipsoid. The covariance matrix used to modulate exploration across directions does \textit{not} depend on the rewards $r_{t,p}$ (although the mean estimate does).\footnote{This is closely related to the fact that the conditional covariance of jointly Gaussian random variable only depends on the covariance of the original matrix.} One pitfall of such an approach is that shrinking the predicted variance in the absence of corresponding observations can lead to an ``overconfident" algorithm which may incorrectly exclude the true parameter $\thetastar$ from its confidence set. We compensate this aggressive exploration strategy by also using a lazy threshold width $\sqrt{2}\beta_t(\delta)$ which is inflated by a small multiplicative factor to allay this effect.

Our main result follows which bounds the regret of both \cref{algo:par_linucb} and \cref{algo:par_lazy_linucb}.

\begin{theorem}
    \label{thm:regret_par_linucb}
    Let \cref{assump:noise,,assump:data,,assump:param} hold and let $\cD_{t}$ denote the event that round $t$ is a doubling round (see \cref{cond:critical_inequality}). Then, for any choice of doubling-round routine $\textsc{DR}$, the regret of both \cref{algo:par_linucb,,algo:par_lazy_linucb} satisfy
    \begin{equation}
    \cR(T, P) \leq O \left( LS P\cdot \sum_{t=1}^T \Ind[\cD_t] \right) + \tlO \left(\sqrt{dTP} \max(R \sqrt{d}+\sqrt{\lambda} S + \sqrt{TP} \epsilon, LS) \right),
    \end{equation}
    with probability at least $1-\delta$.
\end{theorem}
We now make several comments to interpret the result.
\begin{itemize}
    \item The second term in \cref{thm:regret_par_linucb} represents the near-optimal (up to log factors) regret a single processor could achieve in a total of $TP$ rounds interacting with a bandit instance in a purely sequential fashion.\footnote{The standard choice of regularizer in the sequential setting is taken as $\lambda=
L^2$.} The first term in \cref{thm:regret_par_linucb} represents the price of parallelization. \cref{thm:regret_par_linucb} hints at the prospect of obtaining a near-optimal worst-case regret (as $T \to \infty$) if the estimate of the covariance can be stabilized intra-round (so the algorithms do not suffer too many doubling rounds). We show how the \cref{cond:critical_inequality} can be enforced such that the first term in \cref{thm:regret_par_linucb} is \textit{independent} of $T$ and $d$ in the sequel. This can even be done in several cases when the choice of the doubling-round routine \textsc{DR} is taken to be the identity map.
\item \cref{thm:regret_par_linucb} also explicitly represents the scales of noise, covariates and parameters, and model misspecification (i.e., the values in \cref{assump:data,,assump:noise,assump:param}) instead of merely asserting that these quantities are $\Theta(1)$ as is standard in the bandits literature (see \citet{lattimore_szepesvari_2020}). Explicitly representing these quantities allows a fine-grained understanding of the interplay between parallelism and quantities such as \textsc{snr}, which may vary from application to application.
\item The cost of misspecification in the regret is high. An $\epsilon$-level of misspecification\footnote{Note this is only non-vacuous when $\epsilon \ll LS/\sqrt{d}$ since the regret can always be trivially bounded by $O(LS TP)$ under our assumptions.} contributes a linearly-scaling regret of $\epsilon \sqrt{d} TP$ in both $T$ and $P$. Balancing the trade-off between the variance $\tlO(Rd \sqrt{TP})$ and misspecification bias $\tlO(\epsilon \sqrt{d} TP)$ is more nuanced in the setting of sequential learning compared to that of i.i.d. supervised learning. In particular \cref{thm:regret_par_linucb} suggests that especially at long time scales and large levels of parallelism, the errors compounded by using an inflexible feature set can easily overwhelm the reduced exploration needed when regressing in a low-dimensional space. If large values of $P$ are desired the use of a flexible feature expansion (with a higher effective dimension) may be desirable.
\end{itemize}
\cref{sec:stable_cov} provides several sufficient conditions under which the covariance stability can be naturally satisfied. Further, \cref{sec:lb} provides instances that show that in the regime of sufficiently large $TP$ the regret bounds in \cref{thm:regret_par_linucb} are unimprovable.

\subsection{Linear Thompson Sampling (TS) Algorithms}\label{section:lints}

Following  Section~\ref{sec:lin_ucb}, we show that two parallel variants of linear Thompson sampling (TS)~\citep{agrawal2013thompson,agrawal2013further,abeille2017linear} can obtain optimal sequential regret, up to a burn in-term, when they satisfy Condition~\ref{cond:critical_inequality}. Despite being among the oldest bandit methods, and having poorer worst-case theoretical guarantees compared to their deterministic counterparts, Thompson sampling methods often achieve excellent performance in practice \citep{thompson1933likelihood, russo2017tutorial} . 

We use the same notation as in Section~\ref{sec:lin_ucb} and refer to the least-squares estimator using data until round $t$ as $ \hbtheta_t$, and defined as in Equation~\ref{eq:ls_estimator}, to describe our parallel TS variants. Algorithm~\ref{algo:par_lints}  is the corresponding Thompson sampling version of Algorithm~\ref{algo:par_linucb}. During each round $p \in [P]$, \cref{algo:par_lints} samples $P$ independent candidate parameters $\{\ttheta_{t,p}\}_{p=1}^P$
to induce exploration over the parameter set. 
Solving the optimization problems, $\argmax_{\x \in \cX_{t, p}}  \x^\top \ttheta_{t, p}$, over the possibly processor dependent contexts $\cX_{t,p}$, induces $p$ distinct arm choices for the different processors $\{\x_{t, p}\}_{p=1}^P$. While Algorithm~\ref{algo:par_lints} does not update the covariance of the sampling distribution while producing each of the $P$ candidate models $\{\ttheta_{t,p}\}_{p=1}^P$, \cref{algo:par_lazy_lints} mirrors~\cref{algo:par_lazy_linucb}. ~\cref{algo:par_lazy_lints} proceeds by sampling the model parameters sequentially across the different processors, but updates the sampling covariance matrix in between each sampling step intra-round. 

Our main result is the following bound for \cref{algo:par_lints,,algo:par_lazy_lints}:
\begin{theorem}\label{thm:linTS_parallel}
    Let \cref{assump:noise,,assump:data,,assump:param} hold and let $\cD_{t}$ denote the event that round $t$ is a doubling round (see \cref{cond:critical_inequality}). Then, for any choice of doubling-round routine $\textsc{DR}$, the regret of both \cref{algo:par_lints,,algo:par_lazy_lints} satisfy
    \begin{equation}
    \cR(T, P) \leq O \left( LS P \cdot \sum_{t=1}^T \Ind[\cD_t] \right) + \tlO \left(d \sqrt{TP \left(1+\frac{L^2}{\lambda}\right)}\left( R\sqrt{d} + S\sqrt{\lambda} + \sqrt{TP}\epsilon \right)  \right),
    \end{equation}
    with probability at least $1-3\delta$, whenever $\delta \leq \frac{1}{6}$.
\end{theorem}
We make several comments on the result.
\begin{itemize}
    \item \cref{thm:linTS_parallel} has a similar flavor to \cref{thm:regret_par_linucb}---allowing for near perfect parallelization (as $T \to \infty$) relative to the sequential Thompson sampling algorithm when the first term is independent of $T$.\footnote{The standard bound for TS in the sequential setting can be obtained by setting $\lambda=L^2$.} As before, \cref{cond:critical_inequality} can be enforced such that the first term in \textit{independent} of $T$.
    \item Note that even in the sequential setting, the regret of linear Thompson sampling suffers an extra multiplicative $\sqrt{d}$  factor relative to LinUCB  \citep{abeille2017linear}. As \cref{thm:linTS_parallel} shows, the parallel variants of Thompson inherit this $\sqrt{d}$ factor. Despite this extra dimension-dependent  factor (which is needed to maintain the optimism property when using the noisily sampled candidate models for exploration) the performance of the Thompson sampling algorithm is often excellent in practice.
\end{itemize}

\begin{algorithm}[!bt]
\caption{Parallel LinTS}\label{algo:par_lints}
\begin{algorithmic}[1]
\renewcommand{\algorithmicrequire}{\textbf{Input: }}
\renewcommand{\algorithmicensure}{\textbf{Output: }}
\REQUIRE $P, T, R, S, L, \lambda, \epsilon$, \textsc{DR} Routine.
\FOR{$t = 1:T$}
    \STATE Compute $\hbtheta_{t}$, $\V_{t,1}$, and $\beta_{t}$ as in \cref{eq:ls_estimator}, \cref{eq:conf_ellipse1} and \cref{eq:conf_ellipse2}.
    \FOR{$p = 1:P$} 
        \STATE Sample $\boldsymbol{\eta}_{t, p} \sim \mathcal{N}(\mathbf{0}, \mathbb{I}_d )$.
        \STATE Compute parameter $\ttheta_{t,p} = \hbtheta_{t} +  \left(\sqrt{\beta_t} +\sqrt{(t-1)P} \epsilon\right)\V^{-1/2}_{t, 1} \boldsymbol{\eta}_{t, p}$.
        \STATE Given $\cX_{t, p}$, compute $\y_{t, p} \leftarrow \arg \max_{\x \in \cX_{t,p}}  \x^\top \ttheta_{t, p}$.
    \ENDFOR
    \STATE Compute $\tV_{t+1,1} = \V_{t,1} + \sum_{p=1}^{P} \y_{t,p} \y_{t,p}^\top$.
    \IF{ $\tV_{t+1,1}  \preceq 2 \V_{t,1}$ } \label{algo_tslin:if}
        \FOR{$p = 1:P$} 
            \STATE Set $\x_{t,p} \leftarrow \y_{t,p}$ and query $\x_{t,p}$ to receive reward $r_{t,p}$
            \hspace{1em}\smash{$\left.\rule{0pt}{1\baselineskip}\right\}\ \mbox{Executed In Parallel}$}
        \ENDFOR
    \ELSE
        \STATE Set $\{ \x_{t,p} \}_{p=1}^{P} \leftarrow \textsc{DR}(\cF_{t,P})$
        \FOR{$p = 1:P$} 
            \STATE Query $\x_{t,p}$ to receive reward $r_{t,p}$
            \hspace{1em}\smash{$\left.\rule{0pt}{1\baselineskip}\right\}\ \mbox{Executed In Parallel}$}
        \ENDFOR
    \ENDIF
\ENDFOR
\end{algorithmic}
\end{algorithm}

\begin{algorithm}[!bt]
\caption{Parallel Lazy LinTS}\label{algo:par_lazy_lints}
\begin{algorithmic}[1]
\renewcommand{\algorithmicrequire}{\textbf{Input: }}
\renewcommand{\algorithmicensure}{\textbf{Output: }}
\REQUIRE $P, T, R, S, L, \lambda, \epsilon$, \textsc{DR} Routine.
\FOR{$t = 1:T$}
    \STATE Compute $\hbtheta_{t}$, and $\beta_{t}$ as in \cref{eq:conf_ellipse1} and \cref{eq:conf_ellipse2}.
    \FOR{$p = 1:P$} 
         \STATE Compute $\tV_{t, p} = \V_{t,1} + \sum_{k=1}^{p-1} \y_{t,k} \y_{t,k}^\top$.
        \STATE Sample $\boldsymbol{\eta}_{t, p} \sim \mathcal{N}(\mathbf{0}, \mathbb{I}_d )$.
        \STATE Compute parameter $\ttheta_{t,p} = \hbtheta_{t} + \left(\sqrt{ 2\beta_t} +\sqrt{2(t-1)P} \epsilon\right)\V^{-1/2}_{t, p} \boldsymbol{\eta}_{t, p}$.
        \STATE Given $\cX_{t,p}$, compute $\y_{t,p} \leftarrow \arg \max_{\x \in \cX_{t,p}} \x^\top \ttheta_{t,p}$.
    \ENDFOR
\STATE Compute $\tV_{t+1,1} = \V_{t,1} + \sum_{p=1}^{P} \y_{t,p} \y_{t,p}^\top$.
    \IF{ $\tV_{t+1,1}  \preceq 2 \V_{t,1}$ } \label{algo_lin:if}
        \FOR{$p = 1:P$} 
            \STATE Set $\x_{t,p} \leftarrow \y_{t,p}$ and query $\x_{t,p}$ to receive reward $r_{t,p}$
            \hspace{1em}\smash{$\left.\rule{0pt}{1\baselineskip}\right\}\ \mbox{Executed In Parallel}$}
        \ENDFOR
    \ELSE
        \STATE Set $\{ \x_{t,p} \}_{p=1}^{P} \leftarrow \textsc{DR}(\cF_{t,P})$
        \FOR{$p = 1:P$} 
            \STATE Query $\x_{t,p}$ to receive reward $r_{t,p}$
            \hspace{1em}\smash{$\left.\rule{0pt}{1\baselineskip}\right\}\ \mbox{Executed In Parallel}$}
        \ENDFOR
    \ENDIF
\ENDFOR
\end{algorithmic}
\end{algorithm} 

\section{Stable Covariances}
\label{sec:stable_cov}
We demonstrate how \cref{cond:critical_inequality} can be naturally satisfied in a variety of settings. The bounds presented here have a common structure which takes the form:
\begin{align}
    \cR(T,P) \leq \tlO \left(\underbrace{R \sqrt{\snr} \cdot P\cdot \kappa}_{\text{burn-in}} + \cR(TP, 1) \right),
\end{align}
where $\cR(TP, 1)$ captures the regret of that learning algorithm operating in purely sequential fashion and $\kappa$ is a geometry-dependent constant. The price of parallelism is factored into the burn-in term. Although this term is subleading as $T \to \infty$, for many applications of interest (such as protein engineering) it may be the case that $P \sim T$ or $P \gg T$. Thus, understanding the value of $\kappa$ as a function of the geometry of the context set is a question of interest.

\subsection{Arbitrary Contexts}

Our first result shows in the general setting of linear contextual bandits, a uniform bound holds on the number of doubling rounds for any sequence of actions selected.

\begin{lemma}
\label{lem:unif_dr}
Let $\{ \cX_{t,p} \}_{t=1,p=1}^{T,P}$ be an arbitrary sequence of contexts. If \cref{assump:data} holds and the covariance is estimated as in \cref{eq:conf_ellipse2}, then almost surely over any sequence $\x_{t,p}$ of selected covariates, the number of total doubling rounds is bounded by at most $ \Big\lceil \frac{d}{\log(2)} \log \left( 1+ \frac{TP L^2}{d \lambda} \right) \Big\rceil$.
\end{lemma}
\begin{proof}
    First, note that if round $t$ is a doubling round, then there must exist some $\v \in \mathbb{S}^d$ such that $\v^\top \V_{t,P+1} \v > 2 \v^\top  \V_{t,1} \v$. An application of \cref{lem:psd_det} shows this implies $\det(\V_{t,P+1}) > 2 \ \det(\V_{t,1})$. So if $k$ doubling rounds elapse by the end of time $T$ it must be the case that $\det(\V_{T,P+1}) > 2^k \ \det(\V_{1,0}) \implies \log\left(\frac{\det(\V_{T,P+1})}{\det (\V_{1,0})}\right) > k \log(2)$. However by \cref{lem::ellip_potential}, for any sequence of selected covariates satisfying \cref{assump:data}, we have that $\log\left(\frac{\det(\V_{T,P+1})}{\det \V_{1,0}}\right) \leq d \log \left( 1+ \frac{TP L^2}{d \lambda} \right)$. So it follows $k < \left\lceil \frac{d}{\log(2)} \log \left( 1+ \frac{TP L^2}{d \lambda} \right) \right\rceil$.
\end{proof}

\cref{lem:unif_dr} ensures that in broad generality, the number of doubling rounds is bounded by $\tlO(d)$. Instantiating our previous results then gives the following.
\begin{corollary}\label{cor:arb_context_regret}
In the setting of \cref{thm:regret_par_linucb}, choosing $\lambda=L^2$ and taking the doubling-routine \textsc{DR} as the identity map, the regret of both \cref{algo:par_linucb,,algo:par_lazy_linucb} satisfy
\begin{equation*}
    \cR(T, P) \leq \tlO \left( R \cdot \left(P \sqrt{\snr} \cdot d + \sqrt{dTP} ( \sqrt{d}+\sqrt{\snr} + \frac{\epsilon}{R} \sqrt{TP}) \right) \right),
\end{equation*}
with probability at least $1-\delta$. In the setting of \cref{thm:linTS_parallel} with the choice $\lambda=L^2$ and also taking the doubling-routine \textsc{DR} as the identity map, the regret of \cref{algo:par_lints,,algo:par_lazy_lints} satisfy
    \begin{equation}
    \cR(T, P) \leq \tlO \left( R \left( P \sqrt{\snr} \cdot d + d\sqrt{TP} ( \sqrt{d} + \sqrt{\snr}  + \frac{\epsilon}{R} \sqrt{TP} \right)  \right),
    \end{equation}
    with probability at least $1-3\delta$, whenever $\delta \leq \frac{1}{6}$. 
\end{corollary}
\begin{proof}[Proof of \cref{cor:arb_context_regret}]
Note that by \cref{lem:unif_dr}, we must have that $\sum_{t=1}^T \Ind[\cD_t] \leq  \lceil \frac{d}{\log(2)} \log \left( 1+ \frac{TP L^2}{d \lambda} \right) \rceil$. An application of \cref{thm:regret_par_linucb,,thm:linTS_parallel} gives the result after choosing $\lambda=L^2$.
\end{proof}
We can interpret the result as follows.
 \begin{itemize}
    \item The baseline regret of Linear UCB and Lazy Linear UCB interacting in a purely sequential fashion (with the standard choice of regularizer $\lambda=L^2$) for $TP$ rounds scales as,
     \begin{equation*}
         \cR(TP, 1) \leq \tlO(R \sqrt{dTP} \cdot( \sqrt{d} + \sqrt{\snr} + \frac{\epsilon}{R} \sqrt{TP} )),
     \end{equation*}
     with an analogous expression inflated by an extra $\sqrt{d}$ holding for Thompson sampling and Lazy Thompson sampling.
     The canonical normalizations for the noise, parameter, and covariates in the literature assume $\snr=R=\Theta(1)$ as well as $\epsilon=0$, which simplifies to the oft-stated (and optimal) regret $\tlO(d \sqrt{TP})$. In this case, if $T \geq \tlOmega(P)$, the regret of our parallel algorithms nearly matches the optimal worst-case regret of a single sequential agent.
     \item \cref{thm:nonlin_regret_bound} and \cref{ex:lin_mod} essentially degenerate to the aforementioned result after identifying $B \leftrightarrow LS$, showing the analogy between these results. Note however the algorithmic implementation of \cref{algo:par_linucb,,algo:par_lazy_linucb} can use the linear structure in the context sets to obtain better practical performance.
     \item The result in \cref{cor:arb_context_regret} suggests that when we opt for a large choice of $P$, parallelism is particularly beneficial in the small $\snr$ regime. The fact that a low fidelity data-generation process benefits parellelism may seem counter-intuitive. This property arises in part because the algorithms we consider are optimistic -- so environments with large $\snr$ have large parameter norms necessitating the usage of large confidence sets which induce more regret.
 \end{itemize}

\subsection{Finite Context Sets}

Next we show how structure in the context set (in this case, finiteness of the action space) can be leveraged to bound the number of doubling rounds without modifying the standard choices of the hyperparameters for the optimistic algorithms considered here.

\begin{lemma}
\label{lem:doub_direction}
Let $\cX_{t,p} \subset \cX = \{ \x_i \}_{i=1}^m$ for all $t \in [T]$ and $p \in [P]$, where $\cX$ is a finite set of vectors and $|\cX| = m$. If \cref{assump:data} holds and the covariance is estimated as in \cref{eq:conf_ellipse2}, then almost surely over any sequence $\x_{t,p}$ of selected covariates, the number of total doubling rounds is bounded by at most $m \cdot \log_2(\lceil P \rceil)$.
\end{lemma}
\begin{proof}
We first recall from \cref{cond:critical_inequality} that in a doubling round at time $t$  
\begin{equation}
    \V_{t,P+1} \not\preceq 2 \V_{t,1}. \label{eq:doub_violation}
\end{equation}
If $\cX_{t,p} \subset \cX = \{ \x_i \}_{i=1}^m$, then
\begin{equation*}
    \V_{t,1} = \lambda \mathbf{I}_d + \sum_{i=1}^m w_{t, 1}(i), \x_i\x_i^\top  
\end{equation*}
where $w_{t,1}(i)$ corresponds to the number of times action $\x_i$ has been played by all processors up to and including all $P$ actions played at time $t-1$. Whenever $t$ is a doubling round and~\cref{eq:doub_violation} holds, there must exist $i \in [m]$ such that
\begin{equation}
    w_{t, P+1}(i) > 2 w_{t,1}(i),
    \label{eq:weight_increment}
\end{equation}
since otherwise for all $i$ it would hold that $ w_{t, P+1}(i) \leq 2 w_{t,1}(i)$ implying that $ \V_{t,P+1} \preceq 2 \V_{t,1}$, contradicting~\cref{eq:doub_violation}. Observe that each time~\cref{eq:weight_increment} holds,  $w_{t, P+1}(i) - w_{t,1}(i) > w_{t,1}(i)$, implying that during round $t$ arm $i$ was pulled more times than the total number of times it has been pulled thus far. Then, for all $i \in [m]$, the difference $w_{t, P+1}(i ) - w_{t,1}(i) \leq P$ and therefore for any $i \in [m]$ condition~\cref{eq:weight_increment} cannot hold for more than $\lceil \log_2(P) \rceil$ iterations. Since there are only $m$ underlying vectors in the contexts the result follows.
\end{proof}
This observation implies a bound for our parallel bandit algorithms over finite context sets:
\begin{corollary}\label{cor:doub_dir_par_linucb_regret}
In the setting of \cref{thm:regret_par_linucb}, assume  the context sets $\cX_{t,p} \subset \cX$ for all $t \in [T], p \in [P]$ where $\abs{\cX}=m$ is a finite set of vectors. Then choosing $\lambda=L^2$ and taking the doubling-routine \textsc{DR} as the identity map, the regret of both \cref{algo:par_linucb,,algo:par_lazy_linucb} satisfy
\begin{equation*}
    \cR(T, P) \leq \tlO \left( R \cdot \left(P \sqrt{\snr} \cdot m  + \sqrt{dTP} ( \sqrt{d}+\sqrt{\snr} + \frac{\epsilon}{R} \sqrt{TP}) \right) \right),
\end{equation*}
with probability at least $1-\delta$. In the setting of \cref{thm:linTS_parallel} with the choice $\lambda=L^2$ and also taking the doubling-routine \textsc{DR} as the identity map, the regret of \cref{algo:par_lints,,algo:par_lazy_lints} satisfy
    \begin{equation}
    \cR(T, P) \leq \tlO \left( R \left( P \sqrt{\snr} \cdot m + d\sqrt{TP} ( \sqrt{d} + \sqrt{\snr}  + \frac{\epsilon}{R} \sqrt{TP} \right)  \right),
    \end{equation}
    with probability at least $1-3\delta$, whenever $\delta \leq \frac{1}{6}$. 
\end{corollary}
\begin{proof}[Proof of \cref{cor:doub_dir_par_linucb_regret}]
By \cref{lem:doub_direction}, we have that $\sum_{t=1}^T \Ind[\cD_t] \leq m \lceil \log_2(P) \rceil$. An application of \cref{thm:regret_par_linucb,,thm:linTS_parallel} gives the result after choosing $\lambda=L^2$.
\end{proof}

We now interpret this result.
\begin{itemize}
    \item As before, the second terms in \cref{cor:doub_dir_par_linucb_regret} captures the effect of perfect parallel speed-up---it is the regret that a single agent would achieve playing for a total of $TP$ rounds while the first term in \cref{cor:doub_dir_par_linucb_regret} is a burn-in term that bounds the number of rounds which may not be doubling rounds. As before \cref{thm:nonlin_regret_bound} and \cref{ex:finite} essentially degenerate to this result.
    \item Relative to \cref{cor:arb_context_regret} the scaling on the burn-in term, which is proportional to $P$, contains a factor $m$ instead of $d$. Thus, if $m \ll d$, \cref{cor:doub_dir_par_linucb_regret} shows how the algorithms considered here can take advantage of additional structure in the context set to mitigate the the cost of parallelism. This result directly mirrors \cref{ex:finite} obtained from our general nonlinear results.
    \item The proof of \cref{cor:doub_dir_par_linucb_regret} does not exploit any correlation structure in the action set to bound the number of doubling rounds---for example, clusters of arms that are tightly bunched together in the global context space under a suitable notion of distance. Under natural conditions on the action set, sharper instance-dependent bounds may be possible.
\end{itemize}

\subsection{Rich Context Sets}
\label{sec:rich_context}
Finally we consider a setting in which we have a sequence of context sets $\cX_{t,p}$ with regularity structure that are formally defined as follows.
\begin{definition}
\label{def:rich_exp_dist}
    The contexts $\cX_{t,p}$ and distributions $\pi_{t,p}(\cdot)$ are a pair of  \textit{rich exploration contexts/distributions} if there exist $\chi^2$, $\pi_{\max}^2$ and $\pi_{\min}^2$ such that, 
    \begin{itemize}
        \item $\x \sim \pi_{t,p}(\cdot)$ satisfies $\x \in \cX_{t,p}$ almost surely for all $t \in [T]$, $p \in [P]$.  
        \item For any sequence $\x_{t,p} \in \cX_{t,p}$ selected by the bandit algorithm for $t \in [T]$, $\sum_{p=1}^{P} \x_{t,p} \x_{t,p}^\top \leq P \chi^2 \I$. 
        \item Given $\x_{t,p} \sim \pi_{t,p}(\cdot)$, with population mean and covariance defined as, $\mE[\x_{t,p}]=\mu_{t,p}$ and 
        $\mE[(\x_{t,p}-\mu_{t,p}) (\x_{t,p}-\mu_{t,p})^\top ]=\mSigma_{\pi_{t,p}}$, we have that $\pi_{\max}^2 \I \succeq \mSigma_{\pi_{t,p}} \succeq \pi_{\min}^2 \I$ for all $t \in [T], p \in [P]$.
    \end{itemize}
\end{definition}
Intuitively, \cref{def:rich_exp_dist} guarantees the sequence of presented contexts are (1) sufficiently similar since there is a common positive semidefinite ordering for the covariances across all contexts, and (2) sufficiently cover all directions in $\mR^d$ when the parameters $\chi^2$, $\pi_{\max}^2$ and $\pi_{\min}^2$ are of the same order. If these conditions are satisfied then exploration distributions exist which can uniformly explore all directions of the underlying context sets well. We now provide a simple example of a set of stochastically generated contexts which obey \cref{def:rich_exp_dist}. For each processor, let there be a single context set randomly generated as $\cX_{p} = \{ \x : \x \sim \cD(\cdot) \}_{i=1}^m$ where $\cD(\cdot)$ is a $O(1)$-subgaussian and $O(1)$-bounded distribution in $\mR^d$ such that $\cX_{t,p} = \cX_{p}$ for all $t \in [T]$, $p \in [P]$. Now define $\pi_{t,p}(\cdot)$ as the uniform distribution over all the vectors in a given context $\cX_{p}$ at time $t$. Then using a simple matrix concentration argument, we can verify there exists $\chi^2 \leq \tlO(1/d)$, $\pi_{\max}^2 \leq \tlO(1/d)$ and $\pi_{\min}^2 \geq \tlOmega(1/d)$ in \cref{def:rich_exp_dist} when $P \geq \tlOmega(d)$ (with probability at least $1-\delta$ over the randomness in $\cD(\cdot)$ and $\pi(\cdot)$).

Sampling from a rich exploration policy (when it exists) can serve as effective doubling-round subroutine, since it can help stabilize the covariance intra-round in later rounds. Random exploration helps stabilize the covariance in later rounds as a consequence of concentration: given a set of randomly sampled covariates $\{ \x_{i, p} \}_{i=1, p=1}^{N, P}$, we expect $\frac{1}{NP} \sum_{j=1}^{N} \sum_{p=1}^{P} \x_{i,p} \x_{i,p}^\top \approx \mSigma_{\pi}$, where $\mSigma_{\pi} \approx \mSigma_{\pi_{t,p}}$ for all $t \in [T], p \in [P]$. So if a significant number of doubling rounds occur (during which \cref{algo:random_explore} is used as \textsc{DR}), then later rounds are unlikely to be doubling rounds since the covariance matrix will have a large component proportional $\pi_{\min}^2 \I \preceq \mSigma_{\pi}$ in its spectrum. 

\begin{algorithm}[!bt]
\caption{Random Exploration Subroutine}\label{algo:random_explore}
\begin{algorithmic}[1]
\renewcommand{\algorithmicrequire}{\textbf{Input: }}
\renewcommand{\algorithmicensure}{\textbf{Output: }}
\REQUIRE $\pi_{t,p}(\cdot)$.
\FOR{$p = 1:P$} 
    \STATE Sample $\x_{p} \sim \pi_{t,p}(\cdot)$
\ENDFOR
\RETURN $\{ \x_{p} \}_{p=1}^{P}$
\end{algorithmic}
\end{algorithm} 

Using this idea we obtain the following corollary.
\begin{corollary}\label{cor:rich_explore_par_linucb_regret}
In the setting of \cref{thm:regret_par_linucb}, assume  \cref{def:rich_exp_dist} holds for some rich exploration policies, $\{ \pi_{t,p}(\cdot) \}_{t=1,p=1}^{T,P}$, and context pairs, $\{ \cX_{t,p} \}_{t=1,p=1}^{T,P}$. Then choosing $\lambda=L^2$ and taking the doubling-routine \textsc{DR} as \cref{algo:random_explore} with these $\pi_{t,p}(\cdot)$, the regret of both \cref{algo:par_linucb,,algo:par_lazy_linucb} satisfy
\begin{equation*}
    \cR(T, P) \leq \tlO \left( R \cdot \left( \sqrt{\snr} \cdot \left(\frac{L^2 \pi_{\max}^2}{\pi_{\min}^4} + P\frac{\chi^2}{\pi_{\min}^2} \right)+ \sqrt{dTP} ( \sqrt{d}+\sqrt{\snr} + \frac{\epsilon}{R} \sqrt{TP}) \right) \right),
\end{equation*}
with probability at least $1-2\delta$. In the setting of \cref{thm:linTS_parallel} with the choice $\lambda=L^2$ and also taking the doubling-routine \textsc{DR} as \cref{algo:random_explore} with this $\pi_{t,p}(\cdot)$, the regret of \cref{algo:par_lints,,algo:par_lazy_lints} satisfy
    \begin{align*}
    \cR(T, P) \leq \tlO \left( R \left(  \sqrt{\snr} \cdot \left(\frac{L^2 \pi_{\max}^2}{\pi_{\min}^4} + P\frac{\chi^2}{\pi_{\min}^2} \right)  + d\sqrt{TP} ( \sqrt{d} + \sqrt{\snr}  + \frac{\epsilon}{R} \sqrt{TP} \right)  \right),
    \end{align*}
    with probability at least $1-4\delta$, whenever $\delta \leq \frac{1}{6}$.
\end{corollary}
We can provide further intepretation of this result as follows.

\begin{itemize}
    \item The guarantee presented here resembles that of  \cref{cor:doub_dir_par_linucb_regret}. However, here the coefficient on the burn-in term (which bounds the number of doubling rounds) scales as $\frac{L^2 \pi_{\max}^2}{\pi_{\min}^4} + P\frac{\chi^2}{\pi_{\min}^2}$ and is valid even for infinite context sets.\footnote{Recall the coefficient in \cref{cor:doub_dir_par_linucb_regret} scales as $m$ where is a bound on the number of distinct vectors in the presented contexts.} The scaling with $\frac{\chi^2}{\pi_{\min}^2}$ is natural. Quantities proportional to $\frac{\chi^2}{\pi_{\min}^2}$ represent the cost of the non-homogeneous geometry of the presented contexts. 
    \item In contrast to  \cref{cor:arb_context_regret,,cor:doub_dir_par_linucb_regret}, \cref{cor:rich_explore_par_linucb_regret}
    takes advantage of a common geometric structure in the presented contexts $\cX_{t,p}$. \cref{cor:doub_dir_par_linucb_regret} is applicable to general context sets (for which we may have $\pi_{\min}^2 \approx 0$). However, if $m = \Theta(\exp(d))$ for example, the guarantee degrades badly. Similarly, \cref{cor:arb_context_regret} has a burn-in term scaled by $d$. If the context sets are well-conditioned, in the sense that $\frac{\chi^2}{\pi_{\min}^2} \ll d$, then the parallelism cost here is significantly lower then in the aforementioned cases. 
    \item In the example of rich context sets/distributions pair presented in the text, we can have $\frac{L^2 \pi_{\max}^2}{\pi_{\min}^4} + P\frac{\chi^2}{\pi_{\min}^2} \leq \tlO(P)$ with high probability when $P \gtrsim \tlOmega(d)$. As $P \to \infty$, the burn-in term only scales as $\tlO(P)$, which up to logarithmic factors is free of any explicit dependence on the size of the context sets.
\end{itemize}

%% file: lb.tex
\section{Parallel Regret Lower Bounds}
\label{sec:lb}
Lastly, we show that the upper bounds on the parallel regret provided in the previous sections are nearly matched by complementary information-theoretic lower bounds on the parallel regret under natural parameter scalings. These regret lower bounds are essentially synthesized from various results across the literature although we provide proofs for completeness. These bounds capture three specific terms. The first is the regret induced by the learning and optimally playing $\thetastar$ for a well-specified model, the second is the cost of misspecification, and the third a term capturing the difficulty of learning in the parallel setting. These come together in the following theorem.

\begin{theorem}
\label{thm:lb}

For any parallel bandit algorithm, there
exists bandit environments satisfying
\cref{assump:noise,,assump:data,,assump:param},

\begin{enumerate} 
\item Such that when $\epsilon=0$ there is a single global context set (i.e. $\cX = \cX_{t,p}$ for all $t \in [T], p \in [P]$) for which,
\begin{equation}
     \mE[\cR(T, P)] \geq  \Omega\left(Rd\sqrt{TP}\right) \label{eq:sphere_lb}
\end{equation}
when $T \geq d \max(1, \frac{1}{3 \sqrt{2}\sqrt{\snr}})$, 
\item A single (finite) global context set with $m$ vectors such that when $d=\lceil 8 \log(m)/\epsilon^2 \rceil$, $R=0$ and $LS \geq 1$, 
\begin{equation}
     \mE[\cR(T, P)] \geq  \Omega\left(\epsilon \sqrt{ \frac{d-1}{\log(m)}} \cdot \min(TP, m-1)  \right), and \label{eq:misspec_lb}
\end{equation}
\item An (oblivious) adversarial context for which $R=0$,
    \begin{equation}
        \mE[\cR(T, P)] \geq  \Omega\left(LS P \max\left(\min \left(\sqrt{d}, \frac{T}{\sqrt{d}} \right), 1\right)\right).
        \label{eq:linear_p_lb}        
    \end{equation}
\end{enumerate}
\end{theorem}
The proof of \cref{thm:lb} follows from three separate parts. The first component is a simple reduction which argues that the minimax regret of a parallel bandit algorithm, presented with contexts fixed across processors, must be at least as much as its  sequential counterpart given access to the same number of total arm queries. The first part of \cref{thm:lb} then follows from a standard construction for lower bounding the sequential regret of a bandit instance when the context set is taken to be a sphere. The second part is a (noiseless) lower bound which uses a probabilistic argument to witness a finite context set, function $f$, and $\thetastar$ for which the misspecification level must show up multiplicatively in the regret. The third component shows the existence of a hard instance for learning when adversarial contexts are presented which necessitates a linear dependence on $P$ similar to a construction in \citet[Theorem 1]{han2020sequential}.
We now make several comments to further interpret these results.
\begin{itemize}
    \item Together the terms  \cref{eq:sphere_lb,,eq:misspec_lb} show the main components of \cref{thm:regret_par_linucb} are unavoidable---in particular the term corresponding to the variance of learning $R d \sqrt{TP}$ and the term capturing the magnitude of misspecification $\epsilon \sqrt{d} TP$. The final term \cref{eq:misspec_lb} shows that in the general contextual bandit setting (for $T > d$) some nontrivial dimension dependence tied to $P$ -- of the form $P \sqrt{d}$ -- is needed.
    \item The first term \cref{eq:sphere_lb} in \cref{thm:lb} captures the variance of learning under optimism in the parallel regret. For this portion of the lower bound, the context set is taken to be the sphere, which satisfies the conditions of a rich exploration set when the exploration distribution is taken as the uniform distribution over the spherical shell (here $\ell=L$). Hence for $\snr \lesssim d$ and $\epsilon=0$, the guarantee from \cref{cor:rich_explore_par_linucb_regret} matches this lower bound up to logarithmic factors for sufficiently large $T$.
    \item To gain further intuition for \cref{eq:misspec_lb} it is helpful to consider the high-dimensional scaling limit where $m = \Theta(d^k)$ for $k \gg 1$ and $TP \ll m$. Then, under the conditions of the result, we can see for any $\epsilon$ there exists a sufficiently large $d$ so that the parallel regret satisfies $\tOmega(\epsilon \sqrt{d} TP)$. Hence (in the realistic regime) where the context set contains a large numbers of context vectors and there are not sufficiently many queries to observe all $m$ of them, the $\epsilon \sqrt{d} TP$ in the regret within \cref{thm:regret_par_linucb} is unavoidable up to logarithmic factors.
    \item Capturing the necessity of the ``burn-in" terms, which represent the price of parallelism in our upper bounds, is an interesting but challenging research direction. In particular, because in many applications the information-theoretic limits of learning when $P \sim T$ may be of interest. The term \cref{eq:linear_p_lb} shows that the linear $P$ term in the regret must have at least a $\sqrt{d}$ dependence in the presence of adversarial contexts. Closing the gap between this and the $P d$ term in our general upper bound is an open direction. Moreover, the interplay between the structure of the context set and the burn-in terms in the upper bounds in \cref{cor:arb_context_regret,,cor:doub_dir_par_linucb_regret,,cor:rich_explore_par_linucb_regret}
    seems quite nuanced. Any such lower bounds capturing these dependencies will likely need to be constructed on a case-specific basis for different context set geometries as well as be geared towards the small $T$ regime.
\end{itemize}

%% file: experiments.tex
\section{Experiments}

Here we explore the performance of the parallel linear bandit algorithms presented in this paper on several synthetic and real problem instances of increasing complexity. While analysis in the bandits literature is often focused on minimizing regret, in the batch setting best arm identification may be of primary interest for some practical design settings where no cost is incurred for additional arms after the best performing arm within a round. Hence, we explore the performance of our family of algorithms in both parallel regret and best arm identification.

In the synthetic data settings, we investigate both a perfectly linear setting and a misspecified setting generated from the output of a randomly initialized neural network. The real data instances are derived from a material science and biological sequence design applications to provide breadth across a variety of context set geometries and application-specific behaviors. In the real data settings, we consider the performance over parallel variants of all algorithms considered herein against a baseline that is the $\epsilon$-greedy algorithm. Note that this baseline makes no structural assumptions on the conditional model $y_i | x_i$ (and as such is ``unbiased") but also is not able to take advantage of the covariates $\x_i$, since it doesn't construct a regression model to guide exploration. 

For all experimental setups, we fix the total number of arm queries $TP$ and run with three different levels of parallelism (i.e. $P = \{1, 10, 30\}$ for the superconductor setting and $P = \{1, 10, 100\}$ for all others) over 30 separate trials. All algorithms use a doubling round routine which is set to the identity map. As is common, for both Thompson sampling variants, we avoided inflating the confidence set radius by the additional $\sqrt{d}$ factor so it matches the confidence sets of the other algorithms. The misspecification parameter was set to zero for all experiments. The hyperparameters of each algorithm were tuned via a post-hoc grid search over a logarithmically-spaced grid for the random neural network and real data experiments (see \cref{sec:app_exp} for details) as in \citet{foster2018practical}.

\subsection{Synthetic Experiments}

We begin by testing the ability of our optimistic algorithms to parallelize on simulated data. We consider a problem in $d=100$ with a linear reward oracle whose underlying parameter $\thetastar/\norm{\thetastar}$ for $\thetastar \sim \cN(0, \I_d)$ subject to Gaussian additive noise $\epsilon \sim \cN(0, 1)$. We then generate a fixed, global context set $\cX = \{ \x_i/\norm{\x_i}_2 \}_{i=1}^{m}$ for $\x_i \sim \cN(0, \I_d)$ with $m=10^4$ actions. We then set $\cX_{t,p} = \cX$ for all $t \in [T]$ and $p \in [P]$. The hyperparameters of the algorithms were chosen according to their theoretically-motivated values $\lambda = 1, R=1, S=1$ with $\delta=1/T$. As \cref{fig:synthetic} shows, the parallel versions of each of the algorithms asymptotically achieve a nearly perfect speed-up with respect to parallelism as measured by the regret. As $T \to \infty$ the performance of the different types of base algorithms are comparable.

Next we investigate the parallelism of our methods under changing context sets. We generate a random context set $\cX_{t,p} = \{ \x_i/\norm{\x_i}_2 \}_{i=1}^{m}$ for $\x_i \sim \cN(0, \I_d)$ with $m=10^4$ actions for each timestep-processor pair $(t,p)$ with hyperparameters set as before. Once again we see in \cref{fig:randomcontext} that each algorithm achieves near perfect speed-up as measured by parallel regret and all base algorithms are asymptotically comparable. 



Finally, from our theoretical results we recall the importance of the covariances $\V_{t,p}$ remaining quasi-static intra-round. We examine this behavior in the synthetic linear reward setting by determining the minimal doubling round coefficient $\alpha^{\min}_{t,p}$ for each arm query which satisfies $0 \preceq \alpha \V_{t,1} - \V_{t,p}$, where we call $\alpha$ the \textit{doubling round coefficient}. That is, for any doubling round coefficient $\alpha \geq \alpha^{\min}_{t,p}$ for all $p \in [P]$, the critical covariance inequality is satisfied and no doubling round is triggered for that round. Note that in our theoretical results, we arbitrarily set $\alpha = 2$ for our analysis. As before, we generate a fixed global context and theoretically-motivated hyperparameter values, but this time with $d=20$ and $m=10^3$. We then run the algorithms without calling any doubling round routines and observe how the doubling round coefficient changes through time for $P=100$. At each $(t,p)$ timestep, we compute $\alpha^{\min}_{t,p}$ and plot it against the number of arm queries as shown in \cref{fig:doublinground}. As expected, we qualitatively notice a sawtooth pattern in all algorithms as the minimal doubling round coefficient increases for fixed $t$ as the covariance $\V_{t,p}$ gets updated for each additional processor $p$ which then resets back down at the end of each round. Additionally, all algorithms experience a dropoff in minimal doubling round coefficient with each successive round indicating that the covariance gets more stable as more arms are queried with all algorithms having an essentially flat $\alpha^{\min}_{t,p}$ near 1 as $t \geq 2000$. Finally, we observe that in the fixed context setting where diversity of arms isn't introduced via the context, LinUCB has the highest doubling round coefficients since the algorithm chooses the same arm for all processors leading for significant changes in the shape of the covariance within a round making the algorithm susceptible to overconfidence. Note that the other three algorithms have much smaller doubling round coefficients for a much shorter time due to the increased diversity in actions played in the bandit algorithm.

\begin{figure}[!hbt]
\centering
\begin{minipage}[c]{.24\linewidth}
\includegraphics[width=\linewidth]{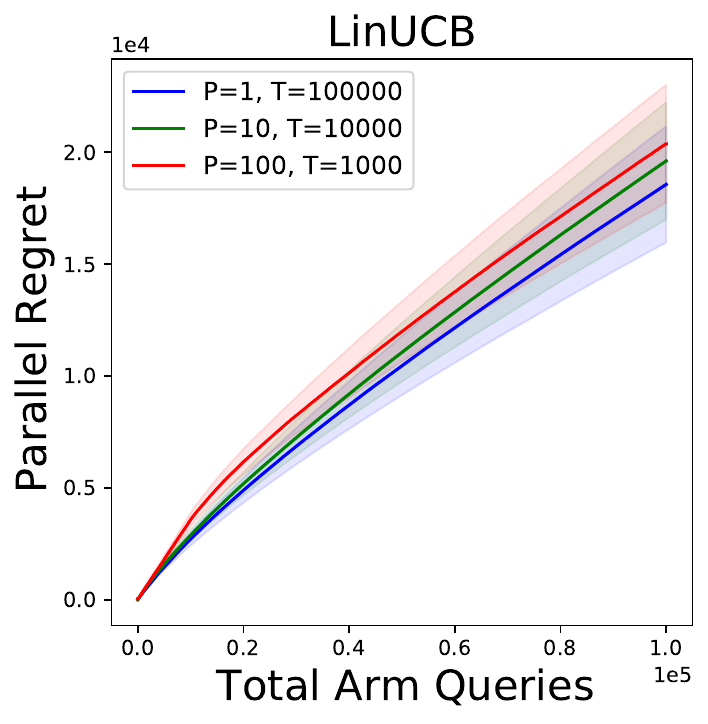}
\end{minipage}
\begin{minipage}[c]{.24\linewidth}
\includegraphics[width=\linewidth]{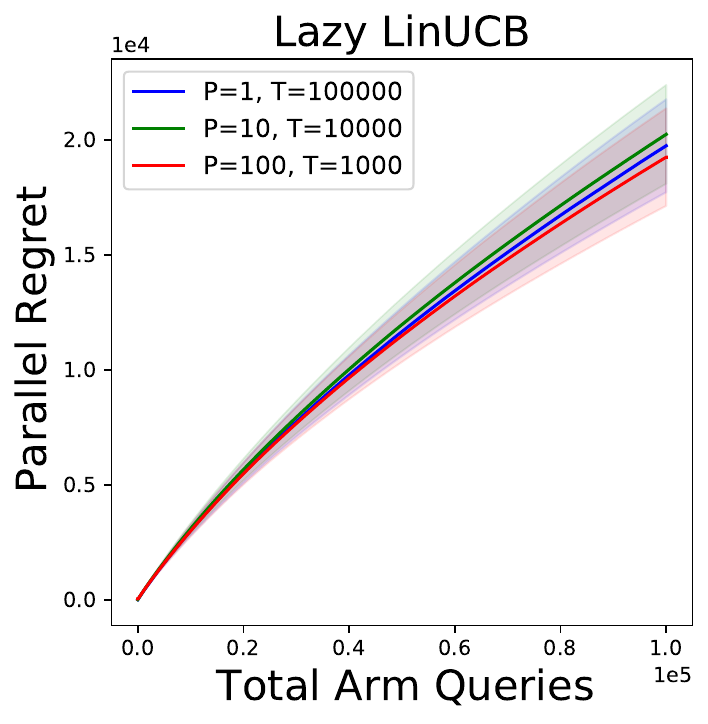}
\end{minipage}
\begin{minipage}[c]{.24\linewidth}
\includegraphics[width=\linewidth]{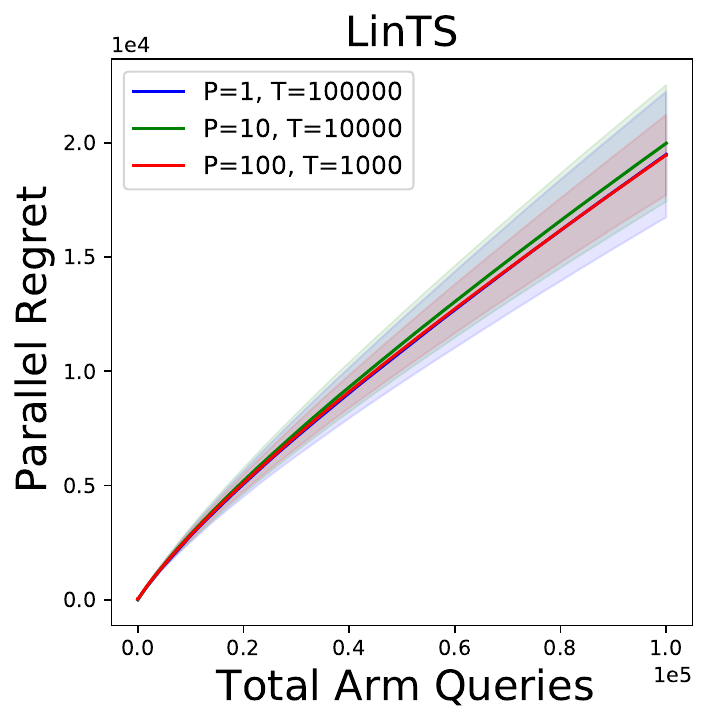}
\end{minipage}
\begin{minipage}[c]{.24\linewidth}
\includegraphics[width=\linewidth]{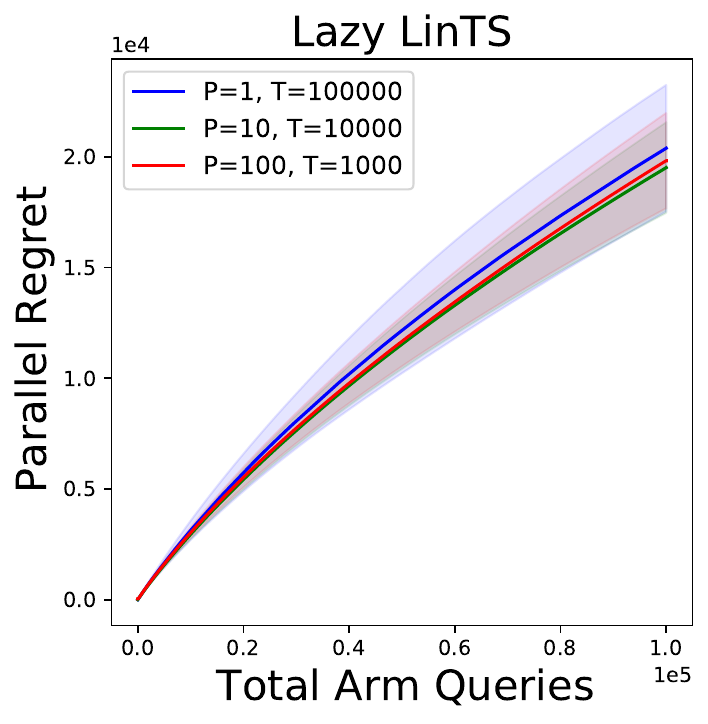}
\end{minipage}
\caption{Fixed context setting. From left to right: Regret of  LinUCB, Lazy LinUCB,  LinTS, and Lazy LinTS for varying values of $P$. The mean regret is plotted across $30$ runs with the standard deviation as the shaded region. Here $d=100$, $m=10^4$.
}
\label{fig:synthetic}
\end{figure}


\begin{figure}[!hbt]
\centering
\begin{minipage}[c]{.24\linewidth}
\includegraphics[width=\linewidth]{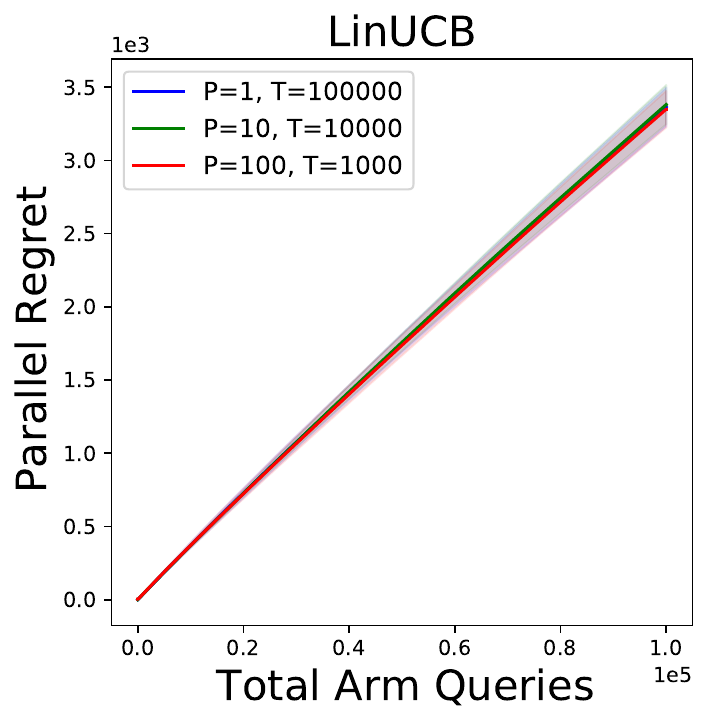}
\end{minipage}
\begin{minipage}[c]{.24\linewidth}
\includegraphics[width=\linewidth]{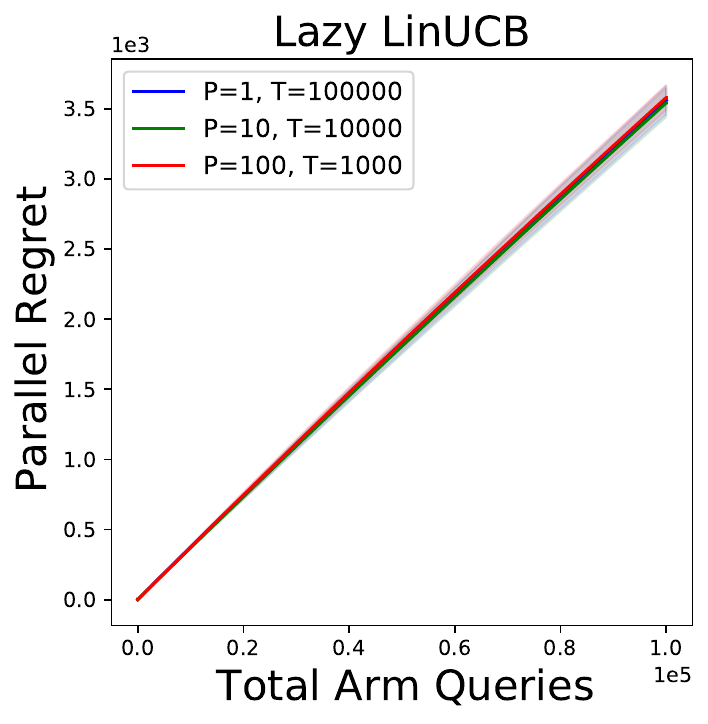}
\end{minipage}
\begin{minipage}[c]{.24\linewidth}
\includegraphics[width=\linewidth]{plots/randomcontext/randomcontext,parallelism,d=100,K=10000LazyLinUCB.pdf}
\end{minipage}
\begin{minipage}[c]{.24\linewidth}
\includegraphics[width=\linewidth]{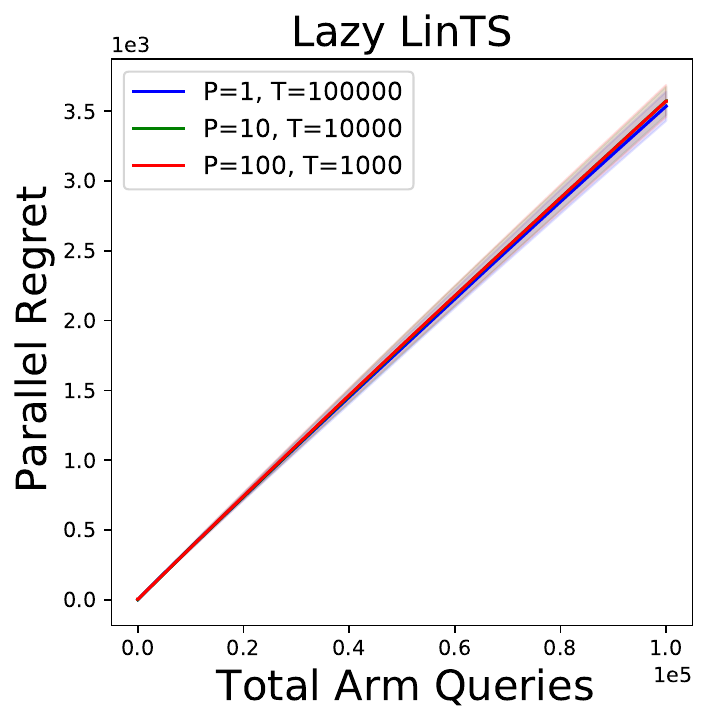}
\end{minipage}
\caption{Changing context setting. From left to right: Regret of LinUCB, Lazy LinUCB, LinTS, and Lazy LinTS for varying values of $P$. The mean regret is plotted across $30$ runs with the standard deviation as the shaded region. Here $d=100$, $m=10^4$.
}
\label{fig:randomcontext}
\end{figure}

\begin{figure}[!hbt]
\centering
\begin{minipage}[c]{.24\linewidth}
\includegraphics[width=\linewidth]{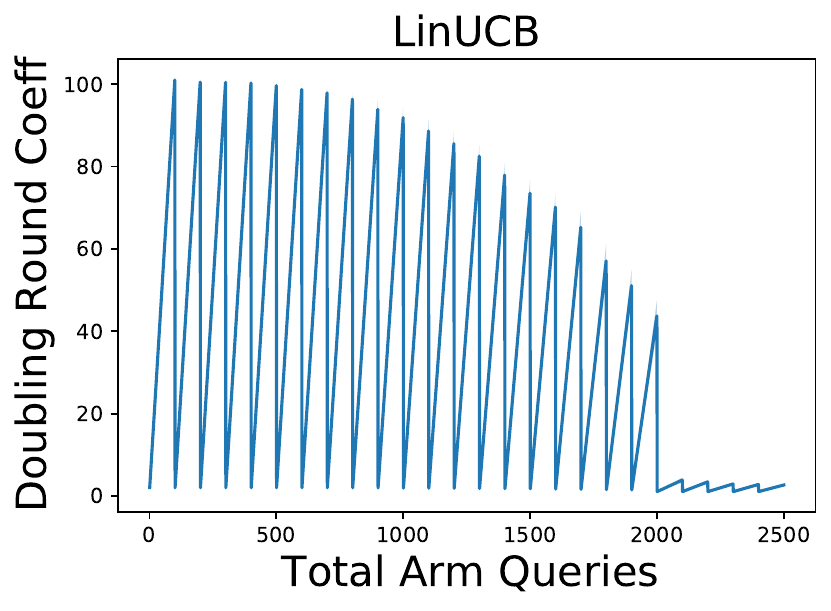}
\end{minipage}
\begin{minipage}[c]{.24\linewidth}
\includegraphics[width=\linewidth]{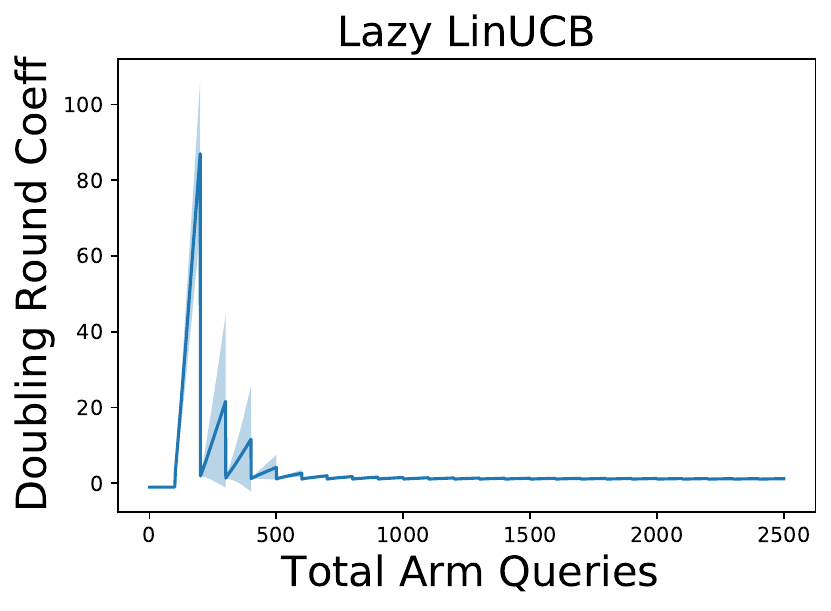}
\end{minipage}
\begin{minipage}[c]{.24\linewidth}
\includegraphics[width=\linewidth]{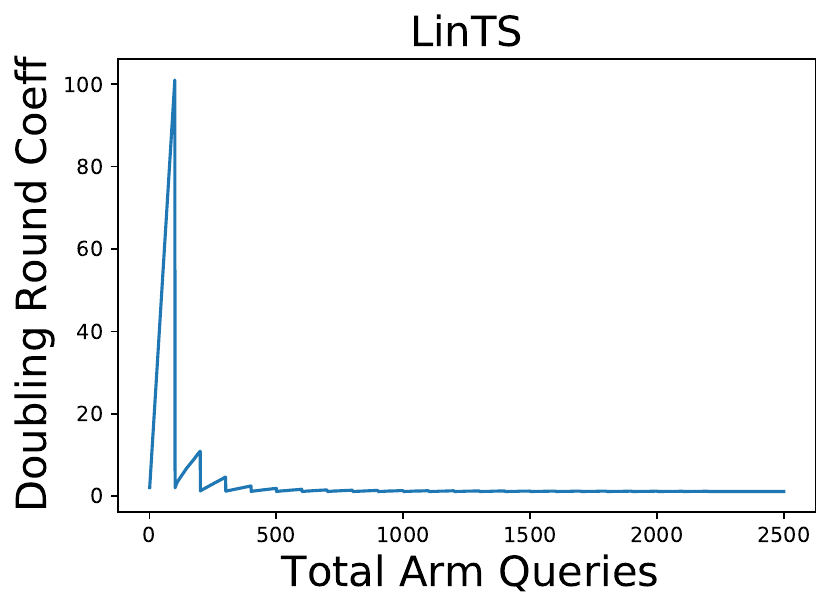}
\end{minipage}
\begin{minipage}[c]{.24\linewidth}
\includegraphics[width=\linewidth]{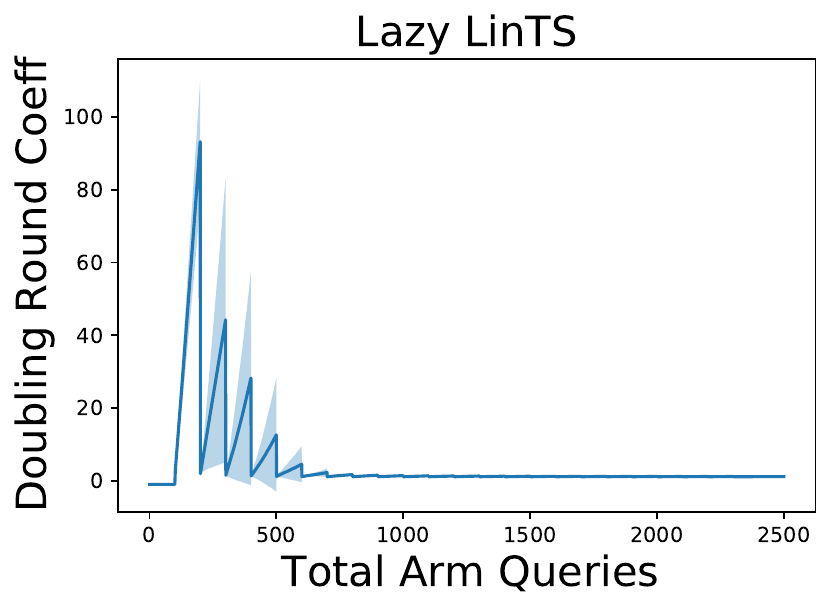}
\end{minipage}
\caption{Doubling round coefficients. From left to right: doubling round coefficients of LinUCB, Lazy LinUCB, LinTS, and Lazy LinTS. The mean coefficient is plotted across $30$ runs with the standard deviation as the shaded region and $d=20$, $m=10^3$, and $P=100$.
}
\label{fig:doublinground}
\end{figure}

\subsection{Randomly Initialized Neural Network Data}
Recent studies (\cite{angermueller2020population}, \cite{brookes2018design}) have modeled fitness landscapes from biological sequence design problems with randomly initialized neural networks.  Such randomly initialized neural network exhibit nearly linear properties for small parameter values; these properties are essential to the guarantees for our algorithms as demonstrated in the prequel.  Moreover, they deviate from a linear model sufficiently to serve as a good testbed for model misspecification. The ``biological sequence" input is modeled as a $14$-length binary string $x_i \in \{0,1\}^{14}$ ($m = 16,384$ sequences) to mimic the combinatorial nature of biological sequences. The fitness landscape $f(x_i)$ is modeled by a feedforward neural network with three hidden layers (of size 128, 256, and 512 hidden units) where each weight is sampled i.i.d.\ via Xavier initialization $w \sim \mathrm{Unif}(-\sqrt{6/(h_i + h_{i+1})},\sqrt{6/(h_i + h_{i+1}}))$ where $h_i$ is the number of units in layer $i$. The output $y_i$ of the randomly initialized feedforward neural network can be thought of as the oracle fitness landscape which we wish to optimize. Note unlike in common use cases for neural networks the initialized weights are never modified. To model experimental noise, we add Gaussian noise to generate the reward, $r_i = y_i + \epsilon_i$, where $\epsilon_i \sim \cN(0, 0.5^2)$. 

Our family of bandit algorithms were run with both linear features only ($d=14$) and quadratic features ($d=210$). In this setting, we first verify that a linear model is appropriate. The best fit linear model for linear features and quadratic features had an $R^2$ of $0.7$ and $0.87$, respectively.
 
As \cref{fig:randomnn} demonstrates in the quadratic feature setting, the variance between runs of the parallel regret for the lazy methods tends to be much higher than the non-lazy methods due to the correlation of covariance updates within a batch. LinUCB performs the best upfront in the purely sequential setting due to the high early round cost that Thompson sampling pays upfront, while paying a much lower price in regret in successive rounds. In higher parallelism regimes, Thompson sampling performs the best in terms of parallel regret after the first few hundred arm queries. This indicates that Thompson sampling benefits from encouraging diversity.

In comparison as shown in \cref{fig:1drandomnn}, the linear features perform demonstrably worse in terms of parallel regret than the quadratic features in just a few hundred arm queries across all levels of parallelization. Furthermore, the performance of LinUCB in the linear feature setting suffers most significantly relative to the other methods further confirming that diversity is important in settings of model misspecification. 

\begin{figure}[!hbt]
\centering
\begin{minipage}[c]{.24\linewidth}
\includegraphics[width=\linewidth]{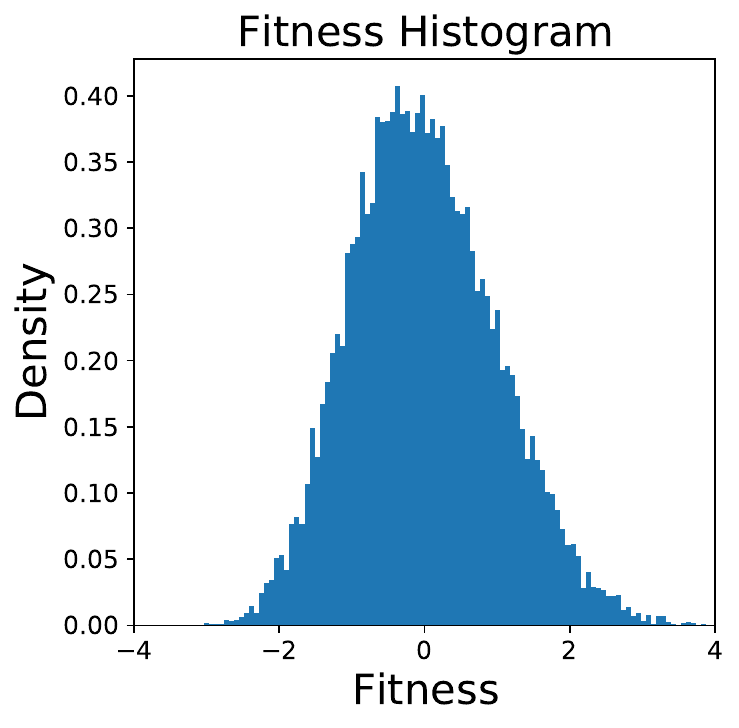}
\end{minipage}
\begin{minipage}[c]{.24\linewidth}
\includegraphics[width=\linewidth]{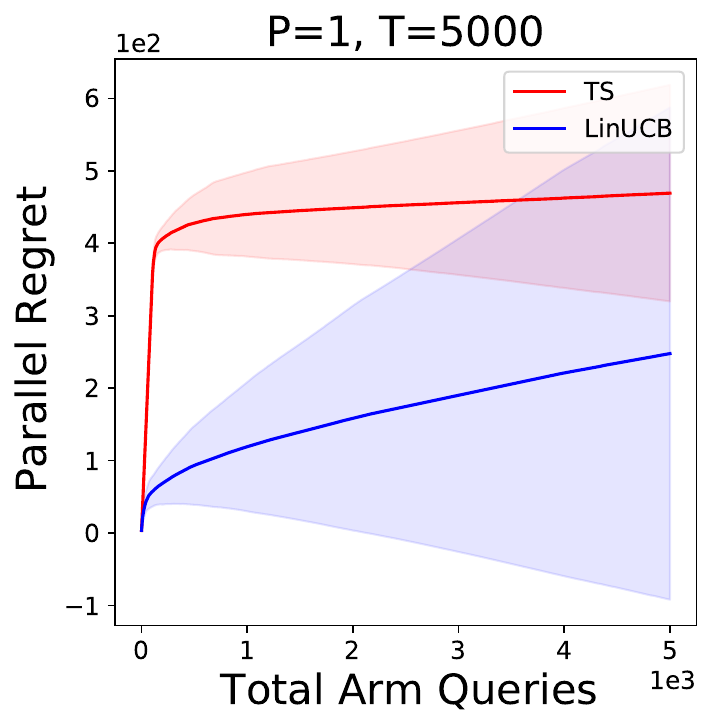}
\end{minipage}
\begin{minipage}[c]{.24\linewidth}
\includegraphics[width=\linewidth]{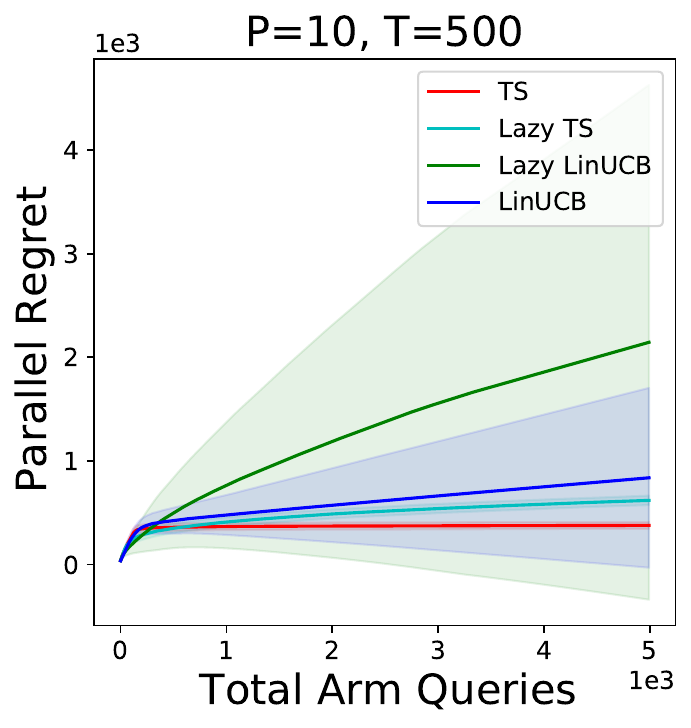}
\end{minipage}
\begin{minipage}[c]{.24\linewidth}
\includegraphics[width=\linewidth]{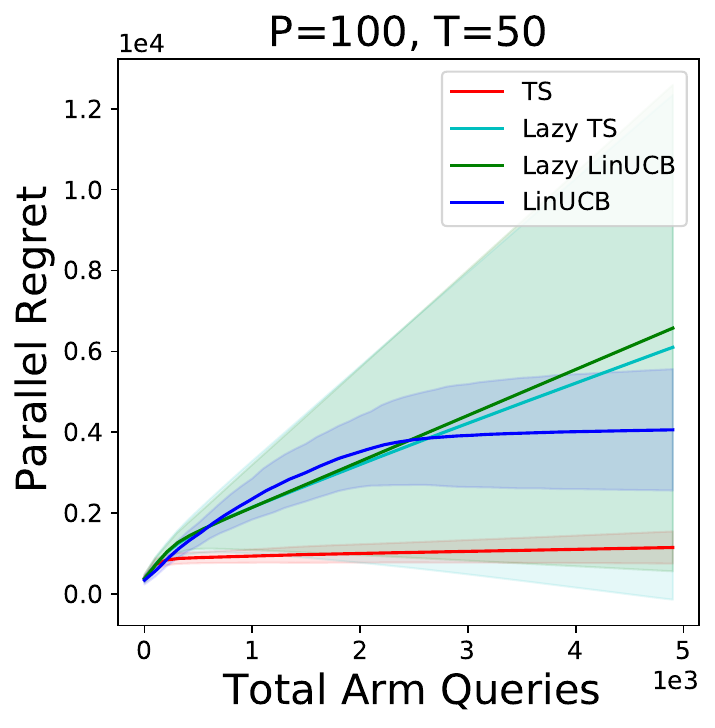}
\end{minipage}
\caption{Top Left: The histogram of fitness values for the RandomNN dataset. Top right: The parallel regret of the purely sequential setting for 5000 queries with a noise standard deviation of $0.5$. Bottom Left: The parallel regret for $P=10$. Bottom Right: The parallel regret for $P=100$. The mean regret and standard deviation are plotted as the solid line and shaded region in all plots.
}
\label{fig:randomnn}
\end{figure}

\subsection{Superconductor Data}

To assess the utility of the parallel bandit algorithms in a realistic setting we constructed a semi-synthetic problem using the UCI dataset in \citet{hamidieh2018data} consisting of a collection of superconducting materials along with their maximum superconducting temperature. The dataset consists of $m=21, 263$ superconducting materials, each with a $d=81$-dimensional feature vector, $x_i$, containing relevant attributes of the materials chemical constituents and a superconducting critical temperature $y_i$. We construct a finite-armed bandit oracle over the $m$ arms which returns a reward $r_i = y_i + \epsilon_i$ for $\epsilon_i \sim \cN(0, 100^2)$ (since $\max_i y_i = 185.0$). The  task in this example is to find the best superconducting material (or arm as measured by $y_i$) given access to a total number of arm queries $\ll m$. 



\begin{figure}[!hbt]
\centering
\begin{minipage}[c]{.24\linewidth}
\includegraphics[width=\linewidth]{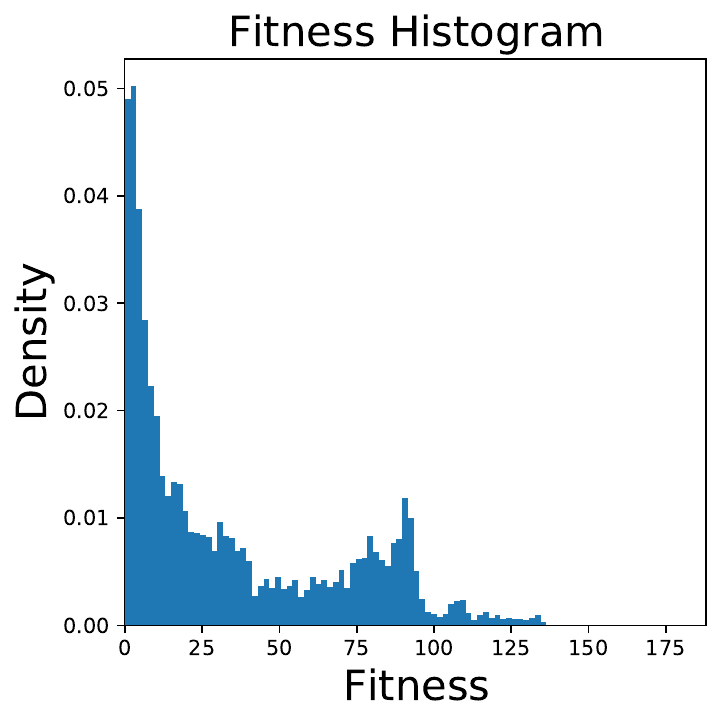}
\end{minipage}
\begin{minipage}[c]{.24\linewidth}
\includegraphics[width=\linewidth]{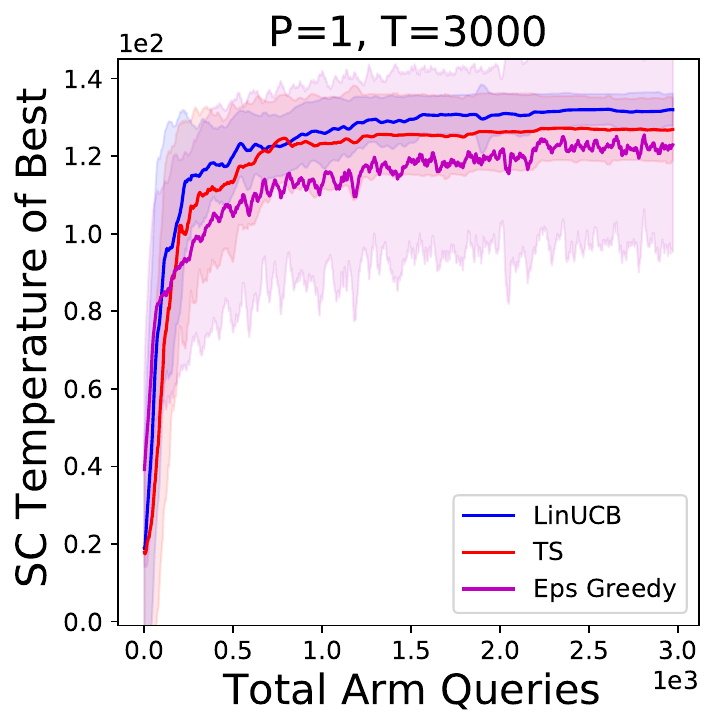}
\end{minipage}
\begin{minipage}[c]{.24\linewidth}
\includegraphics[width=\linewidth]{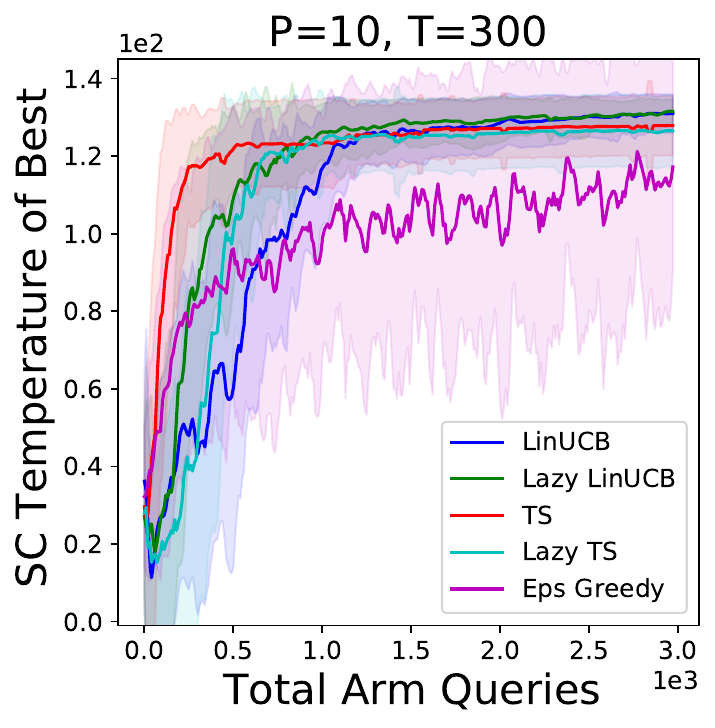}
\end{minipage}
\begin{minipage}[c]{.24\linewidth}
\includegraphics[width=\linewidth]{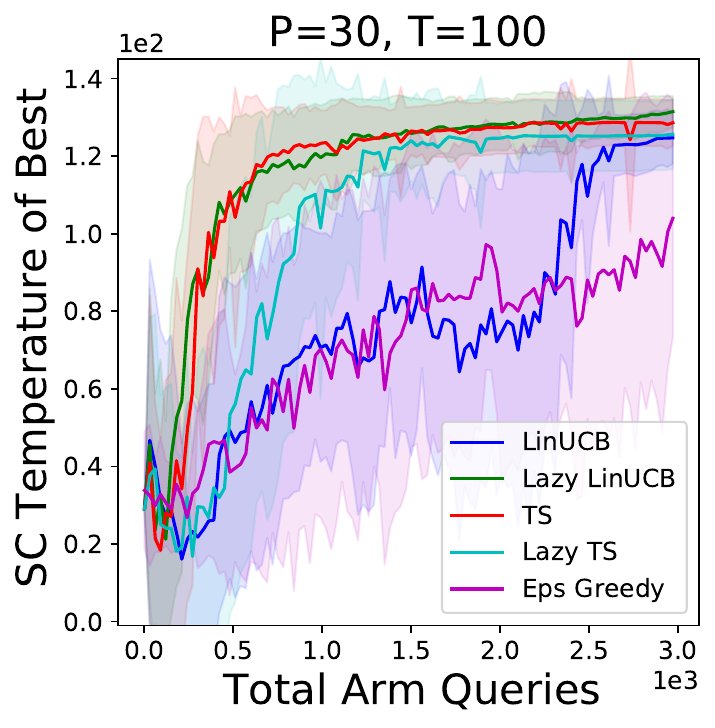}
\end{minipage}
\caption{Leftmost: Fitness Histogram of Landscape. Left to right: Regret of all algorithms for $P=1, 10, \text{ and } 30$, respectively. Here the best superconducting material (by temperature) as determined by the algorithm at the time is displayed. Curves are also smoothed by a moving-average over a window of size 30 for clarity.
}
\label{fig:superconductor}
\end{figure}

As \cref{fig:superconductor} shows although $\epsilon$-greedy is a simple algorithm, it can achieve reasonable performance (at the cost of high variance), when $P=1$. Indeed prior work has shown that other greedy (linear bandit) algorithms are formidable baselines in setting with diverse covariates \citep{bayati2020unreasonable}. However, in our setting, it is still outperformed by all the linear bandit algorithms studied herein. We also see all algorithms quickly saturate to find superconducting materials with temperatures $y_i \approx 120$.

In the cases of $P=10$, and $P=30$ we see all the parallel variants of the linear algorithms studied herein achieve non-trivial parallelism gains; that is the number of sequential rounds needed to discover this best material does not scale linearly with $P$ for any of the methods. Remarkably, Thompson sampling suffers almost no loss in performance even when $P=30$ in this setting with real data where model misspecification is in full force. Thompson sampling outperforms all other algorithms when $P=10$ and $P=30$. As our results show, explicitly introducing diversity into the selection of actions provides value in this setting. 

\subsection{Transcription Factor Binding}
In order to evaluate the effectiveness of the family of proposed parallelized linear bandit algorithms in a realistic biological sequence design setting, we utilized a fully characterized experimental transcription factor binding affinity dataset from \cite{barrera2016survey} (using the software package in \cite{sinai2020adalead}). Changes in transcription factor binding affinity has been shown to have impact on gene regulatory function and subsequently is associated with disease risk. The dataset experimentally characterizes the binding affinity of all possible $8$-length DNA sequence motifs ($m = 4^8 = 65,536$) to a transcription factor DNA binding domain providing a good benchmark for our bandit methods in a real biological application with the combinatorial structure common in biological sequence design. In this setting the number of arms $m$ is $O(\exp(d))$.

Often in biological sequence design problems quadratic features are used to model pairwise interactions (referred to as epistasis in biology). We compared linear features with random ReLU features and a quadratic kernel and found linear features work best in and provide additional insight in \cref{sec:app_exp}.

The fully characterized landscape allows for exact computation of parallel regret and analysis of the impact of realistic forms of model misspecification. Each arm $x_i$ was one-hot encoded. The scaled binding affinity $y_i \in [0,1]$ measured the binding of the arm $x_i$ for the SIX6 REF R1 transcription factor target. The finite-armed bandit oracle as in the superconductor setting was modeled as $r_i = y_i + \epsilon_i$ where $\epsilon_i \sim \cN(0, 0.3^2)$. The task for this application is to find the sequence with the highest transcription factor binding affinity.

In evaluating the best arm reported across varying levels of parallelism, we can see that LinUCB and $\epsilon$-greedy consistently performs the worst with the other three algorithms (Lazy LinUCB, Thompson sampling, and Lazy Thompson sampling) perform comparably. This implies that diversity of arms is important particularly when the model is misspecified. Note that the right tail of the fitness histogram is rather heavy for this task such that getting to an arm with fitness above $0.9$ is rather simple in the noiseless setting and can be optimized in few arm queries as was shown in \cite{angermueller2020population}. However, since biological experiments often have large experimental error we add noise with standard deviation of $0.3$ making the problem significantly harder leading to worse performance of algorithms preventing $\epsilon$-greedy from beating the $0.9$ threshold at all. Similarly, we find the same relative performance of methods in terms of parallel regret as shown in \cref{sec:app_exp}.

\begin{figure}[!hbt]
\centering
\begin{minipage}[c]{.24\linewidth}
\includegraphics[width=\linewidth]{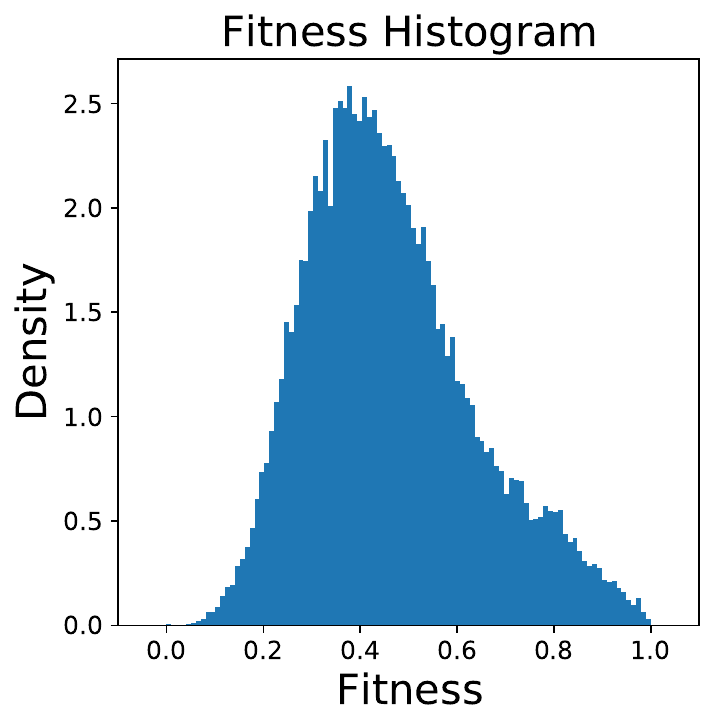}
\end{minipage}
\begin{minipage}[c]{.24\linewidth}
\includegraphics[width=\linewidth]{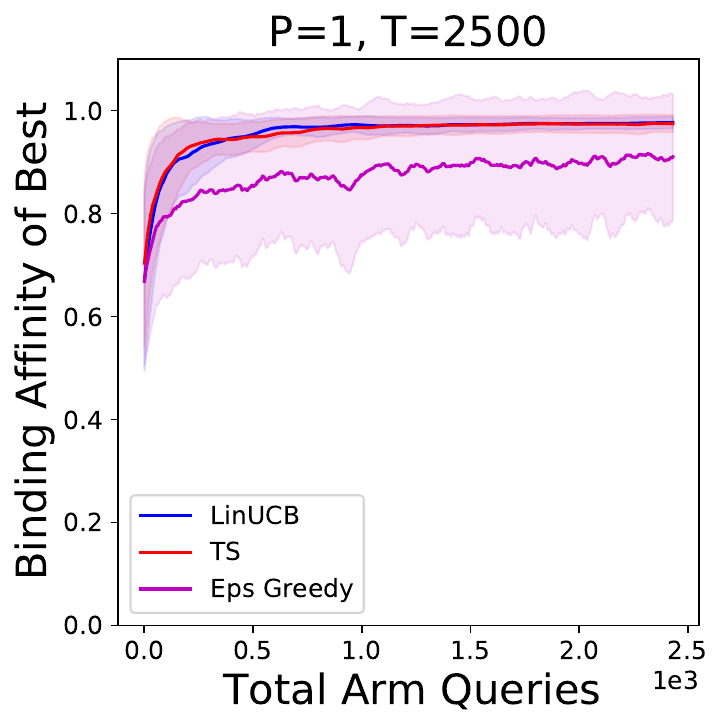}
\end{minipage}
\begin{minipage}[c]{.24\linewidth}
\includegraphics[width=\linewidth]{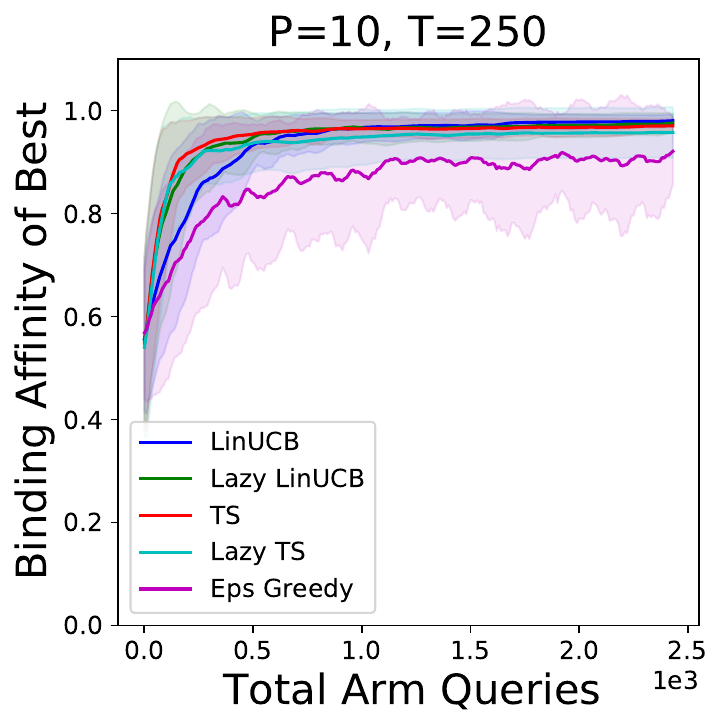}
\end{minipage}
\begin{minipage}[c]{.24\linewidth}
\includegraphics[width=\linewidth]{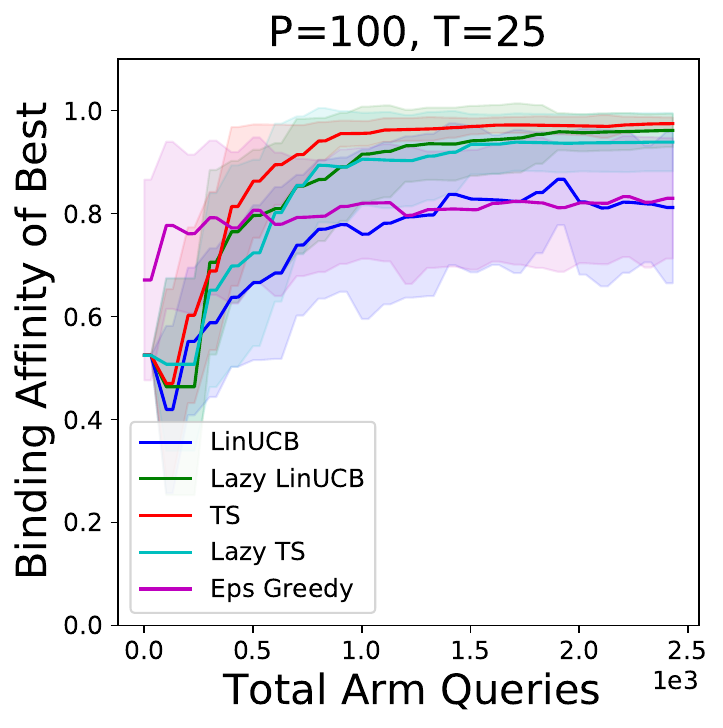}
\end{minipage}
\caption{TFBinding best arm with linear features. Leftmost: The fitness distribution of the dataset. From left to right: The best smoothed binding affinity for each round with error bars indicating standard deviation with $P=1, 10, \text{ and } 100$, respectively.
}
\label{fig:tfbinding}
\end{figure}

%% file: conclusion.tex
\section{Conclusions}
We have presented several parallel (contextual) bandit algorithms inspired by the optimism principle, and also investigated their specializations and extensions in the linear setting. In the general nonlinear setting we provided a general result showing how to characterize the complexity of learning in the parallel setting using the eluder dimension. In the linear setting, we showed how the corresponding notion  of covariance stability, provides a unified analysis of the linearUCB, lazy linearUCB, and Thompson sampling algorithms. Overall our regret upper bounds establish that the performance of these algorithms is nearly identical to that of their sequential counterparts (with the same total number of arm queries), up to a burn-in term which may depend on the geometry of the context set. Finally, we show that the parallelism gains suggested by our theory can also be achieved in several real datasets motivated by practical design problems and demonstrate the importance of diversity in problems that contain model misspecification.

Interesting directions for future work including extending the results herein to a suitably defined notion of best-arm identification. Similarly, understanding the impact of parallelism in simple, greedy heuristic algorithms (which nonetheless perform well in practice \citep{bayati2020unreasonable, sinai2020adalead}) is another important direction. Another interesting direction, would be to extend our techniques (namely characterizing the distributional eluder dimension in the parallel setting) to the setting of full reinforcement learning. 
Lastly, in applications such as protein engineering, it is often of interest to discover a diverse set of high-reward sequences under suitable notions of diversity.

%% file: acknowledgements.tex
\section{Acknowledgements}
This research is supported in part by the NIH R21-AI137433 and R35-GM134922 (JC and YSS). NT thanks the RISELab for it generous support.

%% file: app_par_nonlin_algs.tex
\section{Proofs for Regret Upper Bounds for Eluder Dimension}

We first prove \cref{thm:nonlin_regret_bound}. For the following we find it useful to define the width function of a set $\tilde{\cF} \subset \cF$ at an action $\x$:
\begin{align*}
    w_{\tilde{\cF}}(\x) = \sup_{f, f' \in \tilde{\cF}} \abs{f(\x)-f'(\x)}
\end{align*}

\begin{proof}[Proof of \cref{thm:nonlin_regret_bound}]
    We first condition on the event in \cref{lem:conf_set_nonlin_uniform}, $\tf \in \bigcap_{t=1, p=1}^{\infty, P} \cF_{t,p}$, which holds probability at least $1-\delta$. For notational convenience we denote the action upper and lower bounds as $U_{t,p}(\x) = \sup \{ f(\x) : f \in \cF_{t,p} \}$ and $L_{t,p}(\x) = \inf \{ f(\x) : f \in \cF_{t,p} \}$. Hence if the best approximant satisfies $\tf \in \cF_{t,p}$ we have that $L_{t,p}(\x) \leq \tf(\x) \leq  U_{t,p}(\x)$ for all $\x$. On this event the regret of \cref{algo:par_nonlin_lazy_ucb} can be decomposed as,
    \begin{align*}
        & \fstar(\x_{t,p}^\star) - \fstar(\x_{t,p}) = \fstar(\x_{t,p}^\star) - \tf(\x_{t,p}^\star) + \tf(\x_{t,p}) -\fstar(\x_{t,p})  + \tf(\x_{t,p}^\star) - \tf(\x_{t,p}) \leq \\ & 2\epsilon + U_{t,p}(\x_{t,p}^\star) - L_{t,p}(\x_{t,p}) = 2 \epsilon + w_{\cF_{t,p}}(\x_{t,p}) + U_{t,p}(\x_{t,p}^\star) - U_{t,p}(\x_{t,p}) \leq 2\epsilon+ w_{\cF_{t,p}}(\x_{t,p}).
    \end{align*}
    

    where the first inequality by the misspecification bound and the last inequality follows by appealing to the optimism principle.
    The conclusion follows by summing over $t,p$ to obtain,
    \begin{align*}
        \cR(T,P) \leq O \left(\epsilon TP + \sum_{t=1,p=1}^{T,  P} w_{\cF_{t,p}} (\x_{t,p}) \right)
    \end{align*}
    and applying \cref{lem:nonlin_width_bound}. An identical argument shows \cref{algo:par_nonlin_ucb},
    \begin{align*}
     \fstar(\x_{t,p}^\star) - \fstar(\x_{t,p}) \leq 2\epsilon +  w_{\cF_{t,1}}(\x_{t,p})
    \end{align*} and summing shows 
    \begin{align*}
        \cR(T,P) \leq O \left(\epsilon TP + \sum_{t=1,p=1}^{T,  P} w_{\cF_{t,1}} (\x_{t,p}) \right).
    \end{align*}
Applying \cref{lem:nonlin_width_bound} gives the conclusion.
\end{proof}

\subsection{Bounding the Sum of widths}

We first bound the sum of the widths across the confidence sets. The argument follows similarly to that in \citet{russo2013eluder}.
\begin{lemma}
\label{lem:nonlin_width_bound}
    Let $\beta_{t,p}$ for $t \in \mathbb{N}$, $p \in [P]$ be a nondecreasing sequence, $\cF_{t,p}$ be defined as in \cref{eq:conf_set_nonlin}, and $\alpha_{T,P}^\cF = \frac{B}{TP}$. Then, almost surely,
    \begin{align*}
        \sum_{t=1}^{T} \sum_{p=1}^{P} w_{\cF_{t,p}}(\x_{t,p}) \leq \min O(\{ BTP, BPd + \sqrt{d \beta_{T,P} TP} \}).
    \end{align*}
    Similarly,
    \begin{align*}
        \sum_{t=1}^{T} \sum_{p=1}^{P} w_{\cF_{t,1}}(\x_{t,p}) \leq \min O(\{ BTP, BPd + \sqrt{d \beta_{T,1} TP} \})
    \end{align*}
\end{lemma}
\begin{proof}
    For ease of notation throughout we write $d = \dimeluder{\cF, \frac{1}{T}}$ and $w_{t,p}(\cdot) = w_{\cF_{t,p}}(\cdot)$. We first re-order the sequence in descending order $(w_{1,1}, \hdots, w_{T,P}) \to (w_{i_1}, w_{i_2}, \hdots, w_{i_{TP}})$. We have,
    \begin{align*}
      \sum_{t=1}^{T} \sum_{p=1}^{P} w_{\cF_{t,p}} = \sum_{k=1}^{TP} w_{i_k} = \sum_{k=1}^{TP} w_{i_k} \Ind[w_{i_k} \leq  \alpha_{T,P}^{\cF}] + \sum_{k=1}^{TP} w_{i_k} \Ind[w_{i_k} >  \alpha_{T,P}^{\cF}] \leq B + \sum_{k=1}^{TP} w_{i_k} \Ind[w_{i_k} >  \alpha_{T,P}^{\cF}]
    \end{align*}
    where the final step uses the fact that either $\alpha_{TP}^{\cF} = \frac{B}{TP}$ so $\sum_{k=1}^{TP} w_{i_k} \Ind[w_{i_k} \leq  \alpha_{TP}^{\cF}] \leq B$. By definition $w_{i_k} \leq 2B$. Additionally applying \cref{cor:sum_widths_e} shows that for $w_{i_k} > \frac{B}{TP}$ we have that
    \begin{align*}
        k \leq O \left(\frac{\beta_{T,P}}{w_{i_k}^2}+P \right) \dimeluder{\cF}{w_{i_k}} \leq O \left (\frac{\beta_{T,P}}{w_{i_k}^2}+P \right) \dimeluder{\cF}{\frac{B}{TP}}
    \end{align*}
    where the final step follows since $\dimeluder{\cF}{\epsilon'}$ is nonincreasing in $\epsilon'$. So when, $w_{i_k} \geq \frac{B}{TP}$, $w_{i_k} \leq O( \min \{B, \sqrt{\frac{\beta_{T,P}d}{k-Pd}} \})$.
    Therefore we have that,
    \begin{align*}
        & \sum_{k=1}^{TP} w_{i_k} \Ind[w_{i_k} >  \alpha_{T,P}^{\cF}] \leq O(B Pd) + O \left(\sum_{k=Pd+1}^{TP} \sqrt{\frac{d\beta_{T,P}}{k-Pd}} \right) \leq O(BPd +\sqrt{d \beta_{T,P}} \int_{0}^{TP-Pd} \frac{1}{\sqrt{k}} dk) \\
        & \leq O( BPd + \sqrt{d \beta_{T,P}} \sqrt{TP})
    \end{align*}
    Finally, using the fact the sum of widths is always bounded $O(TB)$ gives the conclusion that,
    \begin{align*}
    \sum_{t=1}^{T} \sum_{p=1}^{P} w_{\cF_{t,p}} \leq \min O(\{ BTP, BPd + \sqrt{d \beta_{T,P} TP} \})
    \end{align*}
    A similar argument shows that,
    \begin{align*}
    \sum_{t=1}^{T} \sum_{p=1}^{P} w_{\cF_{t,1}} \leq \min O(\{ BTP, BPd + \sqrt{d \beta_{T,1} TP} \})
    \end{align*}
\end{proof}

The generic result bounding the sum of widths in the context of the distributional eluder dimension can be directly specialized to handle the setting of contextual bandits. 

\begin{corollary}
\label{cor:sum_widths_e}
Let $\beta_{t,p}$ be a non-decreasing sequence and $\cF_{t,p}$ be defined in \cref{eq:conf_set_nonlin}, and $w_{t,p}(\x) = w_{\cF_{t,p}}(\x)$. Then we have that for all $t \in [T]$,
\begin{align*}
    \sum_{t=1}^{T} \sum_{p=1}^{P} \Ind[w_{t,p}(\x_{t,p}) > \epsilon] \leq 6 \left(\left(\frac{4\beta_{T,P}}{\epsilon^2} + P \right) \dimeluder{\cF}{\epsilon} \right)
\end{align*}
with probability 1.
Similarly, if $w_{t,p}(\x) = w_{\cF_{t,1}}(\x)$ an identical statement holds,
\begin{align*}
    \sum_{t=1}^{T} \sum_{p=1}^{P} \Ind[w_{t,p}(\x_{t,p}) > \epsilon] \leq 6 \left(\left(\frac{4\beta_{T,1}}{\epsilon^2} + P \right) \dimeluder{\cF}{\epsilon} \right)
\end{align*}
with probability 1.
\end{corollary}
\begin{proof}
The results follows by instantiating \cref{prop:sum_widths_de}. We prove the first result since the second follows identically. To instantiate the result we first let the sequence of measures $\{ \mu_{a,b} \}_{a=1, b=1}^{T, P}$ be equal to the Dirac measures located at the fixed actions -- $\mu_{t,p} = \delta_{\x_{t,p}}$. Now fix some sequence of pairs of functions $\{ (f_{t,p}, f_{t,p}') \}_{t=1,p=1}^{T,P}$ such that $f_{t,p}, f_{t,p}' \in \cF_{t,p}$. We set $\phi_{t,p} = f_{t,p}-f'_{t,p}$.  We must now certify the bound on $\sum_{a=1, b=1}^{t-1,p} (\mE_{\mu_{a,b}}[\phi_{t,p}])^2$. From the definitions and the triangle inequality we have that
\begin{align*}
    & \sum_{a=1, b=1}^{t-1,p-1} (\mE_{\mu_{a,b}}[\phi_{t,p}])^2 \leq \sum_{a=1, b=1}^{t-1,p-1} (f_{t,p}(\x_{a,b})-f_{t,p}'(\x_{a,b}))^2] \leq  2 \sum_{a=1, b=1}^{t-1,p-1} (f_{t,p}(\x_{a,b})-\hat{f}_t(\x_{a,b}))^2+ (\hat{f}_t(\x_{a,b})-f_{t,p}'(\x_{a,b}))^2]= \\
    & 2(\norm{f-\hat{f}_t}_{t,p}^2 + \norm{f'-\hat{f}_t}_{t,p})^2) 
    \leq 4 \beta_{t,p}.
\end{align*}
which holds for all $t \in [T], p \in [P]$. Hence applying \cref{prop:sum_widths_de} gives that
\begin{align*}
    \sum_{t=1}^{T} \sum_{p=1}^{P} \Ind[f(\x_{t,p}) - f'(\x_{t,p})> \epsilon] \leq O \left(\left(\frac{\beta_{T,P}}{\epsilon^2} + P \right) \dimeluder{\cF}{\epsilon} \right), \forall \{ (f_{t,p}, f_{t,p}') \}_{t=1,p=1}^{T,P} \text{ in }  \cF_{t,p}
\end{align*}
with probability 1. Since the statement holds uniformly $\forall \{ (f_{t,p}, f_{t,p}') \}_{t=1,p=1}^{T,P}$  in  $\cF_{t,p}$, taking the supremum over both sides gives the conclusion,
\begin{align*}
    \sum_{t=1}^{T} \sum_{p=1}^{P} \Ind[w_{t,p}(\x_{t,p}) > \epsilon] \leq 6 \left(\left(\frac{4\beta_{T,P}}{\epsilon^2} + P \right) \dimeluder{\cF}{\epsilon} \right)
\end{align*}
Note for the second statement we can pick the non-decreasing sequence of $\beta_{t,p}$ to match the choice appropriate for the widths in the second statement.
\end{proof}

The next series of results we show hold for the generalized notion of sequential complexity called the distributional eluder dimension introduced in \citet{jin2021bellman}.

\begin{definition}
[Distributional Independence and Distributional Eluder Dimension]
The generalization of the Eluder dimension follows as:

\begin{itemize}
\item A probability measure $\mu$ is $\epsilon$-dependent on a sequence of probability measures $\{\mu_i \}_{i=1}^{n}$ if any $\phi \in \Phi$ satisfying $\sqrt{\sum_{i=1}^n (\mE_{\mu_i}[\phi])^2} \leq \epsilon$ also satisfies $\abs{\mE_{\mu}[\phi]} \leq \epsilon$. A measure $\mu$ is epsilon-independent of $\{\mu_i \}_{i=1}^{n}$ with respect to $\Phi$ if $\mu$ is not $\epsilon$-dependent on $\{\mu_i \}_{i=1}^{n}$.
\item Consider a function class $\Phi$ on $\mathcal{X}$ and $\Pi$, a family of probability measures on $\mathcal{X}$. The $\epsilon$-distributional eluder dimension $\dimdeeluder{\Phi}{\Pi}{\epsilon}$ is the length of the longest sequence of measures in $\{ \mu_i \}_{i=1}^{n} \subset \Pi$ such that for some $\epsilon' \geq \epsilon$, every element is $\epsilon'$-independent of its predecessors.
\end{itemize}
\end{definition}

We describe sequences of measures (or actions) as being contained in $D_t = \bigcup_{A=1}^{t} \otimes_{a=1}^{A} B_{a}$
where $B_a$ denotes a batch of $P$ measures (or actions) at a time $a$. Using this we can prove the following result. This generalizes the corresponding result of \citet{russo2013eluder} to the parallel setting and to hold for the generalized distributional eluder dimension. It relies on a novel construction to show the existence of parallel $\epsilon$-dependent sequences using Hall's marriage theorem. We believe this result may find useful in showing similar results for the parallel RL setting where the sequential analog of the result is used.

\begin{proposition}
\label{prop:sum_widths_de}
Consider a function class $\Phi$ defined on $\cX$, and a family of probability measures $\Pi$ on $\cX$. Given a sequence of functions $\{ \phi_{a,b} \}_{a=1, b=1}^{T, P} \subset \Phi$ and measures  $\{ \mu_{a,b} \}_{a=1, b=1}^{T, P} \subset \Pi $ that satisfy for all $t \in [T], p \in [P]$, $\sum_{a=1, b=1}^{t-2,P} (\mE_{\mu_{a,b}}[\phi_{t,p}])^2 + \sum_{b=1}^{p-1} (\mE_{\mu_{t-1,b}}[\phi_{t,p}])^2\leq \beta_{t,p}$ for a non-decreasing sequence $\beta_{t,p}$\footnote{The sequence should be non-decreasing in both $t$ and $p$.} we have that for all $t \in [T]$,
\begin{align*}
    \sum_{t=1}^{T} \sum_{p=1}^{P} \Ind[\abs{\mE_{\mu_{t,p}}[\phi_{t,p}]} > \epsilon] \leq  6\left(\left(\frac{\beta_{T,P}}{\epsilon^2} + P \right) \dimdeeluder{\Phi}{\Pi}{\epsilon} \right)
\end{align*}
\end{proposition}
\begin{proof}
We first bound the number of distinct measure sequences in $D_{t-1}$ that a measure $\mu_{t,p} \in B_t$ can be $\epsilon$-dependent on when $\abs{\mE_{\mu_{t,p}}[\phi_{t,p}]} > \epsilon$. By the definition of the distributional $\epsilon$-independence, if $\abs{\mE_{\mu_{t,p}}[\phi_{t,p}]} > \epsilon$ and $\mu$ is $\epsilon$-dependent on a sequence $S \subset D_{t-1}$ then it must be the case $\sum_{\nu \in S} (\mathbb{E}_{\nu}[\phi_{t,p}])^2 \geq \epsilon^2$ (to see this consider the contrapositive). Hence if $\mu_{t,p}$ (satisfying $\abs{\mE_{\mu_{t,p}}[\phi_{t,p}]} > \epsilon$) is $\epsilon$-dependent on $K$ disjoint sequences, $S_1, \hdots, S_K$ in $D_{t-1}$ then it must be the case that 
\begin{align*}
\beta_{t,p} \geq \sum_{a=1}^{t-1} \sum_{b=1}^{p-1} (\mE_{\mu_{a,b}}[\phi_{t,p}])^2 \geq K \epsilon^2 \implies K \leq \frac{\beta_{t,p}}{\epsilon^2} \leq \frac{\beta_{T,P}}{\epsilon^2} 
\end{align*}
Now we prove a lower bound on $K$. We work in the general setting where we consider a batched sequence of measures $\{\bar{\mu}_{a, b} \}_{a \in [\tau], b\in [P_i]}$. Similar to before, we use the notation $\bar{B}_a = \{ \bar{\mu}_{a, b} \}_{b \in [P_i]}$ to denote the $a$th batch in $\{\bar{\mu}_{a, b} \}_{a \in [\tau], b\in [P_i]}$ and $\bar{D}_t = \bigcup_{A=1}^{t} \otimes_{a=1}^{A}  \bar{B}_{a}$. For ease of notation we will let $d = \dimdeeluder{\Phi}{\Pi}{\epsilon}$ in the following.

Fix a time $\tau \in \mathbb{N}$, let $P_i \leq P$ be the batch-size of time $i$. Define $\tilde K$ to be the largest integer such that, $\tilde K d  \leq \sum_{i=1}^{\tau-1} P_i$. Therefore $\sum_{i=1}^{\tau-1} P_i -d  \leq \tilde Kd$. Now we show there exists a batch number $\ell \leq \tau$ and a within-batch index $i$ such that $\mu_{l,i}$ is $\epsilon$-dependent on a subset  of disjoint sequences of size at least $\tilde{K}/2$ of a set of $\tilde{K}$ disjoint sequences $\bar{S}_1, \hdots, \bar{S}_{\tilde{K}} \in \bar{D}_t$


We now constructively build the sequences $\bar{S}_1, \hdots, \bar{S}_{\tilde{K}}$ by setting the first element in each $\bar{S}_i$ equal to $i$th element in the set $\{\bar{\mu}_{a, b} \}_{a \in [\tau], b\in [p]}$ lexigraphically. Since the argument will proceed iteratively, we will denote the running batch index in the argument as $k$--which initialized to $k= \argmin \{ j \text{ s.t. } \sum_{i=1}^j P_i \geq \tilde{K} \}$ at the beginning of the construction.
Now if there existed a measure $\mu \in \bar{B}_{k+1}$ such that it is $\epsilon$-dependent on at least $\tilde{K}/2$ of the sequences $\{ \bar{S}_i \}_{i=1}^{\tilde{K}}$ the result follows. Hence we consider the case where $\mu \in \bar{B}_{k+1}$ is $\epsilon$-independent of at least $\tilde{K}/2$ sequences.

We now seek to show the sequences $\{ \bar{S}_i \}_{i=1}^{\tilde{K}}$ can be expanded using elements from $\bar{B}_{k+1}$ while preserving their $\epsilon$-independence. To this end
consider a bipartite graph with vertex sets consisting of $S_1, \hdots, S_{\tilde{K}}$ and the elements in $\bar{B}_{k+1}$. Let an edge exist between $\mu \in \bar{B}_{k+1}$ and $\bar{S}_j \in \{ \bar{S}_i \}_{i=1}^{\tilde{K}}$ if $\mu$ is $\epsilon$-independent of $\bar{S}_j$. If for each $\mu \in \bar{B}_{k+1}$ there are at least $\tilde{K}/2$ sequences in $\{ \bar{S}_i \}_{i=1}^{\tilde{K}}$ for which each $\mu$ is $\epsilon$-independent of and $\frac{\tilde{K}}{2} \geq P \geq P_{k+1}$, \cref{lem:graph_lemma} implies the existence of a matching such that
there is a $\bar{B}_{k+1}$-perfect matching to nodes in $\{ \bar{S}_i \}_{i=1}^{\tilde{K}}$ of size $P_{k+1}$. Note that by taking $\tau$ such that $\sum_{i=1}^{\tau-1} P_i \geq 2Pd+d$ we can always ensure $\frac{\tilde{K}}{2} \geq P$ so the matching always exist.

Then, if the index $k=\tau$ at this iteration it must be the case that $\sum_{i=1}^{\tau-1} P_i$  have been accommodated in $\tilde{K}$ sequences. By the definition of $\tilde{K} \leq (\sum_{i=1}^{\tau-1} P_i)/d$ it must be the case that each $\bar{S}_i$ satisfies $\abs{\bar{S}_i} \geq d$ to accommodate the $\sum_{i=1}^{\tau-1} P_i$ points. However since each element in each sequence is  $\epsilon$-independent of its predecessors we must have $\abs{\bar{S}_i}=d$. Hence once the time $\tau$ is reached there must exist a measure $\mu_{l,i} \in \bar{B}_{\tau}$ such that is $\epsilon$-dependent on at least $\tilde{K}/2 \geq \frac{\sum_{i=1}^{\tau-1} P_i-d}{2d}$ subsequences.

Combining the results of the previous two points if $
\sum_{i=1}^{\tau-1} P_i \geq 2Pd+d$                there must exist a batch number $\ell$ and within batch index $i$ such that $\mu_{\ell, i}$ is $\epsilon$-dependent on at least $\tilde{K}/2 \geq \frac{\sum_{i=1}^{\tau-1} P_i-d}{2d}$ sequences. Hence if $\sum_{i=1}^{\tau-1} P_i \geq 2Pd+d$ and  $\frac{\tilde{K}}{2} \geq \frac{\sum_{i=1}^{\tau-1} P_i-d}{2d} \geq \frac{\beta_{T,P}}{\epsilon^2}+1$ we must have $\epsilon$-dependence for at least $\frac{\beta_{T,P}}{\epsilon^2}$ sequences. Both of these conditions can be guaranteed by taking $\tau$ such that $\sum_{i=1}^{\tau-1} P_i \geq \max\left( 2Pd+d , \frac{2\beta_{t,p}d}{\epsilon^2}+2d \right)$. Since $Pd \geq d$ both of these conditions hold provided $\sum_{i=1}^{\tau-1} P_i \geq  5Pd +\frac{2\beta_{T,P}d}{\epsilon^2}$. When this holds there must exist a measure $\mu_{l,i} $ such that it is $\epsilon$-dependent on at least $\frac{     \beta_{T,P}}{\epsilon^2}+1$ sequences in any batched sequence of measures.

We have proven that in any batch sequence of measures $\{\bar{\mu}_{a, b} \}_{a \in [\tau], b\in [P_i]}$ where $\tau $ satisfies $\sum_{i=1}^{\tau-1} P_i \geq  5Pd +\frac{2\beta_{T,P}d}{\epsilon^2}$ there must exist a measure $\mu_{l,j} $ with $l \leq \tau$ and $j\leq P_l$ such that it is $\epsilon$-dependent on at least $\frac{     \beta_{T,P}}{\epsilon^2}+1$ sequences.


We can now apply the result to the series of measures $\mu_{a,b}$ for $a \in [T], b \in [P]$, and in particular to any such that $\abs{\mE_{\mu_{t,p}}[\phi_{t,p}]} > \epsilon$. The previous result imply that if  $\sum_{t=1}^{T} \sum_{p=1}^{P} \Ind[\abs{\mE_{\mu_{t,p}}[\phi_{t,p}]} > \epsilon] \geq 5Pd + 2 \frac{\beta_{t,p}d}{\epsilon^2}$ there must exist a measure $\mu$ contained in the set of measures satisfying $\abs{\mE_{\mu_{t,p}}[\phi_{t,p}]} > \epsilon$ that is $\epsilon$-dependent on at least $\frac{ \beta_{T,P}}{\epsilon^2}+1$ sequences (note the sum $\sum_{t=1}^{T} \sum_{p=1}^{P} \Ind[\abs{\mE_{\mu_{t,p}}[\phi_{t,p}]} > \epsilon]$ corresponds to the sum of the $P_i$s in the previous argument). However this is impossible by the original argument, so the conclusion 
\begin{align*}
    \sum_{t=1}^{T} \sum_{p=1}^{P} \Ind[\abs{\mE_{\mu_{t,p}}[\phi_{t,p}]} > \epsilon] \leq O\left(\frac{\beta_{T,P}d}{\epsilon^2} + Pd \right)
\end{align*}
follows.
\end{proof}
\begin{lemma}
\label{lem:graph_lemma}
If G is a bipartite graph with vertex sets $A, B$ such that $|A| = K$ and $|B|=P$. If for all nodes $v \in B$, $\abs{N(v)} \geq K/2$ then there exists a $B$-perfect matching; that is there is a perfect matching between the nodes in B to nodes in A.
\end{lemma}
\begin{proof}
    This result follows immediately from an application of the Hall's marriage theorem.
\end{proof}

\subsection{Constructing Confidence Sets}
We recall several pieces of notation introduce in the main text. The action-induced empirical 2-norm is $\norm{g}_{2,t,p}^2 = \sum_{a=1}^{t-1} \sum_{b=1}^{P} g(\x_{a,b})^2 + \sum_{b=1}^{p-1} g(\x_{t,b})^2$. We also consider the empirical squared loss $L_{2, t} = \sum_{a=1}^{t-1} \sum_{p=1}^{P} (f(\x_{a,b})-r_{a,b})^2$ and the corresponding least-squares minimizer over some base function class $\cF$ as $\hat{f}_t = \argmin_{f \in \cF_{t}} L_{2, t}$. We then consider the confidence sets:
\begin{align}
\label{eq:conf_set_nonlin}
    \cF_{t,p} = \{ f \in \cF : \norm{f-\hat{f}_t}_{t, p} \leq \sqrt{\beta_{t,p}(\cF, \alpha, \delta, \epsilon)} \}
\end{align}
We also recall the best approximant to the true function $\fstar$ in the function class $\cF$ in the sup-norm $\tf$. The arguments herein are modifications of the corresponding ones in \citet{russo2013eluder} to handle parallelism and model misspecification.

Using this we construct a series of tight confidence sets,
\begin{lemma}
\label{lem:conf_set_nonlin}
For any $\delta > 0$ and $f : \mathcal{X} \to \mathbb{R}$,
\begin{align*}
     L_{2,t}(f) \geq L_{2,t}(\tf)  + \frac{1}{4} \norm{\tf-f}_{t,p}^2 - 8 \epsilon^2 TP - 8 R^2 \log(\frac{1}{\delta}), \forall t \in \mathbb{N}
\end{align*}
with probability at least $1-\delta$.
\end{lemma}
\begin{proof}
The proof follows similarly to the proof of \citet{russo2013eluder}, except we also account for misspecification bias. Let $\cH_{t,p}$ be the $\sigma$-algebra of all events up to time $t$, and parallel batch index $p$. Expanding the square shows that \begin{align*}
    & Z_{t,p} = (\tf(\x_{t,p})-r_{t,p})^2-(f(\x_{t,p})-r_{t,p})^2 = \\ & (\tf(\x_{t,p})-\fstar(\x_{t,p}))^2 + (\tf(\x_{t,p})-\fstar(\x_{t,p})) \xi_{t,p} - (f(\x_{t,p})-\fstar(\x_{t,p}))^2 - \xi_{t,p} (f(\x_{t,p})-\fstar(\x_{t,p})) = \\
    &  (\tf(\x_{t,p})-\fstar(\x_{t,p}))^2 - (f(\x_{t,p})-\fstar(\x_{t,p}))^2 + \xi_{t,p}(\tf(\x_{t,p})-f(\x_{t,p}))
\end{align*}
\begin{align*}
    & \mu_{t,p} = \mE[Z_{t,p} | \cH_{t,p} ] = (\tf(\x_{t,p})-\fstar(\x_{t,p}))^2 - (f(\x_{t,p})-\fstar(\x_{t,p}))^2 \\
    & \psi_{t,p}(\lambda) = \log \mE[\exp(\lambda [Z_{t,p} - \mu_{t,p}])]| \cH_{t,p}] \leq \frac{\left(2 \lambda(\tf(\x_{t,p})-f(\x_{t,p}))R \right)^2}{2}
\end{align*}
Further manipulations show that,
\begin{align*}
    & \mu_{t,p} = (\tf(\x_{t,p})-\fstar(\x_{t,p}))^2 - (f(\x_{t,p})-\fstar(\x_{t,p}))^2 = (\tf(\x_{t,p})-\fstar(\x_{t,p}))^2 - (f(\x_{t,p})-\tf(\x_{t,p}) +\tf(\x_{t,p})- \fstar(\x_{t,p}))^2 \\
    & =-(f(\x_{t,p})-\tf(\x_{t,p}))^2-2(f(\x_{t,p})-\tf(\x_{t,p}))(\tf(\x_{t,p})-\fstar(\x_{t,p})) \leq -\frac{1}{2}(f(\x_{t,p})-\tf(\x_{t,p}))^2 + 8(\tf(\x_{t,p})-\fstar(\x_{t,p}))^2
\end{align*}
using the inequality $ab \leq \frac{1}{4}a^2+4b^2$.
Now applying Lemma 4 (a martingale concentration result) in \citet{russo2013eluder} shows that for all $\lambda \geq 0$, $x \geq 0$,
\begin{align*}
    & \Pr \left(\sum_{t=1,p=1}^{T,P} \lambda Z_{t,p} \leq x  + \sum_{t=1,p=1}^{T,P}[\lambda \mu_{t,p} + \psi_{t,p}(\lambda)], \forall t \in \mathbb{N}, p \in [P] \right) \geq 1-\exp(-x) \implies \\
    & \sum_{t=1,p=1}^{T,P} Z_{t,p} \leq \frac{x}{\lambda} + \sum_{t=1,p=1}^{T,P} (\tf(\x_{t,p})-\fstar(\x_{t,p}))^2 - (f(\x_{t,p})-\fstar(\x_{t,p}))^2 +  2\lambda R^2 \left(\tf(\x_{t,p})-f(\x_{t,p}) \right)^2 \leq \\
    & \frac{x}{\lambda} + \sum_{t=1,p=1}^{T,P} 8(\tf(\x_{t,p})-\fstar(\x_{t,p}))^2  +  (2\lambda R^2-\frac{1}{2}) \left(\tf(\x_{t,p})-f(\x_{t,p}) \right)^2
\end{align*}
with probability at least $1-\exp(-x)$. Choosing $\lambda = \frac{1}{8 R^2}$ and $x=\log(\frac{1}{\delta})$ shows,
\begin{align*}
   & L_{2,t}(f) \geq L_{2,t}(\tf) - 8 \sum_{t=1, p=1}^{T,P} (\tf(\x_{t,p})-\fstar(\x_{t,p}))^2 + \frac{1}{4} \norm{\tf-f}_{t,p}^2 - 8 R^2 \log(\frac{1}{\delta}) \implies \\
   & L_{2,t}(f) \geq L_{2,t}(\tf) - 8 \epsilon^2 TP + \frac{1}{4} \norm{\tf-f}_{t,p}^2 - 8 R^2 \log(\frac{1}{\delta})
\end{align*}
\end{proof}

We also need to provide a uniform confidence band around the least-squares estimate to ensure this band contains the true underlying function with high probability:

\begin{lemma}
\label{lem:conf_set_nonlin_uniform}
For all $\delta > 0, \alpha > 0$, the sets $\cF_{t,p}$ in  \cref{eq:conf_set_nonlin} satisfy,
\begin{align*}
    \tf \in \bigcap_{t=1, p=1}^{\infty, P} \cF_{t,p}
\end{align*}
with probability at least $1-\delta$.
\end{lemma}
\begin{proof}
Let $\cF_{\alpha} \subseteq \cF$ be an $\alpha$-cover of $\cF$ in the sup-norm (i.e. for any $f \in \cF$ there exists an $\falpha \in \cF_{\alpha}$ such that $\norm{\falpha-f}_{\infty} \leq \alpha$). By a union bound, we have that
\begin{align*}
    L_{2,t}(\falpha)-L_{2,t}(\tf) \geq \frac{1}{4} \norm{\falpha-\tf}_{t,1}^2 -8\epsilon^2 TP- 8R^2 \log(\abs{\cF^{\alpha}}/\delta), \forall t \in \mathbb{N}, f \in \cF^{\alpha}
\end{align*}
Accounting for the discretization error, shows that with probability at least $1-\delta$ for all $t \in \mathbb{N}$ and $f \in \cF$,
\begin{align*}
    & L_{2,t}(f)-L_{2,t}(\tf) \geq \frac{1}{4}\norm{f-\tf}_{t,1}^2-8R^2 \log(\abs{\cF_{\alpha}}) - 8 \epsilon^2 TP+ \\
    &  \underbrace{\frac{1}{4}  \norm{\falpha - \tf}^2_{t,1}-\frac{1}{4} \norm{f-\tf}_{t,1}^2+ L_{2,t}(f) - L_{2,t}(\tf)}_{\text{Discretization Error}}
\end{align*}
where $\falpha$ is the closest function in the cover to $f$. Additionally, adding the ``lazy" terms both sides of the inequality gives,
\begin{align*}
    & L_{2,t}(f)-L_{2,t}(\tf) + \frac{1}{4}\abs{\sum_{b=1}^{p-1}(f(\x_{t,b})-\tf(\x_{t,b}))^2} \geq \frac{1}{4}\norm{f-\tf}_{t,p}^2-8R^2 \log(\abs{\cF_{\alpha}}) + \text{Discretization Error}
\end{align*}
which also holds uniformly. Now using the fact that $\abs{\sum_{b=1}^{p-1}(f(\x_{t,b})-\tf(\x_{t,b}))^2} \leq 2(p-1) B^2$ from the boundedness of $\cF$, \cref{lem:discretization_error} to control the discretization error, and taking $\hat{f}_t$ is chosen as $f$ (so as to minimize $L_{2,t}(f)$), we find:
\begin{align*}
    \frac{1}{4} \norm{\hat{f}_t-\tf}_{t,p}^2 \leq 8R^2 \log(\abs{\cF_{\alpha}}/\delta)) +\alpha \eta_t + 2(p-1)B^2 + 8\epsilon^2 TP
\end{align*}
for $\eta_t = [(t-1)P]6 B + (t-1)P  R \sqrt{8 \log(4(tP)^2/\delta)}$. Rearranging and taking the infimum over $\alpha$-covers gives the result as:
\begin{align*}
    \norm{\hat{f}_t - \fstar}_{t,p} \leq \sqrt{\underbrace{8R^2\log(N(\cF, \alpha,\norm{\cdot}_{\infty})/\delta) + 2 \alpha \eta_t + 4(p-1)B^2 + 8\epsilon^2 TP}_{\beta_{t,p}(\cF, \alpha, \delta, \epsilon)}}
\end{align*}

\end{proof}

We also need to control the discretization error arising from the covering.
\begin{lemma}
\label{lem:discretization_error}
    Let $\falpha$ be some element satisfying $\norm{\falpha-f}_{\infty} \leq \alpha$ for some $f \in \cF$. Then under \cref{assump:functions}, with probability at least $1-\delta$,
    \begin{align*}
        & \left| \frac{1}{4} \norm{\falpha-\tf}_{t,1} - \frac{1}{4} \norm{f-\tf}_{t,1} + L_{2,t}(f)-L_{2,t}(\falpha) \right| \leq \\
        & [(t-1)P]6 \alpha B + (t-1)P \alpha R \sqrt{8 \log(4(tP)^2/\delta)}, \forall t \in \mathbb{N}
    \end{align*}
\end{lemma}
\begin{proof}
Note that due to \cref{assump:functions} we only need consider $\alpha \leq B$. Using \cref{assump:functions} for any $\x \in \cX_{t,p}$,
    \begin{align*}
        \left | \falpha(\x)^2-f(\x)^2 \right | \leq \max_{-\alpha \leq y \leq \alpha} \left | (f(\x)+y)^2-f(\x)^2 \right | = 2 f(\x) \alpha + \alpha^2 \leq 2 B \alpha + \alpha^2
    \end{align*}
    which implies,
    \begin{align*}
        \left | (\falpha(\x)-\tf(\x))^2 - (f(\x)-\tf(\x))^2 \right | = \left | \falpha(\x)^2-f(\x)^2 + 2 \tf(\x) (f(\x)-\falpha(\x)) \right | \leq 4 B \alpha + \alpha^2
    \end{align*}
    and 
    \begin{align*}
        \left | (r_{t,p}-f(\x))^2 - (r_{t,p}-\falpha(\x))^2 \right | = \abs{2 r_{t,p}(\falpha(\x)-f(\x)) - f(\x)^2-\falpha(\x)^2} \leq 2\alpha \abs{r_{t,p}} + 2B \alpha+\alpha^2
    \end{align*}
Summing over the appropriate indices,
\begin{align*}
    & \left| \frac{1}{4} \norm{\falpha-\tf}_{t,1} - \frac{1}{4} \norm{f-\tf}_{t,1} + L_{2,t}(f)-L_{2,t}(\falpha) \right| \leq \\
    & \frac{1}{2} \sum_{a=1}^{t-1} \sum_{b=1}^{P} [4B \alpha + \alpha^2] + \sum_{a=1}^{t-1} \sum_{b=1}^{P} [2\alpha \abs{r_{t,p}} + 2B \alpha+\alpha^2] \leq \\
    & [(t-1)P]4B \alpha  + 2 \alpha \sum_{a=1}^{t-1}\sum_{b=1}^{P} \abs{r_{t,p}}
\end{align*}
Using the $R$-sub-gaussianity of $\xi_{t,p}$ we have that $\Pr \left[\abs{\xi_{t,p}} > R \sqrt{2 \log(2/\delta)} \right] \leq \delta$, so by a union bound
\begin{align*}
    \Pr \left[\exists a \in [t-1], \exists p \in [P], \abs{\xi_{t,p}} \geq R \sqrt{2 \log(4(tP)^2/\delta))}\right] \leq \frac{\delta}{2} \sum_{a=1}^{t-1} \sum_{b=1}^P \frac{1}{(tP)^2} \leq  \delta.
\end{align*}
Since $\abs{r_{t,p}} \leq B+\abs{\xi_{t,p}}$ we can conclude that the discretization error is bounded by
\begin{align*}
    [(t-1)P]6 \alpha B + (t-1)P \alpha R \sqrt{8 \log(4(tP)^2/\delta)}. 
\end{align*}
\end{proof}

%% file: app_par_lin_algs.tex
\section{Proofs for Regret Upper Bounds}

In this section we provide the proofs of the regret upper bounds for all of the linear algorithms considered.

\subsection{Proofs for \cref{sec:lin_ucb}}
Here we include the Proof of \cref{thm:regret_par_linucb}.

\begin{proof}[Proof of \cref{thm:regret_par_linucb}]

We first decompose the regret for \cref{algo:par_linucb} by splitting into the linear term and misspecification component.
\begin{align*}
    & \cR(T, P) = \sum_{t=1}^{T} \left(  \sum_{p=1}^P   f(\x^{\star}_{t,p}) - f(\x_{t,p}) \right) = \sum_{t=1}^{T} \left(  \sum_{p=1}^P   \langle \x_{t,p}^\star-\x_{t,p}, \thetastar \rangle \right) + \sum_{t=1}^T \Big( \sum_{p=1}^P f(\x^{\star}_{t,p}) - \langle \x^{\star}_{t,p}, \thetastar \rangle + f(\x_{t,p}) - \langle \x_{t,p}, \thetastar \rangle \Big).
\end{align*}
The second term can be immediately upper bounded as,
\begin{align*}
    \sum_{t=1}^T \Big( \sum_{p=1}^P f(\x^{\star}_{t,p}) - \langle \x^{\star}_{t,p}, \thetastar \rangle + f(\x_{t,p}) - \langle \x_{t,p}, \thetastar \rangle \Big) \leq 2 \epsilon TP.
\end{align*}
using \cref{assump:param}. 
We now approach the linearized reward term. We split this term in each round over the event $\cD_t$,
\begin{align*}
     \sum_{t=1}^{T} \left(  \sum_{p=1}^P   \langle \x_{t,p}^\star-\x_{t,p}, \thetastar \rangle \right) = \sum_{t=1}^{T} \Ind[\cD_t] \left(  \sum_{p=1}^P   \langle \x_{t,p}^\star-\x_{t,p}, \thetastar \rangle \right) + \sum_{t=1}^{T} \Ind[\cD_t^c] \left(  \sum_{p=1}^P   \langle \x_{t,p}^\star-\x_{t,p}, \thetastar \rangle \right)
\end{align*}
The first term here can be bounded using \cref{assump:data,,assump:param} along with the Cauchy-Schwarz inequality which gives $\left(  \sum_{p=1}^P   \langle \x_{t,p}^\star-\x_{t,p}, \thetastar \rangle \right) \leq 2 LS P$ so:
\begin{align*}
    \sum_{t=1}^{T} \Ind[\cD_t] \left(  \sum_{p=1}^P   \langle \x_{t,p}^\star-\x_{t,p}, \thetastar \rangle \right) \leq 2 L S P \sum_{t=1}^T \Ind[\cD_t].
\end{align*}
Note the above bounds hold for any choices of $\x_{t,p} \in \cX_{t,p}$ selected by any doubling-round routine.
We now turn our attention to the second term. For this term we use essentially the same techniques to bound the instantaneous regret by the exact same value for both \cref{algo:par_linucb} and \cref{algo:par_lazy_linucb}, but separate the analysis into two cases for clarity.
\begin{itemize}
\item For \cref{algo:par_linucb} we refer to the optimistic model of processor $p$ at round $t$ as:
\begin{align*}
    \ttheta_{t,p} =\argmax_{\theta \in \cC_{t,1}( \hbtheta_{t}, \V_{t,1}, \beta_t(\delta), \epsilon )}\left( \max_{\x \in \cX_{t,p}} \ \langle \x, \btheta \rangle \right)
\end{align*}
for \cref{algo:par_linucb}. Conditioned on the event in \cref{thm:conf_ellipse_misspec_linucb}--which we denote $\cE_1$--the models $\ttheta_{t,p}$ are optimistic:
\begin{equation*}
   \langle \x_{t,p}, \ttheta_{t, p} \rangle \Ind[\cE_1] \geq \langle \x_{t,p}^\star, \thetastar \rangle \Ind[\cE_1].
\end{equation*}
Hence, 
\begin{align*}
    & \Ind[\cE_1]   \Ind[\cD_t^c]    \langle \x_{t,p}^\star-\x_{t,p}, \thetastar \rangle  \leq \Ind[\cE_1]  \Ind[\cD_t^c]   \langle \x_{t,p}, \ttheta_{t,p}-\thetastar \rangle  \leq \\
    & \Ind[\cE_1]  \Ind[\cD_t^c]  \| \x_{t,p}\|_{\V^{-1}_{t, p}} \| \ttheta_{t,p}-\thetastar \|_{\V_{t,p}}  \leq 2\sqrt{2} \Ind[\cE]  \Ind[\cD_t^c]  \| \x_{t,p}\|_{\V^{-1}_{t, p}} \| \hbtheta_{t}-\thetastar \|_{\V_{t,1}} 
\end{align*}
using optimism in the first inequality, Cauchy-Schwartz in the second, and the fact that on event $\cD_{t}^c$ round $t$ is not a generalized doubling round in the in the final inequality. Finally, recall on the event $\cE_1$, $\Ind[\cE_1] \| \hbtheta_{t} - \thetastar \|_{\V_{t, 1}} \leq (\sqrt{\beta_t(\delta)} + \sqrt{(t-1)P}\epsilon) )\Ind[\cE_1]$ and note $\beta_t(\delta)$ is an increasing function of $t$ so $\beta_t(\delta) \leq \beta_{T}(\delta)$ for all $t \leq T$. Hence it follows,
\begin{align*}\Ind[\cE_1]   \Ind[\cD_t^c]    \langle \x_{t,p}^\star-\x_{t,p}, \thetastar \rangle \leq \Ind[\cE_1] 2\sqrt{2} (\sqrt{\beta_{T}(\delta)} + \sqrt{(t-1)P} \epsilon) \| \x_{t,p}\|_{\V^{-1}_{t, p}} 
\end{align*}
additionally relaxing $\Ind[\cD_t^c] \leq 1$. 
\item For \cref{algo:par_lazy_linucb} we also refer to the optimistic model of processor $p$ at round $t$ as:
\begin{align*}
    \ttheta_{t,p}' =\argmax_{\theta \in \cC_{t,p}( \hbtheta_{t}, \V_{t,p}, 2\beta_t(\delta), 2 \epsilon )}\left( \max_{\x \in \cX_{t, p}} \ \langle \x, \btheta \rangle \right).
\end{align*} 
Conditioned on the event of  \cref{thm:conf_ellipse_misspec_lazy_linucb} restricted to not being a doubling round, the models $\ttheta'_{t,p}$ are optimistic:
\begin{equation*}
   \langle \x_{t,p}, \ttheta'_{t, p} \rangle \Ind[\cE_2] \Ind[\cD_{t}^c] \geq \langle \x_{t,p}^\star, \thetastar \rangle \Ind[\cE_2] \Ind[\cD_{t}^c].
\end{equation*}
Hence, 
\begin{align*}
    & \Ind[\cE_2]   \Ind[\cD_t^c]    \langle \x_{t,p}^\star-\x_{t,p}, \thetastar \rangle  \leq \Ind[\cE_2]  \Ind[\cD_t^c]   \langle \x_{t,p}, \ttheta'_{t,p}-\thetastar \rangle  \leq \\
    & \Ind[\cE_2]  \Ind[\cD_t^c]  \| \x_{t,p}\|_{\V^{-1}_{t, p}} \| \ttheta'_{t,p}-\thetastar \|_{\V_{t,p}} 
\end{align*}
using optimism in the first inequality, Cauchy-Schwartz in the second. Finally, on the event $\cE_2 \cap \cD_t^c$, $\Ind[\cE_2] \Ind[\cD_t^c] \| \ttheta'_{t,p} - \thetastar \|_{\V_{t, p}} \leq 2 \sqrt{2} \sqrt{\beta_t(\delta)} \Ind[\cE_2] \Ind[\cD_t^c] \leq 2 \sqrt{2} \sqrt{\beta_T(\delta)} \Ind[\cE_2] \Ind[\cD_t^c]$. Hence it follows,
\begin{align*}\Ind[\cE_2]   \Ind[\cD_t^c]    \langle \x_{t,p}^\star-\x_{t,p}, \thetastar \rangle \leq 2 \sqrt{2} \Ind[\cE_2] (\sqrt{\beta_{T}(\delta)}+\sqrt{(t-1)P} \epsilon) \| \x_{t,p}\|_{\V^{-1}_{t, p}} 
\end{align*}
additionally relaxing $\Ind[\cD_t^c] \leq 1$. 
\end{itemize}
The remainder of the proof follows identically for both \cref{algo:par_linucb,,algo:par_lazy_linucb}. Without loss of generality we use $\cE$ to refer to either event $\cE_1$ or $\cE_2$ in the following (note both hold with probability at least $1-\delta$ by \cref{thm:conf_ellipse_linucb,,thm:conf_ellipse_lazy_linucb}). Recalling that the instantaneous regret is $\leq 2 LS$ we can combine this bound with the aforementioned bounds to conclude that,
\begin{align*}
    & \Ind[\cE]  \cdot \left( \sum_{t=1}^{T} \Ind[\cD_t^c] \left(  \sum_{p=1}^P   \langle \x_{t,p}^\star-\x_{t,p}, \thetastar \rangle \right) \right) \\ 
    & \stackrel{(i)}{\leq} \Ind[\cE] \cdot \sqrt{ TP \left( \sum_{t=1}^{T} \Ind[\cD_t^c] \left(  \sum_{p=1}^P   \langle \x_{t,p}^\star-\x_{t,p}, \thetastar \rangle^2 \right) \right)} \\
    & \stackrel{(ii)}{\leq}
    4\sqrt{2}  \Ind[\cE] \sqrt{ TP \sum_{t=1}^{T} \sum_{p=1}^P \min((LS)^2,   2(\beta_{T}(\delta) + (t-1)P \epsilon^2) \| \x_{t, p}\|^2_{\V_{t, p}^{-1}})} \\
    & \stackrel{(iii)}{\leq} 4\sqrt{2}  \Ind[\cE] \sqrt{ TP \sum_{t=1}^{T} \sum_{p=1}^P \min((LS)^2,   2 (\beta_{T}(\delta) + TP \epsilon^2) \| \x_{t, p}\|^2_{\V_{t, p}^{-1}})} \\
    & \stackrel{(iv)}{\leq} 4\sqrt{2} \Ind[\cE] \sqrt{ TP 2 (\beta_{T}(\delta) + TP \epsilon^2) \max \left(2, \frac{(LS)^2}{2 (\beta_{T}(\delta) + TP \epsilon^2)} \right) \cdot \sum_{t=1}^T \sum_{p=1}^P \log \left(1+ \sum_{t=1}^{T} \sum_{p=1}^P \| \x_{t, p}\|^2_{\V_{t, p}^{-1}} \right)} \\
    & \stackrel{(v)}{\leq} 8 \Ind[\cE] \sqrt{TP} \max(\sqrt{2} \sqrt{\beta_{T}(\delta)} + \sqrt{TP} \epsilon, LS) \cdot \sqrt{d \log\left( 1+\frac{TP L^2}{\lambda} \right)}.
\end{align*}
Inequality $(i)$ follow by Cauchy-Schwarz. Inequality $(ii)$ employs both our bounds on the instantaneous regret. Inequality $(iii)$ follows by upper bounding $t-1 \leq T$ for the misspecification term.
Inequality $(iv)$ follows because for all $a, x > 0$, we have $\min(a,x ) \leq \max(2, a) \log(1+x)$. Inequality $(v)$ follows because $\sum_{t=1}^T \sum_{p=1}^P \log(1 + \| \x_{t,p} \|^2_{\V^{-1}_{t,p}} ) \leq d \log\left( 1+\frac{PT L^2}{\lambda}) \right)$ by instantiating \cref{lem::ellip_potential}. Assembling the bounds in the original regret splitting over doubling rounds and accounting for the original misspecification term shows that, 
\begin{align*}
    \Ind[\cE] \cR(T, P) \leq \Ind[\cE] 8 \cdot \left (LS P\sum_{t=1}^T \Ind[\cD_{t}] + \sqrt{TP} \max(\sqrt{2}(\sqrt{\beta_{T}(\delta)} + \sqrt{TP} \epsilon), LS) \cdot \sqrt{d \log\left( 1+\frac{TP L^2}{\lambda} \right)} + \epsilon TP \right)
\end{align*}
where $\sqrt{\beta_{T}(\delta)} \leq R \sqrt{ d \log \left(\frac{1+ TPL^2/\lambda}{\delta} \right) } + \sqrt{\lambda} S$. This inequality holds on the event $\cE$ which occurs with probability at least $1-\delta$ for both \cref{algo:par_linucb} and \cref{algo:par_lazy_linucb}. Inserting this value for $\beta_{T}(\delta)$ and hiding logarithmic factors shows on the event $\cE$,
\begin{align*}
    \cR(T, P) \leq \tlO \left( LS P\cdot \sum_{t=1}^T \Ind[\cD_t] + \sqrt{dTP} \max(\sqrt{2}(R \sqrt{d}+\sqrt{\lambda} S + \sqrt{TP} \epsilon), LS) \right)
\end{align*}
\end{proof}

\subsection{Proofs of Section~\ref{section:lints}}

We start by stating a folklore lemma regarding the anti-concentration properties of a Gaussian distribution. 

\subsubsection{Concentration and Anti-Concentration properties of the Gaussian distribution}

\begin{lemma}\label{lemma::anticoncentration_gaussian}
Let $X$ be a random variable distributed according to $\mathcal{N}(\mu, \sigma^2)$, a one dimensional Gaussian distribution with mean $\mu$ and variance $\sigma^2$. The following holds:
\begin{equation*}
    \mathbb{P}\left( X-\mu \geq \tau \right) \geq \frac{1}{\sqrt{2\pi}} \frac{\sigma \tau  }{\tau^2 + \sigma^2 }\exp\left( -\frac{\tau^2}{2\sigma^2}\right) 
\end{equation*}
\end{lemma}

We will also make use of the following concentration inequality for Lipschitz functions of Gaussian vectors:

\begin{theorem}[Theorem 2.4 in \cite{wainwright2019high}]\label{theorem::gaussian_Vectors_lipschitz_concentration} 
Let $\boldsymbol{\eta} \sim \mathcal{N}( \mathbf{0}, \mathbb{I}_d)$ be a standard Gaussian vector and let $f: \mathbb{R}^d \rightarrow \mathbb{R}$ be $L-$Lipschitz with respect to the Euclidean norm. Then the variable $f(\boldsymbol{\eta} ) - \mathbb{E}\left[ f( \boldsymbol{\eta} )\right]$ is subgaussian with parameter at most $L$ and hence:
\begin{equation*}
    \mathbb{P}\left( f(X) \geq \mathbb{E}\left[ f(X) \right] + t \right) \leq \exp\left( -\frac{t^2}{2L^2}\right)
\end{equation*}
\end{theorem}

We'll make use of these two results to prove Lemma~\ref{lemma::anticoncentration_concentration_gaussians} which we restate for the reader's convencience:

\begin{lemma}
\label{lemma::anticoncentration_concentration_gaussians}
The Gaussian distribution satisfies (anticoncentration) for every $\bv \in \mathbb{R}^d$ with $\| \bv \|=1$:
\begin{equation}\label{equation::anti_concentration_gaussian_appendix}
    \mathbb{P}_{\boldsymbol{\eta} \sim \mathcal{N}( \mathbf{0}, \mathbb{I}_d)} \left(\bv^\top \boldsymbol{\eta} \geq 1\right) \geq  \frac{1}{4}.  
\end{equation}
And (concentration), $\forall \delta \in (0,1)$:
\begin{equation}\label{equation::concentration_gaussian_appendix}
    \mathbb{P}_{\boldsymbol{\eta} \sim \mathcal{N}( \mathbf{0}, \mathbb{I}_d)} \left(\| \boldsymbol{\eta}\| \leq \sqrt{d} + \sqrt{2\log\left(\frac{1}{\delta}\right) }\right) \geq 1-\delta. 
\end{equation}
\end{lemma}

\begin{proof}
Equation~\ref{equation::anti_concentration_gaussian_appendix} is a simple consequence of the following two observations:
\begin{enumerate}
    \item For any unit norm vector $\bv \in \mathbb{R}^d$ the random variable $X = \bv^\top \eta $ is distributed as a one dimensional Gaussian with unit variance $\mathcal{N}( 0, 1)$.
    \item Setting parameters $\mu = 0,\sigma =1$, and $\tau = 1$ Lemma~\ref{lemma::anticoncentration_gaussian} implies that $\mathbb{P}\left( X \geq 1 \right) \geq \frac{1}{\sqrt{2\pi}}\cdot \frac{1}{2} \exp( - \frac{1}{2}) \geq $.
\end{enumerate}

Equation~\ref{equation::concentration_gaussian_appendix} instead follows from Theorem~\ref{theorem::gaussian_Vectors_lipschitz_concentration}. Since the function $f(\cdot) = \| \cdot \|$ is $1-$Lipschitz and $\mathbb{E}_{\boldsymbol{\eta} \sim \mathcal{N}(\mathbf{0}, \mathbb{I}_d) }\left[ \| \boldsymbol{\eta} \|  \right] \leq (\mathbb{E}_{\boldsymbol{\eta} \sim \mathcal{N}(\mathbf{0}, \mathbb{I}_d) }\left[ \| \boldsymbol{\eta} \|^2  \right])^{1/2} = \sqrt{d}  $ :

\begin{equation*}
    \mathbb{P}_{\boldsymbol{\eta} \sim \mathcal{N}( \mathbf{0}, \mathbb{I}_d )}\left( \| \boldsymbol{\eta} \| \geq \sqrt{d}  + \sqrt{2\log\left(\frac{1}{\delta}\right) }    \right)  \leq \delta.
\end{equation*}
The result follows.
\end{proof}

Recall that:

\begin{equation*}
      \sqrt{\beta_{t}(\delta)} =  R\sqrt{  \log\left( \frac{\det(\V_{t-1,0})}{\lambda^d \delta^2}\right) } + \sqrt{\lambda} S  \leq  R \sqrt{ d \log \left(\frac{1+ tPL^2/\lambda}{\delta} \right) } + \sqrt{\lambda} S
\end{equation*}

\subsubsection{ Concentration of $\ttheta_{t,p} $ }

The main objective of this section is to show that with high probability the sampled parameter $\ttheta_{t,p}$ is not too far from the true parameter $\btheta_\star$ for all times $t$ and processors $p$. The result is encapsuled by Lemma~\ref{lemma::conditional_closeness_theta_tilde}. 

\begin{lemma}\label{lemma::conditional_closeness_theta_tilde}
Let:
\begin{equation*}
\gamma_{t, p}(\delta) :=     \sqrt{2}(\sqrt{\beta_t(\delta)} + \sqrt{(t-1)P} \epsilon) \left( \sqrt{d} + 2\sqrt{ \log\left( \frac{t(P-1) + p}{\delta}\right)}  + 1\right).
\end{equation*}
The following conditional probability bound holds:
\begin{equation}\label{equation::conditional_closeness_theta_tilde_1}
    \mathbb{P}\left( \Ind[ \cD^c_t \cap \mathcal{E} ] \| \ttheta_{t,p} - \btheta_\star \|_{\V_{t,p}} \geq \gamma_{t, p}(\delta)  \Big| \mathcal{F}_{t, p-1} \right) \leq \frac{\delta}{2(t(P-1) + p)^2}
\end{equation}
And therefore with probability at least $1-2\delta$ and for all $t \in N$ simultaneously:
\begin{equation}\label{equation::conditional_closeness_theta_tilde}
    \| \ttheta_{t,p} - \btheta_\star \|_{\V_{t,p}} \leq \gamma_{t, p}(\delta) 
\end{equation}
We refer to this event as $\mathcal{E}'$.
\end{lemma}

In order to prove~\cref{lemma::conditional_closeness_theta_tilde} let's start by showing that for any time-step $t$ and processor $t$ the sample $\ttheta_{t-1, p}$ is close to $\btheta_\star$ with high probability:
\begin{lemma}\label{lemma::concentration_theta_tilde}
The following conditional probability bound holds:
\begin{equation*}
    \mathbb{P}\left( \| \ttheta_{t, p} - \hbtheta_{t}\|_{\V_{t, p}}  \geq \sqrt{2\beta_t(\delta) } \left( \sqrt{d} + 2\sqrt{ \log\left( \frac{t(P-1) + p}{\delta}\right)} \right)  \Big | \mathcal{F}_{t, p-1} \right) \leq \frac{\delta}{2(t(P-1) + p)^2}
\end{equation*}
Where $\mathcal{F}_{t, p-1}$ corresponds to the sigma algebra generated by all the events up to and including the reveal of contexts $\cX_{t, p}$. And therefore with probability at least $1-\delta$ simultaneously and unconditionally for all $t \in \mathbb{N}$:
\begin{equation*}
    \|  \ttheta_{t,p}  - \hbtheta_{t}\|_{\V_{t,p}} \leq \sqrt{2}(\sqrt{\beta_t(\delta)} + \sqrt{(t-1)P} \epsilon) \left( \sqrt{d} + 2\sqrt{ \log\left( \frac{t(P-1) + p}{\delta}\right)} \right)
\end{equation*}
\end{lemma}

\begin{proof}

In order to bound $\| \ttheta_{t,p}- \hbtheta_{t} \|_{\V_{t,p}}$ we make use of Lemma~\ref{lemma::anticoncentration_concentration_gaussians}. Observe that by definition:

\begin{equation}\label{equation::definition_eta}
    \| \ttheta_{t,p} - \hbtheta_{t}\|_{\V_{t,p}} = \sqrt{2}(\sqrt{\beta_t(\delta)} + \sqrt{(t-1)P} \epsilon) \| \boldsymbol{\eta}_{t,p}\|_2.
\end{equation}

Therefore a simple use of Lemma~\ref{lemma::anticoncentration_concentration_gaussians} implies that (concentration):
\begin{equation*}
    \mathbb{P}_{\boldsymbol{\eta}_{t,p} \sim \mathcal{N}(\mathbf{0}, \mathbb{I}_d ) } \left( \| \boldsymbol{\eta}_{t,p} \| \leq \sqrt{d} + 2\sqrt{ \log\left( \frac{t(P-1) + p}{\delta}\right)}\Big | \mathcal{F}_{t, p-1} \right) \leq \frac{\delta}{2(t(P-1) + p)^2}.   
\end{equation*}
And therefore as a consequence of Equation~\ref{equation::definition_eta}:
\begin{equation*}
    \mathbb{P}\left( \| \ttheta_{t, p} - \hbtheta_{t}\|_{\V_{t, p}}  \geq \sqrt{2}(\sqrt{\beta_t(\delta)} + \sqrt{(t-1)P} \epsilon) \left( \sqrt{d} + 2\sqrt{ \log\left( \frac{t(P-1) + p}{\delta}\right)} \right)  \Big | \mathcal{F}_{t, p-1} \right) \leq \frac{\delta}{2(t(P-1) + p)^2}
\end{equation*}
Furthermore, a simple union bound implies that for all $t \in \mathbb{N}$:
\begin{equation}\label{equation::anytime_concentration_eta}
    \mathbb{P}\left( \exists t \text{ s.t. } \| \boldsymbol{\eta}_{t,p} \|_2 \geq \sqrt{d} + 2\sqrt{\log\left(\frac{t(P-1) + p}{\delta}\right) }      \right) \leq  \frac{\delta}{2}\sum_{t=1}^\infty \sum_{p=1}^P \frac{1}{(t(P-1) + p)^2} \leq \delta
\end{equation}
Combining equations~\ref{equation::definition_eta} and~\ref{equation::anytime_concentration_eta} yields:
\begin{equation*}
    \mathbb{P}\left(  \| \ttheta_{t,p} - \hbtheta_{t}\|_{\V_{t,p}} \leq \sqrt{2}(\sqrt{\beta_t(\delta)} + \sqrt{(t-1)P} \epsilon) \left( \sqrt{d} + 2\sqrt{ \log\left( \frac{t(P-1) + p}{\delta}\right)} \right)   \right) \leq \delta
\end{equation*}
\end{proof}

Lemma~\ref{lemma::concentration_theta_tilde}, conditioning on the $1-\delta$ probability event $\mathcal{E}$, a simple use of the triangle inequality along with the identity $\sum_{p=1}^P \sum_{t=1}^\infty \frac{1}{(t(P-1)+p)^2} = \frac{\pi^2}{6} < 2$ finalizes the proof of Lemma~\ref{lemma::conditional_closeness_theta_tilde}. From now on we will denote as $\mathcal{E}'$ to the $1-2\delta$ probability event defined by Lemma~\ref{lemma::conditional_closeness_theta_tilde}.

\subsubsection{Anti-concentration of $ \ttheta_{t, p}$ }

The main objective of this section will be to prove the following upper bound for the instantaneous regret $\btheta_\star^\top \x_{t, p}^\star - \btheta_\star^\top \x_{t, p}$ in the event the round is not a doubling round and $\mathcal{E}'$ holds. This is one of the main components of the proof of~\cref{thm:linTS_parallel}.

\begin{lemma}\label{lemma::using_optimism}
Let $\ttheta_{t, p}'$ be a copy of $\ttheta_{t, p}$, equally distributed to $\ttheta_{t, p}$ and independent of it conditionally on $\mathcal{F}_{t, p-1}$. We call $\x_{t, p}'$ to the resulting argmax action $\argmax_{\x \in \cX_{t, p}} \langle \ttheta_{t, p}', \x \rangle$. The following inequality holds:

\begin{equation*}
    \Ind[ \cD^c_t \cap \mathcal{E}' ] \left( \btheta_\star^\top \x_{t, p}^\star - \btheta_\star^\top \x_{t, p} \right)\leq  \frac{2\gamma_{t, p}(\delta)}{\frac{1}{4}-\frac{\delta}{2(t(P-1) + p)^2}} \Ind[ \cD^c_t \cap \mathcal{E}' ]\mathbb{E}\left[  \|\x_{t, p}'\|_{\V_{t, p}^{-1}}  \Big| \mathcal{F}_{t,p-1} \right] +\gamma_{t, p}(\delta) \Ind[ \cD^c_t \cap \mathcal{E}' ]\| \x_{t,p}\|_{\V_{t,p}^{-1}}.
\end{equation*}
\end{lemma}

Before proving \cref{lemma::using_optimism} we show that with constant probability, the estimated value of the action taken at time and processor tuple $(t, p)$ is optimistic with constant probability:

\begin{lemma}\label{lemma::lower_bound_prob_optimism}
For all $t \in N$:
\begin{equation*}
    \mathbb{P}\left(  \ttheta_{t, p}^\top \x_{t,p} \geq \btheta_\star^\top \x_{t,p}^\star \Big| \mathcal{F}_{t, p-1}, \mathcal{E} \right) \geq \frac{1}{4}
\end{equation*}
\end{lemma}

\begin{proof}
Recall that whenever $\mathcal{E}$ holds:
\begin{equation*}
    \| \hbtheta_t - \btheta_\star \|_{\V_{t,p}} \leq \sqrt{2}(\sqrt{\beta_t(\delta)} + \sqrt{(t-1)P} \epsilon).
\end{equation*}

Notice that by definition $\x_{t,p}$ satisfies:
\begin{equation*}
    \ttheta_{t, p}^\top \x_{t,p} \geq \ttheta_{t, p}^\top \x_{t,p}^\star.
\end{equation*}
Therefore:
\begin{align*}
    \ttheta_{t, p}^\top \x_{t, p}^\star &= \left(\ttheta_{t, p} - \hbtheta_t \right)^\top \x_{t, p}^\star + \left( \hbtheta_t - \btheta_\star \right)^\top \x_{t, p}^\star + \btheta_\star^\top \x_{t, p}^\star\\
    &\stackrel{(i)}{\geq}   \sqrt{2}(\sqrt{\beta_t(\delta)} + \sqrt{(t-1)P} \epsilon) \boldsymbol{\eta}_{t, p}^\top \V_{t,p}^{-1/2} \x_{t, p}^\star - \| \hbtheta_t - \btheta_\star \|_{\V_{t,p}} \| \x_{t,p}^\star \|_{\V_{t,p}^{-1}} + \btheta_\star^\top \x_{t, p}^\star \\
    &\stackrel{(ii)}{\geq} \sqrt{2}(\sqrt{\beta_t(\delta)} + \sqrt{(t-1)P} \epsilon) \boldsymbol{\eta}_{t, p}^\top \V_{t,p}^{-1/2} \x_{t, p}^\star - \sqrt{2}(\sqrt{\beta_t(\delta)} + \sqrt{(t-1)P} \epsilon) \| \x_{t,p}^\star \|_{\V_{t,p}^{-1}} + \btheta_\star^\top \x_{t,p}^\star  \\
    &= \sqrt{2}(\sqrt{\beta_t(\delta)} + \sqrt{(t-1)P} \epsilon)\left( \boldsymbol{\eta}_{t, p}^\top \V_{t,p}^{-1/2} \x_{t,p}^\star - \| \V_{t,p}^{-1/2} \x_{t,p}^\star\|_2 \right)  .  
\end{align*}
Inequality $i$ holds as a consequence of Cauchy Schwartz inequality. Inequality $(ii)$ holds by conditioning on $\mathcal{E}$ and because $t \in N$. 

By Equation~\ref{equation::anti_concentration_gaussian_appendix} in Lemma~\ref{lemma::anticoncentration_concentration_gaussians}, and by noting that $\x^\star_{t,p}$ is conditionally independent of $\boldsymbol{\eta}_{t, p}$, we can infer that $\boldsymbol{\eta}_{t, p}^\top \V_{t,p}^{-1/2} \x_{t,p}^\star  \geq \| \V_{t,p}^{-1/2} \x_{t,p}^\star\|_2$ with probability at least $1/4$. The result follows.

\end{proof}

Let's define the set of optimistic model parameters:
\begin{equation*}
    \boldsymbol{\Theta}_{t, p}= \{ \btheta \in \mathbb{R}^d \text{ s.t. }  \max_{\x \in \cX_{t, p}}  \x^\top \btheta \geq (\x_{t,p}^\star)^\top \btheta_\star  \}.  
\end{equation*}

Where $\mathcal{D}_t$ denotes the event that round $t$ is a doubling round. We now show how to bound the instantaneous regret $r_{t,p}$ during all rounds by using these results:

\begin{proof}[Proof of~\cref{lemma::using_optimism}]
Recall that $\mathcal{E}'$ is the event defined by Equation~\ref{equation::conditional_closeness_theta_tilde} in Lemma~\ref{lemma::conditional_closeness_theta_tilde} applied to the $\{\ttheta_{t,p}\}_{t, p}$ sequence. Define $\{\mathcal{E}''_{t,p}\}_{t,p}$ be the corresponding event family defined by Equation~\ref{equation::conditional_closeness_theta_tilde_1} in Lemma~\ref{lemma::conditional_closeness_theta_tilde} applied to the $\{\ttheta_{t,p}'\}_{t, p}$ sequence. It follows that $\mathbb{P}(\mathcal{E}_{t,p}''|\mathcal{F}_{t, p-1}) \geq 1-\frac{\delta}{2(t(P-1)+p)^2}$.
Notice that if $\cD^c_t \cap \mathcal{E}'$ holds (meaning $\| \ttheta_{t,p} - \btheta_\star \|_{\V_{t,p}} \leq \delta_{t, p}$ and because $\x_{t,p} = \argmax_{\x \in \cX_{t,p}} \ttheta_{t,p}^\top \x_{t,p}$ then:
\begin{equation*}
    \ttheta_{t, p}^\top \x_{t,p} \geq \inf_{\btheta \in \cC(\btheta_{\star}, \V_{t, p}, \gamma_{t, p}(\delta))} \max_{\x \in \cX_{t,p} } \btheta_{t,p}^\top \x := \bar{\btheta}_{t, p}^\top \bar{\x}_{t, p} .
\end{equation*}

When $\ttheta_{t, p}'$ is optimistic:
\begin{equation}\label{equation::optimism_conditioning}
    \btheta_\star^\top \x_{t, p}^\star - \ttheta_{t,p}^\top \x_{t, p} \leq \langle  \ttheta_{t, p}', \x_{t, p}' \rangle - \bar{\btheta}_{t, p}^\top \bar{\x}_{t,p} \quad \Big|   \ttheta_{t, p}' \in \boldsymbol{\Theta}_{t, p}.
\end{equation}
 Therefore:
\begin{align*}
\Ind[ \cD^c_t \cap \mathcal{E}' ]\left(\btheta_\star^\top \x_{t,p}^\star - \btheta_\star^\top \x_{t,p} \right)&= \Ind[ \cD^c_t \cap \mathcal{E}' ]\left(\btheta_\star^\top \x_{t,p}^\star - \ttheta_{t,p}^\top \x_{t,p} \right) + \Ind[ \cD^c_t \cap \mathcal{E}' ]\left( \ttheta_{t,p}^\top \x_{t,p} - \btheta_\star^\top \x_{t,p}\right)\\
&\leq \Ind[ \cD^c_t \cap \mathcal{E}' ]\left(\btheta_\star^\top \x_{t,p}^\star - \ttheta_{t,p}^\top \x_{t,p} \right) + \Ind[ \cD^c_t \cap \mathcal{E}' ]\| \ttheta_{t,p} - \btheta_\star\|_{\V_{t,p}} \| \x_{t,p}\|_{\V_{t,p}^{-1}} \\
    &\stackrel{(i)}{\leq} \Ind[ \cD^c_t ]\mathbb{E}\left[   \Ind[  \mathcal{E}' ]\left(\langle \ttheta_{t, p}', \x_{t, p}' \rangle - \bar{\btheta}_{t, p}^\top \bar{\x}_{t, p}   \right)  \Big| \mathcal{F}_{t,p-1}, \ttheta_{t,p}' \in \boldsymbol{\Theta}_{t, p}, \mathcal{E}''_{t,p} \right] + \\
    &\text{  }~\quad \Ind[ \cD^c_t \cap \mathcal{E}' ]\| \ttheta_{t,p} - \btheta_\star\|_{\V_{t,p}} \| \x_{t,p}\|_{\V_{t,p}^{-1}} \\
    &\stackrel{(ii)}{\leq}  \Ind[ \cD^c_t  ]\mathbb{E}\left[  \Ind[ \mathcal{E}' ] \langle \ttheta_{t, p}'- \bar{\btheta}_{t, p}, \x_{t, p}' \rangle  \Big| \mathcal{F}_{t,p-1}, \ttheta_{t, p}' \in \boldsymbol{\Theta}_{t, p}, \mathcal{E}''_{t,p} \right] + \\
    &~\text{   }~\quad \Ind[ \cD^c_t \cap \mathcal{E}' ]\| \ttheta_{t,p} - \btheta_\star\|_{\V_{t,p}} \| \x_{t,p}\|_{\V_{t,p}^{-1}}\\
    &\stackrel{(iii)}{\leq}\Ind[ \cD^c_t ] \mathbb{E}\left[ \Ind[  \mathcal{E}' ] \|   \ttheta_{t, p}'- \bar{\btheta}_{t, p}\|_{\V_{t, p}} \|\x_{t, p}'\|_{\V_{t, p}^{-1}}  \Big| \mathcal{F}_{t,p-1}, \ttheta_{t, p}' \in \boldsymbol{\Theta}_{t, p},\mathcal{E}''_{t,p}\right] + \\
    &~\text{    }~\quad \Ind[ \cD^c_t \cap \mathcal{E}' ]\| \ttheta_{t,p} - \btheta_\star\|_{\V_{t,p}} \| \x_{t,p}\|_{\V_{t,p}^{-1}} \\
    &\stackrel{(iv)}{\leq} 2 \gamma_t(\delta) \Ind[ \cD^c_t \cap \mathcal{E}' ]\mathbb{E}\left[  \|\x_{t, p}'\|_{\V_{t, p}^{-1}}  \Big| \mathcal{F}_{t, p-1}, \ttheta_{t, p}' \in \boldsymbol{\Theta}_{t, p}, \mathcal{E}''_{t,p} \right] + \\
    &~\text{   }~\quad \Ind[ \cD^c_t \cap \mathcal{E}' ]\| \ttheta_{t,p} - \btheta_\star\|_{\V_{t,p}} \| \x_{t,p}\|_{\V_{t,p}^{-1}}\\
    &\stackrel{(v)}{\leq} \frac{2\gamma_{t, p}(\delta)}{\frac{1}{4}-\frac{\delta}{2(t(P-1) + p)^2}} \Ind[ \cD^c_t \cap \mathcal{E}' ]\mathbb{E}\left[  \|\x_{t, p}'\|_{\V_{t, p}^{-1}}  \Big| \mathcal{F}_{t,p-1} \right] + \gamma_{t, p}(\delta) \Ind[ \cD^c_t \cap \mathcal{E}' ]\| \x_{t,p}\|_{\V_{t,p}^{-1}}
\end{align*}

Inequality $(i)$ follows by Equation~\ref{equation::optimism_conditioning}, $(ii)$ by the definition of $\bar{x}_t$, $(iii)$ is a consequence of Cauchy Schwartz, $(iv)$ follows by the definition of $\mathcal{E}'$ and $\mathcal{E}''_{t,p}$, and $(v)$ follows because $\| \x_{t,p}' \|_{\V_{t, p}^{-1}} $ is nonnegative and because by Lemma~\ref{lemma::conditional_closeness_theta_tilde}, it follows that $\mathbb{P}\left( \ttheta_{t, p}' \in \boldsymbol{\Theta}_t, \mathcal{E}''_{t,p} | \mathcal{F}_{t,p-1}\right) \geq  \frac{1}{4} - \frac{\delta}{2(t(P-1)+p)^2} $. The result follows.

\end{proof}

\subsubsection{Ancillary Lemmas}

In the proof of~\cref{thm:linTS_parallel} we will also make use of the following supporting result:

\begin{lemma}\label{lemma::upper_bounding_norm_inverse_exp_to_true}
Similar to Lemma~\ref{lemma::using_optimism},  let $\ttheta_{t, p}'$ copy of $\ttheta_{t, p}$, equally distributed to $\ttheta_{t, p}$ and independent conditionally on $\mathcal{F}_{t, p-1}$. We call $\x_{t, p}'$ to the resulting argmax action $\argmax_{\x \in \cX_{t, p}} \langle \x, \ttheta_{t, p}'\rangle$.  With probability at least $1-\delta$:
\begin{equation*}
    \sum_{t=1}^T \sum_{p=1}^P \mathbb{E}[  \| \x'_{t,p}   \|_{\V^{-1}_{t, p}}  | \mathcal{F}_{t, p-1}] \leq \sum_{t=1}^T \sum_{p=1}^P  \| \x_{t,p}\|_{\V^{-1}_{t, p}}  + \frac{2L}{\sqrt{\lambda}}\sqrt{TP \log( 1/\delta)}
\end{equation*}

\end{lemma}

\begin{proof}

Define the martingale difference sequence $Z_{t, p} =  \| \x_{t,p}   \|_{\V^{-1}_{t, p}} - \mathbb{E}[  \| \x'_{t,p}   \|_{\V^{-1}_{t, p}}  | \mathcal{F}_{t, p-1}] $ (the indexing is lexicographic over pairs $(t, p)$). It is easy to see that $| Z_{t,p} | \leq 2\frac{L}{\sqrt{\lambda}}$, since for all valid $\x \in \cX_{t, p}$, $\| \x \|_{\V_{t, p}^{-1}} \leq \frac{L}{\sqrt{\lambda}}$ for all $t,p$ pairs. Consequently a simple use of Hoeffding bound yields the result.

\end{proof}

\subsection{Proof of Theorem~\ref{thm:linTS_parallel}} 

We proceed to prove the general version of Theorem~\ref{thm:linTS_parallel}:

\begin{theorem}
    Let \cref{assump:noise,,assump:data,,assump:param} hold and $\cD_{t}$ denote the event that round $t$ is a doubling round (see \cref{cond:critical_inequality}). Then the regret of both \cref{algo:par_lints,,algo:par_lazy_lints} satisfy ,
    \begin{equation}
    \cR(T, P) \leq \tlO \left( LS P \cdot \sum_{t=1}^T \Ind[\cD_t] + 4\gamma_T(\delta) \left( \sqrt{dTP \left( 1 + \frac{L^2}{\lambda}\right)\ln\left( \frac{d\lambda + TPL}{d\lambda}\right) }+  \frac{2L}{\sqrt{\lambda}}\sqrt{TP \log( 1/\delta)} \right) + \epsilon TP \right)
    \end{equation}
    with probability at least $1-3\delta$, whenever $\delta \leq \frac{1}{6}$. Here     $$\gamma_{T}(\delta) = \sqrt{2}(\sqrt{\beta_T(\delta)} + \sqrt{(T-1)P} \epsilon) \left( \sqrt{d} + 2\sqrt{ \log\left( \frac{T(P-1) + P}{\delta}\right)}  + 1\right).$$ 
\end{theorem}

\begin{proof}[Proof of \cref{thm:linTS_parallel}]

The regret in terms of $f(\cdot)$ can be linearized at the cost of an additive $2\epsilon \sqrt{d}TP$ as in the Proof of \cref{thm:regret_par_linucb}. After this we can  decompose the regret for \cref{algo:par_lints}, splitting on the event each round over the event $\cD_t$,
\begin{align*}
    & \cR(T, P) = \sum_{t=1}^{T} \left(  \sum_{p=1}^P   \underbrace{\langle \x_{t,p}^\star-\x_{t,p}, \thetastar \rangle}_{r_{t,p}} \right) = \sum_{t=1}^{T} \Ind[\cD_t] \left(  \sum_{p=1}^P   \langle \x_{t,p}^\star-\x_{t,p}, \thetastar \rangle \right) + \sum_{t=1}^{T} \Ind[\cD_t^c] \left(  \sum_{p=1}^P   \langle \x_{t,p}^\star-\x_{t,p}, \thetastar \rangle \right)
\end{align*}
The first term can be bounded using \cref{assump:data,,assump:param} along with the Cauchy-Schwarz inequality which gives $\left(  \sum_{p=1}^P   \langle \x_{t,p}^\star-\x_{t,p}, \thetastar \rangle \right) \leq 2 LS P$ so:
\begin{align*}
    \sum_{t=1}^{T} \Ind[\cD_t] \left(  \sum_{p=1}^P   \langle \x_{t,p}^\star-\x_{t,p}, \thetastar \rangle \right) \leq 2 L S P \sum_{t=1}^T \Ind[\cD_t].
\end{align*}
This holds for any sequence of $\x_{t,p}$ chosen by the doubling round routine $\textsc{DR}$. We turn our attention to the second term. 

\begin{equation}
     \sum_{t=1}^{T} \Ind[\cD_t^c] \left(  \sum_{p=1}^P   \langle \x_{t,p}^\star-\x_{t,p}, \thetastar \rangle \right) =  \sum_{t=1}^{T} \left(  \sum_{p=1}^P   \Ind[\cD_t^c \cap \mathcal{E}']  \langle \x_{t,p}^\star-\x_{t,p}, \thetastar \rangle \right) +  \sum_{t=1}^{T} \left(  \sum_{p=1}^P  \Ind[\cD_t^c \cap \left(\mathcal{E}'\right)^c]  \langle \x_{t,p}^\star-\x_{t,p}, \thetastar \rangle \right)\label{equation::terms_to_bound}
\end{equation}
Notice that:
\begin{align*}
    \sum_{t=1}^{T}  \sum_{p=1}^P  \Ind[\cD_t^c \cap \left(\mathcal{E}'\right)^c]   \langle \x_{t,p}^\star-\x_{t,p}, \thetastar \rangle  &\leq \sum_{t=1}^{T}  \sum_{p=1}^P  2\Ind[\cD_t^c \cap \left(\mathcal{E}'\right)^c]  LS\leq 2LS TP \Ind[  \left(\mathcal{E}'\right)^c ] 
\end{align*}

And therefore we can forget this term when we condition on $\mathcal{E}'$, an event that occurs with probability at least $1-2\delta$ (recall $\mathcal{E'}$ as the event from \cref{lemma::conditional_closeness_theta_tilde}).

It remains to bound the first term of Equation~\ref{equation::terms_to_bound}. 
\begin{align*}
    \sum_{t=1}^{T} \sum_{p=1}^P   \Ind[\cD_t^c \cap \mathcal{E}']   \langle \x_{t,p}^\star-\x_{t,p}, \thetastar \rangle  &\stackrel{(i)}{\leq} \sum_{t=1}^T   \sum_{p=1}^P  \frac{8}{1-3\delta}\gamma_{t, p}(\delta) \Ind[ \cD^c_t \cap \mathcal{E}' ] \mathbb{E}\left[ \| \x'_{t,p} \|_{\V_{t, p}^{-1}}   |\mathcal{F}_{t, p-1}  \right]  + \\
    &\text{  }~\quad \gamma_{t, p}(\delta) \Ind[ \cD^c_t \cap \mathcal{E}' ]\| \x_{t,p}\|_{\V_{t, p}^{-1}}   \\
    &\stackrel{(ii)}{\leq} \sum_{t=1}^T   \sum_{p=1}^P \frac{8}{1-3\delta} \gamma_{t, p}(\delta) \mathbb{E}\left[ \| \x'_{t,p} \|_{\V_{t-1, p-1}^{-1}}   |\mathcal{F}_{t-1, p-1}  \right]    + \gamma_{t, p}(\delta) \| \x_{t,p}\|_{\V_{t, p}^{-1}} \\
    &\stackrel{(iii)}{\leq}\left(\frac{8}{1-3\delta}+1\right)\gamma_{T, P}(\delta) \left( \sum_{t=1}^T \sum_{p=1}^P  \| \x_{t,p} \|_{\V_{t-1, p-1}^{-1}}+  \frac{2L}{\sqrt{\lambda}}\sqrt{TP \log( 1/\delta)} \right) \\
    &\leq 16\gamma_T(\delta) \left( \sqrt{TP}\sum_{t=1}^T \sum_{p=1}^P  \| \x_{t,p} \|^2_{\V_{t-1, p-1}^{-1}}+  \frac{2L}{\sqrt{\lambda}}\sqrt{TP \log( 1/\delta)} \right) \\
    &\leq 16\gamma_T(\delta) \left( \sqrt{dTP \left( 1 + \frac{L^2}{\lambda}\right)\ln\left( \frac{d\lambda + TPL}{d\lambda}\right) }+  \frac{2L}{\sqrt{\lambda}}\sqrt{TP \log( 1/\delta)} \right)
\end{align*}
Inequality $(i)$ holds by Lemma~\ref{lemma::using_optimism} and the assumption that $\delta \leq \frac{1}{6}$. Inequality $(ii)$ holds because all terms $\|\x'_{t,p} \|_{\V_{t, p}^{-1}}$ and $\| \x_{t,p} \|_{\V_{t,p}^{-1}}$ are nonnegative. Inequality $(iii)$ holds with probability at least $1-\delta$ and is a consequence of lexicographic monotonicity (in $t,p$) of $\gamma_{t, p}(\delta)$ and Lemma~\ref{lemma::upper_bounding_norm_inverse_exp_to_true}. The last two inequalities are a simple consequence of the determinant lemma (\cref{lem::ellip_potential}). 
\end{proof}

\subsection{Auxiliary Results}
Here we summarize the self-normalized vector martingale inequality used to establish the confidence ball for the least-squares estimator in a well-specified linear model,
\begin{align}
    r_{t, p} = \x_{t, p}^\top \thetastar + \xi_{t, p}. \label{eq:lin_bandit}
\end{align}
Here $\xi_{t,p}$ is an i.i.d. noise process.

\begin{theorem}\label{thm:conf_ellipse_linucb}[Theorem 1 in \cite{abbasi2011improved}] 
For all $t \in \mathbb{N}$:
\begin{equation*}
    \norm{\hbtheta_{t}-\thetastar}_{\V_{t,1}} \leq \sqrt{\beta_t(\delta)}
\end{equation*}
with probability at least $1-\delta$. Moreover, on this event by definition,
\begin{equation*}
    \thetastar \in \cC(\hbtheta_{t}, \V_{t,1}, \beta_t(\delta)).
\end{equation*}
\end{theorem}

We can now prove a generalization of this result which applies to the analysis of the lazy LinUCB algorithm in a well-specified model.

\begin{theorem}\label{thm:conf_ellipse_lazy_linucb}

Let $N \subseteq [T]$ be the set of rounds which are not doubling rounds (see \cref{cond:critical_inequality}). Then, 
\begin{equation*}
    \forall t \in N, \quad \norm{\hbtheta_{t}-\thetastar}_{\V_{t,p}} \leq \sqrt{2} \sqrt{\beta_t(\delta)}
\end{equation*}
with probability at least $1-\delta$. Moreover, on this event by definition,
\begin{align*}
    \thetastar \in \cC(\hbtheta_{t}, \V_{t, p}, 2 \beta_t(\delta)).
\end{align*}
\end{theorem}
\begin{proof}
Let $N \subseteq \mathbb{N}$ be the set of rounds which are not doubling rounds. Then if $t \in N$,
\begin{align*}
    \norm{\hbtheta_{t}-\thetastar}_{\V_{t, p}} \leq \sqrt{2} \norm{\hbtheta_{t}-\thetastar}_{\V_{t,1}}
\end{align*}
from the definition in \cref{cond:critical_inequality}. Hence,
\begin{align*}
    & \Pr[\norm{\hbtheta_{t}-\thetastar}_{\V_{t,p}} \geq \sqrt{2} \sqrt{\beta_t(\delta)} , t \in N] \leq \Pr[\norm{\hbtheta_{t}-\thetastar}_{\V_{t,1}} \geq \sqrt{\beta_t(\delta)}, t \in N] \leq \\
    & \Pr[\norm{\hbtheta_{t}-\thetastar}_{\V_{t,1}} \geq \sqrt{\beta_t(\delta)}, \forall t \in \mathbb{N}]
     \leq \delta.
\end{align*}
where the final inequality follows by \cref{thm:conf_ellipse_linucb}.

\end{proof}

Define $\mathcal{E}$ to the $1-\delta$ probability event defined in Theorem~\ref{thm:conf_ellipse_lazy_linucb}:
\begin{equation*}
    \mathcal{E} := \{     \forall t \in N, \quad \norm{\hbtheta_{t}-\thetastar}_{\V_{t,p}} \leq \sqrt{2} \sqrt{\beta_t(\delta)}  \} 
\end{equation*}

We are now in a position to prove generalizations of these results which provide valid confidence sets for the linear regression estimator in \textit{misspecified} models. In summary, the confidence sets are modified with a growing, additive correction to accommodate the bias arising from the misspecification.

\begin{theorem}\label{thm:conf_ellipse_misspec_linucb}
If the rewards are generated from a model satisfying \cref{assump:param}, then for all $t \in \mathbb{N}$:
\begin{equation*}
    \norm{\hbtheta_{t}-\thetastar}_{\V_{t,1}} \leq \sqrt{\beta_t(\delta)} + \sqrt{(t-1)P} \epsilon
\end{equation*}
with probability at least $1-\delta$. Moreover, on this event by definition,
\begin{equation*}
    \thetastar \in \cC(\hbtheta_{t}, \V_{t,1}, \beta_t(\delta), \epsilon).
\end{equation*}
\end{theorem}
\begin{proof}
    The argument uses a bias-variance decomposition. First, define the linearized reward $\tlr_{a,b} = \x_{a,b}^\top \thetastar + \xi_{a,b}$ and linear estimator using these rewards as $\ttheta = \V_{t, 1}^{-1} (\sum_{a=1}^{t-1} \sum_{b=1}^{P}  \x_{a,b} \tlr_{a,b})$. By definition, $r_{a,b}-\tlr_{a,b} = \epsilon_{a,b}$ (which all satisfy $\abs{\epsilon_{a,b}} \leq \epsilon$ uniformly for all $\x_{a,b}$ by assumption). Then, 
\begin{align*}
    & \hbtheta_{t} - \thetastar = \ttheta - \thetastar = \ttheta - \thetastar + \V_{t, 0}^{-1} (\sum_{a=1}^{t-1} \sum_{b=1}^{P}  \x_{a,b} \epsilon_{a, b}) \implies \\
    & \norm{\hbtheta_{t}-\thetastar}_{\V_{t,1}} \leq \norm{\ttheta_{t}-\thetastar}_{\V_{t,1}} + \norm{\sum_{a=1}^{t-1} \sum_{p=1}^{P} \x_{a,b} \epsilon_{a,b}}_{\V^{-1}_{t,1}}
\end{align*}
The first term can be bounded by $\sqrt{\beta_t(\delta)}$ with probability $1-\delta$ exactly by using \cref{thm:conf_ellipse_linucb}. Using the projection bound in \citet[Lemma 8]{zanette2020learning} it follows that, 
\begin{align}
    \norm{\sum_{a=1}^{t-1} \sum_{p=1}^{P} \x_{a,b} \epsilon_{a,b}}_{\V^{-1}_{t,1}} \leq \sqrt{(t-1)P} \epsilon
\end{align}
since $\abs{\epsilon_{a,b}} \leq \epsilon$ uniformly for all $a,b$.

\end{proof}
The analogue for the lazy confidence set follows similarly,
\begin{theorem}\label{thm:conf_ellipse_misspec_lazy_linucb}
Let $N \subseteq [T]$ be the set of rounds which are not  doubling rounds (see \cref{cond:critical_inequality}). Then, 
If the rewards are generated from a model satisfying \cref{assump:param}, for all $t \in N$:
\begin{equation*}
    \norm{\hbtheta_{t}-\thetastar}_{\V_{t,p}} \leq \sqrt{2}(\sqrt{\beta_t(\delta)} + \sqrt{(t-1)P} \epsilon)
\end{equation*}
with probability at least $1-\delta$. Moreover, on this event by definition,
\begin{equation*}
    \thetastar \in \cC(\hbtheta_{t}, \V_{t,1}, 2 \beta_t(\delta), 2 \epsilon).
\end{equation*}
\end{theorem}
\begin{proof}
First, if $t \in N$,
\begin{align*}
    \norm{\hbtheta_{t}-\thetastar}_{\V_{t, p}} \leq \sqrt{2} \norm{\hbtheta_{t}-\thetastar}_{\V_{t, 1}}
\end{align*}
from the definition in \cref{cond:critical_inequality}. The remainder of the argument follows as in the previous result,
\begin{align*}
    & \Pr[\norm{\hbtheta_{t}-\thetastar}_{\V_{t, p}} \geq \sqrt{2}( \sqrt{\beta_t(\delta)} + \sqrt{(t-1)P} \epsilon) , t \in N] \leq \Pr[\norm{\hbtheta_{t}-\thetastar}_{\V_{t,1}} \geq \sqrt{\beta_t(\delta)} + \sqrt{(t-1)P} \epsilon, t \in N] \leq \\
    & \Pr[\norm{\ttheta_{t}-\thetastar}_{\V_{t,1}} \geq \sqrt{\beta_t(\delta)}, \forall t \in \mathbb{N}]
     \leq \delta.
\end{align*}
where the final inequality follows by \cref{thm:conf_ellipse_linucb}, since $\ttheta = \V_{t, 1}^{-1} (\sum_{a=1}^{t-1} \sum_{b=1}^{P}  \x_{a,b} \tlr_{a,b})$ is the estimator utilizing the linearized rewards.
\end{proof}

Next we recall the elliptical potential lemma which control the volumetric growth of the space spanned by sequence of covariance matrices.
\begin{lemma}\label{lem::ellip_potential}[Lemma 19.4 in \citet{lattimore_szepesvari_2020}] 
Let $\V_0 \in \mR^{d \times d}$ be a positive-definite matrix, $\x_1, \hdots, \x_n \in \mR^d$ be a sequence of vectors with $\norm{\x_i} \leq L$ for all $i \in [n]$, and $\V_n = \V_0 + \sum_{s \leq n} \x_s \x_s^\top$. Then,
\begin{equation*}
    \sum_{s=1}^n \log(1+\| \x_s \|^2_{\V^{-1}_{s-1}}) = \log \left (\frac{\det \V_n}{\det \V_0} \right) 
\end{equation*}
\begin{equation*}
    \log \left( \frac{\det \V_n}{\det \V_0} \right) \leq d \log \left(\frac{\tr(\V_0) + n L^2}{d} \right) - \log( \det \V_0 )
\end{equation*}
\end{lemma}

Finally, we state a prove a simple fact from linear algebra.

\begin{lemma}
If $\A \succeq \B \succ 0$,
\begin{align*}
    \forall \x \neq 0 \in \mR^d, \quad \frac{\x^\top \A \x}{\x^\top \B \x} \leq \frac{\det(\A)}{\det(\B)}.
\end{align*}
\label{lem:psd_det}
\end{lemma}
\begin{proof}[Proof of \cref{lem:psd_det}]
To this end, let $\x = \B^{-1/2} \y$ for $\y \in \mR^d$. Then note that,
\begin{align*}
    \sup_{\x \neq 0} \frac{\x^\top \A \x}{\x^\top \B \x} = \norm{\B^{-1/2} \A \B^{-1/2}}_2
\end{align*}
by the definition of the operator norm. Similarly, we can rewrite $\frac{\det(\A)}{\det(\B)} = \det(\B^{-1/2} \A \B^{-1/2})$.
The claim then follows because all the eigenvalues of $\B^{-1/2} \A \B^{-1/2}$ are $\geq 1$ (note $\inf_{\x : \norm{\x}_2=1} \x^\top \B^{-1/2} \A \B^{-1/2} \x = 1 + \x^\top \B^{-1/2} (\A-\B) \B^{-1/2} \x \geq 1$ since $\A \succeq \B$).
\end{proof}

\section{Proofs in \cref{sec:rich_context}}

We now present the proof of  \cref{cor:rich_explore_par_linucb_regret}.

\begin{proof}[Proof of \cref{cor:rich_explore_par_linucb_regret}]
As a consequence of \cref{thm:regret_par_linucb} we have for any choice of doubling round routine that the regret of \cref{algo:par_linucb} and \cref{algo:par_lazy_linucb} obeys,
\begin{align*}
    \Ind[\cE] \cR(T, P) \leq \Ind[\cE] 8 \cdot \left (\underbrace{LS P\sum_{t=1}^T \Ind[\cD_{t}]}_{A} + \underbrace{\sqrt{TP} \max(\sqrt{\beta_{T}(\delta)} + \sqrt{TP} \epsilon, LS) \cdot \sqrt{d \log\left( 1+\frac{TP L^2}{\lambda} \right)} + \epsilon TP}_{B} \right)
\end{align*}
where $\sqrt{\beta_{T}(\delta)} \leq R \sqrt{ d \log \left(\frac{1+ TPL^2/\lambda}{\delta} \right) } + \sqrt{\lambda} S$. This inequality holds on the event $\cE$ which occurs with probability at least $1-\delta$ for both \cref{algo:par_linucb} and \cref{algo:par_lazy_linucb}. Now we introduce an additional event $\cG = \{ \sum_{t=1}^{T} \Ind[\cD_t] \geq \lceil \thresh \rceil \}$. 
Then we can consider two cases,
\begin{itemize}
    \item First, 
    \begin{align*}
        \Ind[\cG^c] (A + B) \leq \Ind[\cG^c]  \left(LS \left\lceil \thresh \right\rceil + B \right) 
    \end{align*}
    simply by definition of the event.
    \item Second, by definition of the $\V_{t,p}$ and $\cG$ all rounds which are not doubling rounds (denoted by the set $N$) lead to randomly generated covariates being used to estimate the covariance. Let $r_1 \in N$ be the first doubling round to exceed the threshold $\lceil \thresh \rceil$. We claim (with high probability) there can be no more doubling rounds after round $r_1$. This follows since first by
    \cref{lem:mat_conc} on the event $\cG$,
    \begin{align*}
        \Vert \sum_{a \in N} \sum_{p=1}^{P} \x_{a, p} \x_{a,p}^\top- (\mSigma_{\pi_{a,p}} + \mu_{a,p} \mu_{a,p}^\top)) \Vert_2 \leq \begin{cases}
        & \frac{1}{10} \pi_{\min}^2 \abs{N} P \quad \text{ when } \abs{N} > 1\\
        & \frac{1}{10} \pi_{\min}^2 P \quad \text{ when } \abs{N} = 1
        \end{cases}
    \end{align*}
    with probability at least $1-\delta$ (denote this further event $\cC$). This previous inequality follows by considering the two separate cases where $\abs{N} > 1$ and $\abs{N}=1$ respectively.
    Thus in any round $t>r_1$ we must have with probability $1-\delta$,
    \begin{align*}
        & \V_{t,1} \succeq \sum_{a \in N} \sum_{p=1}^{P} (\mSigma_{\pi_{a,p}} +\mu_{i,p} \mu_{i,p}^\top) - \frac{\pi_{\min}^2 \abs{N} P}{10} \I \succeq  \frac{9}{10} \abs{N}P \pi_{\min}^2 \I \quad \text{ when } \abs{N} > 1 \\
        & \V_{t,1} \succeq P (\mSigma_{\pi_{1,p}} +\mu_{1,p} \mu_{1,p}^\top)  - \frac{\ell^2  P}{10} \I \succeq  \frac{9}{10} P \ell^2 \I \quad \text{ when } \abs{N}=1
    \end{align*}
    using \cref{def:rich_exp_dist}. In summary for all $\abs{N} \geq 1$ we then have that 
    \begin{align*}
        \V_{t,1} \succeq  \frac{9}{10} \abs{N}P \pi_{\min}^2 \I.
    \end{align*}
    If, additionally it is the case that $\abs{N} \geq \lceil \frac{10}{9} \frac{\chi^2}{\pi_{\min}^2} \rceil$, let $r_2$ be the first round after which this occurs. Then it  follows that for all $t>\max(r_1, r_2)$, 
    \begin{align}
        & \Ind[\cC] (\sum_{p=1}^{P} \x_{t,p} \x_{t,p}^\top) \preceq \Ind[\cC] P\chi^2 \I \preceq \frac{9}{10} \abs{N} P \pi_{\min}^2 \I \preceq \V_{t,1},
    \end{align}
    so no $t > \max(r_1, r_2)$ can be a doubling round on this event. Concluding we have that,
    \begin{align*}
        \Ind[\cG] \Ind[\cC] A \leq LSP \cdot  \max(\lceil \thresh \rceil, \lceil \frac{10}{9} \frac{\chi^2}{\pi_{\min}^2} \rceil).
    \end{align*}
    Together, we obtain,
    \begin{align*}
        \Ind[\cG] \Ind[\cC] (A + B) \leq \Ind[\cG] \Ind[\cC] \left( LSP \cdot  \max(\lceil \thresh \rceil, \lceil \frac{10}{9} \frac{L^2}{\ell^2} \rceil) + B \right)
    \end{align*}
\end{itemize}
Assembling and summing these two cases, then shows that,
\begin{align*}
    \Ind[\cC] (A+B) \leq (  LSP \cdot  \max(\lceil \thresh \rceil, \lceil \frac{10}{9} \frac{\chi^2}{\ell^2} \rceil) + B).
\end{align*}
So it follows that, 
\begin{align*}
    & \Ind[\cE]\Ind[\cC] \cR(T, P) \leq \Ind[\cE]\Ind[\cC] \cdot 8 \cdot ( LSP \cdot  \max(\lceil \thresh \rceil, \lceil \frac{10}{9} \frac{\chi^2}{\ell^2} \rceil)  + \\
    & \sqrt{TP} \max(\sqrt{\beta_{T}(\delta)} + \sqrt{TP} \epsilon, LS) \cdot \sqrt{d \log\left( 1+\frac{TP \chi^2}{\lambda} \right)} + \epsilon TP )
\end{align*}
where both $\cE$ and $\cC$ hold with probability $1-\delta$. We can simplify the first term to,
\begin{align*}
    LSP \cdot  \max(\lceil \thresh \rceil, \lceil \frac{10}{9} \frac{\chi^2}{\ell^2} \rceil ) \leq \tlO(LSP \cdot(\thresh + 1 + \frac{\chi^2}{\ell^2} ) ) \leq \tlO( LS (\frac{L^4}{\ell^4}+P \frac{\chi^2}{\ell^2}))
\end{align*}

Hiding logarithmic factors, this implies that,
\begin{align*}
    \cR(T, P) \leq \tlO \left( R \left( (\frac{L^2 \norm{\Sigpi}}{\ell^4} + P \frac{\chi^2}{\ell^2}) \sqrt{\snr} + d \sqrt{TP} + \frac{\epsilon}{R} TP \right)\right) 
\end{align*}
with probability at least $1-2 \delta$.

An identical argument establishes the result for the Thompson sampling algorithms save with Thompson sampling regret $\cR(T, P)$ used in place of the LinUCB regrets in the previous argument.
\end{proof}

We now present the matrix concentration result we use,
\begin{lemma}
    Let \cref{def:rich_exp_dist} and \cref{assump:data} hold and consider $N$ i.i.d. copies of sets (with $P$ elements) sampled from \cref{algo:random_explore}, labeled as $\{ \x_{i, p} \}_{i=1, p=1}^{N, P}$. Then,
    \begin{align*}
        \Vert \frac{1}{NP} (\sum_{i=1}^{N} \sum_{p=1}^{P} \x_{i,p} \x_{i,p}^\top - (\mSigma_{\pi_{i,p}} + \mu_{i,p} \mu_{i,p}^\top)) \Vert_{2} \leq 12 \left( L \sqrt{\pi_{\max}^2}  \sqrt{\frac{\log(4d/\delta)}{NP}} + \frac{L^2 \log(4d/\delta)}{NP}  \right) 
    \end{align*}
    with probability at least $1-\delta$.
    \label{lem:mat_conc}
\end{lemma}
\begin{proof}
We first center the expression around its mean $\mE_{\pi_{i,p}}[\x_{i,p}]=\mu_{i,p}$. That is,
\begin{align*}
        & \Vert \frac{1}{NP} \sum_{i=1}^{N} \sum_{p=1}^{P} \x_{i,p} \x_{i,p}^\top - (\mSigma_{\pi_{i,p}} + \mu_{i,p} \mu_{i,p}^\top) \Vert_{2} \leq \\
        & \Vert \frac{1}{NP} \sum_{i=1}^{N} \sum_{p=1}^{P} (\x_{i,p}-\mu_{i,p})(\x_{i,p}-\mu_{i,p})^\top - \mSigma_{\pi_{i,p}} \Vert_{2} + 2 \Vert \frac{1}{NP} \sum_{i=1}^{N} \sum_{p=1}^{P} \mu_{i,p} (\x_{i,p}-\mu_{i,p})^\top -\mu_{i,p}\mu_{i,p}^\top \Vert_2
\end{align*}

We now apply the matrix Bernstein inequality to control this first term \citep[Theorem 1.6.2]{tropp2012user}. Note that for all $i$, $\norm{(\x_{i,p}-\mu_{i,p})(\x_{i,p}-\mu_{i,p})^\top} \leq 2L^2$ and the matrix variance is bounded by $\norm{\mE[(\x_{i,p}-\mu_{i,p})(\x_{i,p}-\mu_{i,p})^\top - \mSigma_{\pi_{i,p}})^2]}_2 \leq  \cdot \norm{\mE[\norm{\x_{i,p}-\mu}_2^2 (\x_{i,p}-\mu)(\x_{i,p}-\mu)^\top)]}_2 \leq 2L^2 \pi_{\max}^2$. Thus we obtain,
\begin{align}
     \Vert \frac{1}{NP} \sum_{i=1}^{N} \sum_{p=1}^{P} (\x_{i,p}-\mu_{i,p})(\x_{i,p}-\mu_{i,p})^\top - \mSigma_{\pi_{i,p}} \Vert_{2} \leq 4 \left( L \sqrt{\pi_{\max}^2}  \sqrt{\frac{\log(4d/\delta)}{NP}} + \frac{L^2 \log(4d/\delta)}{NP} \right)
\end{align}
with probability at least $1-\delta/2$.
For the second term note by Jensen's inequality that $\norm{\mu_{i,p}}_2 = \norm{\mE[\x_{i,p}]}_2 \leq \mE[\norm{\x_{i,p}}_2] \leq L$ since $\norm{\x_{i,p}}_2 \leq L$. An identical calculation before to bound almost surely bound this term and its matrix variance we have that,
\begin{align*}
    \Vert \frac{1}{NP} \sum_{i=1}^{N} \sum_{p=1}^{P} \mu_{i,p} (\x_{i,p}-\mu_{i,p})^\top -\mu_{i,p}\mu_{i,p}^\top \Vert_2 \leq 4 \left( L \sqrt{\pi_{\max}^2}  \sqrt{\frac{\log(4d/\delta)}{NP}} + \frac{L^2 \log(4d/\delta)}{NP} \right)
\end{align*}
by the matrix Bernstein inequality with probability at least $1-\delta/2$. Summing the terms and applying a union bound over the events on which they hold gives the result.
\end{proof}

%% file: app_lb.tex
\section{Proofs in \cref{sec:lb}}
Here we include the proof of the main lower bound.
\label{sec:app_lb}

\begin{proof}[Proof of \cref{thm:lb}]

The proof follows by first noting that in each of the instances claimed a single, fixed global context vector is used for all time and processors. Hence the parallel to sequential regret reduction established in \cref{prop:regret_reduction} is applicable. Thus it suffices to establish the lower bounds for $R(TP, 1)$ in lieu of $R(T, P)$ for the instances claimed. The first term/result is an immediate consequence of \cref{lem::lower_bound_sphere}. The second term/result we can obtain from \cref{lem:misspec_lower_bound}. For the validity of the results, we inherit the constraints $d \geq  \lceil 8 \log(m) L^2/\epsilon^2 \rceil$ and $S \geq \norm{\thetastar}_2 = \frac{\epsilon}{L} \sqrt{\frac{d-1}{8 \log(m)}} \implies \epsilon^2 \leq \frac{8 (LS)^2 \log(m)}{d-1}$. If we take $d=\lceil 8 \log(m)/\epsilon^2 \rceil$ then the second constraint reduces too $\epsilon^2 \leq \frac{8 (LS)^2 \log(m)}{8 \log(m)/\epsilon^2} \implies LS \geq 1$. The final component follows directly from \cref{lem:adv_context_lower_bound}.

\end{proof}

\subsection{Parallel to Sequential Reduction}
Here we establish the reduction from parallel regret to sequential regret when considering parallel linear bandits  where there is a single fixed action set/context set across all processors at a given time. So $\cX_{t,p} = \cX_t$ across all processors $p \in [P]$. Additionally, we assume as in the preamble that the reward of an action $\x$ is determined by $r = f(\x)+\epsilon$--so the law of the rewards is completely specified by a mean reward function $f$ and mean-zero noise distribution $\epsilon$.

We formalize the reduction from parallel to sequential bandits by first defining the canonical bandit environment. We consider the bandit instance to be indexed by the law of the rewards and the sequence of context sets.

\paragraph{Sequential}

In a model of purely sequential interaction we consider instances defined by two ingredients:
\begin{itemize}
    \item the conditional distribution of the policy $\pi_t(\cdot | \x_{i <t}, \cX_{i \leq t}, r_{i < t})$.
    \item the sequence of reward distributions $\nu_s \equiv \Pr_t(\cdot |  x_{i <t}, \cX_{i \leq t}, r_{i < t}, f) \equiv \Pr_{\x_t}(\cdot | f)$ for selected actions and the sequence of presented contexts $\cX_t$.
\end{itemize}

\paragraph{Parallel}
In a model of parallel interaction,
there are two ingredients:
\begin{itemize}
    \item the conditional distribution of policy $\psi_{t, p}(\cdot | x_{i,j <t, p}, \cX_{i \leq t}, r_{i,j < t,p})$. 
    \item the sequence of reward distributions $\nu_p \equiv \Pr_{t, p}(\cdot |  \x_{i,j < t, p}, \cX_{i \leq t}, r_{i, j < t, 1}) \equiv \Pr_{\x_{t,p}}(\cdot | f)$ for selected actions and sequence of selected contexts $\cX_t$ (which for fixed $t$ are equal across all $p \in [P]$).
\end{itemize}

To formalize the reduction we make the following claim:
\begin{proposition}
If we consider the lexicographic ordering for $t, p \in [T, P]$, then for any sequence of parallel policy-environment interactions with law $( \psi_{t,p}(\cdot), \Pr_{t,p}(\cdot | f) )$ and presented context sets $\cX_t$ (identical across $p \in [P]$), there exists a corresponding coupling to a purely sequential bandit environment $(\pi_{m}(\cdot), \Pr_m(\cdot | f))$ for $m \in [TP]$ and sequence of context sets $\cX_m = \cX_t$ for $m \in [tP, (t+1)P]$  with an identical distribution.
\label{prop:par_to_seq_reduction}
\end{proposition}
\begin{proof}
 We can construct the sequential environment inductively from the sequence of parallel interaction by coupling. To see this consider the first round of parallel interactions which are described by the measure,
    \begin{align*}
        \Pi_{p=1}^P \psi_{1, p}( \cdot |\x_{i, j < 1, p}, \cX_1, r_{i, j < 1, 1}) \Pr_{\x_{1,p}}(\cdot | f).
    \end{align*}
By defining the measure over sequential policy-environment interactions,
    \begin{align*}
        \Pi_{m=1}^P \pi_{m}( \cdot |\x_{a < m} \cX_m, r_{a < m}) \Pr_{\x_{m}}(\cdot | f)
    \end{align*}
we can set, $\pi_{m}( \cdot |\x_{a < m} \cX_{a<m}, r_{a < m}) = \psi_{1, p}( \cdot |\x_{i, j < 1, p} \cX_1, r_{i, j < 1, 1})$ by ignoring the conditioning on the further contexts and reward information in the sequential interaction to enforce equality of the policies. Since the policies are tamen to be identical, by coupling the randomness between the sequential and parallel policies/environments, the sequence of selected actions will be identical (this uses the fact the context sets in the parallel blocm and sequential interaction can be tamen equal). Inductively proceeding with the construction over the blocms of parallel interaction completes the argument.
\end{proof}

With this claim in hand the reduction follows since the expected regret of an algorithm over an environment (in expectation) is determined by the law of the policy-environment interactions.
Before beginning we first introduce the following notation for the expected parallel regret (further indexed by policy and environment,
\begin{equation}
    R_{\psi, \nu_p}(T, P) = \sum_{t=1}^T \sum_{p=1}^P \max_{\x \in \cX_t} \mu(\x) - \mE_{\psi, \nu_p}[\mu(\x)]
\end{equation}
where $\psi$ denotes the parallel policy and $\nu_p$ the parallel environment indexed by the mean reward and sequence of context sets. Similarly, the sequential regret can be defined analogously as,
\begin{equation}
    R_{\pi, \nu_s}(TP, 1) = \sum_{a=1}^{T P} \max_{\x \in \cX_a} \mu(\x) - \mE_{\pi, \nu_s}[\mu(\x)]
\end{equation}
where $\pi$ denotes the sequential policy and $\nu_s$ the sequential environment indexed by the mean reward and sequence of context sets which \textit{are fixed to be the same over consecutive blocms of length $P$}.

\begin{proposition}
\label{prop:regret_reduction}
Consider both a parallel bandit policy/environment class and sequential bandit policy/environment class as defined in \cref{prop:par_to_seq_reduction}. Then,
    \begin{align}
    \inf_{\psi} \sup_{\nu_p} R_{\psi, \nu_p}(T, P) \geq \inf_\pi \sup_{\nu_s} R_{\pi, \nu_s} (TP, 1)
    \end{align}
where the infima over $\psi$ is tamen place over the class of all parallel policies and the infima on the right hand side is tamen over the class of all sequential policies.
\end{proposition}

\begin{proof}
By the the preceding claim in \cref{prop:par_to_seq_reduction}
every pair of parallel  $(\psi, \nu_p)$ measures can be reproduced by a corresponding sequential measure $(\pi, \nu_s)$. Let the set of such induced sequential policies be $\cP$. So the pointwise inequality,
\begin{align}
        R_{\psi, \nu_p} (T, P) =  R_{\pi, \nu_s} (TP, 1)
\end{align}
holds by simply re-indexing the summation in lexigraphic order since the expectations are identical. Hence, for a fixed $\psi$ (and induced $\pi$) the equality also holds after taming a supremum over the same indexing set on both sides (parameterized by $f$ and $\cX_t$),
\begin{align}
        \sup_{\nu_p} R_{\psi, \nu_p} (T, P) =  \sup_{\nu_s} R_{\pi, \nu_s} (TP, 1)
\end{align}
Now taming an infima over the policy class $\psi$ and equivalent induced sequential policy class shows,
\begin{align}
\inf_{\psi} \sup_{\nu_p} R_{\psi, \nu_p} (T, P) = \inf_{\pi \in \cP} \sup_{\nu_s} R_{\pi, \nu_s} (TP, 1) \geq \inf_{\pi} \sup_{\nu_1} R_{\pi, \nu_s} (TP, 1)
\end{align}
by relaxing the final infima to tame place over the class of all sequential policies.
\end{proof}

\subsection{Unit Ball Lower Bound}

Here we record the lower bound over a fixed action set for the unit ball for a sequential bandit instance. The proof is an immediate generalization of~\citep[Theorem 24.2]{lattimore_szepesvari_2020} with the scales restored. Throughout this section we assume that the rewards are generated as,
\begin{align}
    r_{t,p} = \x_{t,p}^\top \thetastar + \xi_{t, p}
\end{align}
where $\xi_{t,p} \sim \cN(0, R^2)$.

\begin{lemma}\label{lem::lower_bound_sphere}
Let the fixed action set $\cX  = \{ \x \in \mathbb{R}^d : \| \x \|_2 \leq L\} $ and parameter set $\Theta = \{ \pm \Delta \}^d$ for $\Delta = \frac{R\sqrt{d}}{3\sqrt{2T}L}$. If $T \geq d \max(1, \frac{1}{3 \sqrt{2} \sqrt{\snr}})$, then for any policy $\pi$, there is a vector $\thetastar \in \Theta$ such that: 
\begin{equation*}
     R_{\pi, (\cX, \btheta)}(T, 1) \geq  \Omega\left(Rd\sqrt{T}\right).
\end{equation*}
\end{lemma}
\begin{proof}[Proof of \cref{lem::lower_bound_sphere}]
Let $\mathcal{A} = \{ \x \in \mathbb{R}^d : \| \x \|_2 \leq L\}$ and $\btheta \in \mathbb{R}^d$ such that $\| \btheta\|_2^2 = S^2$. Let $\Delta  =  \frac{R\sqrt{d}}{3\sqrt{2T}L}$ and $\btheta \in \{ \pm \Delta \}^d$ and for all $i \in [d]$ define $\tau_i = \min( T, \min(t: \sum_{s=1}^t \x_{s, i}^2 \geq \frac{\alpha n}{d} \} $. We will set $\alpha = L^2$ at the end of the proof but in order to mame the derivations clearer and easier to read we will meep this $\alpha$ explicit. Then for any policy $\pi$:

\begin{align*}
      R_{\pi, (\cX, \btheta)}(T, 1) &= \Delta \mathbb{E}_{\theta}\left[ \sum_{t=1}^T \sum_{i=1}^d \left( \frac{L}{\sqrt{d}} -\x_{t,i} \mathrm{sign}(\theta_i)\right)   \right] \\
      &\geq \frac{\Delta \sqrt{d}}{2L} \mathbb{E}_{\theta}\left[   \sum_{t=1}^T \sum_{i=1}^d \left( \frac{L}{\sqrt{d}} - \x_{t, i} \mathrm{sign}(\theta_i)\right)^2    \right] \\
      &\geq \frac{\Delta \sqrt{d}}{2L} \sum_{i=1}^d \mathbb{E}_{\theta}\left[   \sum_{t=1}^{\tau_i}  \left( \frac{L}{\sqrt{d}} - \x_{t, i} \mathrm{sign}(\theta_i)\right)^2    \right]
\end{align*}

Where the first inequality uses that $\| \x_t \|^2_2 \leq L^2$. Fix $i \in [d]$. For $x\in \{ -1, 1\}$, define $U_i(x) = \sum_{t=1}^{\tau_i} \left(  \frac{L}{\sqrt{d}} - \x_{t,i}x\right)^2$. And let $\theta' \in \{\pm \Delta\}^d$ be another parameter vector with $\theta_j = \theta_j'$ for $j \neq i$ and $\theta_i' = -\theta_i$. Assume without loss of generality that $\theta_i >0$. Let $\Pr_{\btheta}$ and $\Pr_{\btheta'}$ be the laws of $U_i(1)$ w.r.t. the bandit learner interaction measure induced by $\theta$ and $\theta'$ respectively. By a simple calculation we conclude that:

\begin{align}
    \KL(\Pr_{\btheta}, \Pr_{\btheta'}) \leq  2 \frac{\Delta^2}{4R^2} \mathbb{E}_{\theta}\left[ \sum_{t=1}^{\tau_i} \x_{t,i}^2 \right] 
\end{align}
Also, observe that:

\begin{equation*}
    U_i(1) = \sum_{t=1}^{\tau_i} \left(L/\sqrt{d} - \x_{t, i} \right)^2 \leq 2L^2 \sum_{t=1}^{\tau_i} \frac{1}{d} + 2\sum_{t=1}^{\tau_i} \x_{t,i}^2 \leq \left( \frac{2L^2T + 2\alpha T}{d} + 2L^2 \right) 
\end{equation*}

Then:

\begin{align*}
    \mathbb{E}_\theta[U_i(1)] &\geq \mathbb{E}_{\theta'}[U_i(1) ] - \left( \frac{2L^2T + 2\alpha T}{d} + 2L^2 \right)  \sqrt{\KL(\Pr_{\btheta}, \Pr_{\btheta'})} \\
    &\geq \mathbb{E}_{\theta'}[U_i(1) ] - \left( \frac{2L^2T + 2\alpha T}{d} + 2L^2 \right) \frac{\Delta}{2R} \sqrt{ \mathbb{E}_{\theta}\left[ \sum_{t=1}^{\tau_i} \x_{t,i}^2 \right]  }\\
    &\geq \mathbb{E}_{\theta'}[U_i(1) ] - \left( \frac{2L^2T+ 2\alpha T}{d} + 2L^2 \right) \frac{\Delta}{2R} \sqrt{\frac{T\alpha}{d} + L^2 } \\
    &\geq  \mathbb{E}_{\theta'}[U_i(1) ] -  \left( \frac{2L^2T + 4\alpha T}{d} \right) \frac{\Delta}{2R} \sqrt{\frac{2T\alpha}{d}  }
\end{align*}

The last inequality follows by assuming $\alpha n \geq dL^2$, which holds because by assumption $d \geq L^2$ (recall $\alpha = L^2$).

We can then conclude that:

\begin{align*}
    \mathbb{E}_{\btheta}\left[ U_i(1) \right]  + \mathbb{E}_{\btheta'}\left[ U_i(-1) \right] &\geq \mathbb{E}_{\btheta'}\left[  U_i(1) + U_i(-1)  \right] -  \left( \frac{2L^2T + 4\alpha T}{d} \right) \frac{\Delta}{2R} \sqrt{\frac{2T\alpha}{d}  }\\
    &= 2\mathbb{E}_{\theta'}\left[ \frac{\tau_i L^2}{d} + \sum_{t=1}^{\tau_i} \x_{t, i}^2  \right] -  \left( \frac{2L^2n + 4\alpha T}{d} \right) \frac{\Delta}{2R} \sqrt{\frac{2T\alpha}{d}  }\\
    &\geq \min\left(\frac{2\alpha T}{d},\frac{2L^2 T}{d} \right) -   \left( \frac{2L^2T + 4\alpha T}{d} \right) \frac{\Delta}{2R} \sqrt{\frac{2T\alpha}{d}  }
\end{align*}

 Therefore using the Randomization Hammer we conclude there must exist a parameter $\theta \in \{ \pm \Delta \}^d$ such that $R_{\pi, (\cX, \btheta)}(T, 1)$ such that:
\begin{equation*}
     R_{\pi, (\cX, \btheta)}(T, 1) \geq d\frac{\Delta \sqrt{d}}{2L} \left( \min\left(\frac{2\alpha T}{d},\frac{2L^2 T}{d} \right) -   \left( \frac{2L^2T + 4\alpha T}{d} \right) \frac{\Delta}{2R} \sqrt{\frac{2T\alpha}{d}  } \right).
\end{equation*}
Let $\alpha = L^2$ and $\Delta = \frac{R\sqrt{d}}{3\sqrt{2T}L}$. In this case:
\begin{equation*}
    R_{\pi, (\cX, \btheta)}(T, 1) \geq \frac{Rd}{6\sqrt{2}}\sqrt{T}.
\end{equation*}
The additional constraint on $T$ comes from the fact we must have that $S \geq \norm{\thetastar}_2 = \sqrt{d} \Delta = \frac{Rd}{3 \sqrt{2} TL} \implies T \geq \frac{dR}{3 \sqrt{2}LS} = \frac{d}{3 \sqrt{2}\sqrt{\snr}}$.
\end{proof}



\subsection{Misspecification Lower Bound}

In this section record a scale-aware version of the lower bound for misspecified linear bandits from~\citep{lattimore2020learning}. For future reference, recall the definition of misspecification stated in~\cref{assump:param} specialized to a finite context set of size $m$: Let $\cX \subset \mathbb{R}^d$ be a finite context set of size $m$. A function $f: \cX \rightarrow \mathbb{R}$ is $\epsilon-$close to linear if there is a parameter $\thetastar$ such that for all $\x \in \cX$,
\begin{align*}
    \abs{f(\x)-\x^\top \thetastar} \leq \epsilon .
\end{align*}

Before stating our main result we will require the following supporting lemma from~\citep{lattimore2020learning},

\begin{lemma}\label{lemma::negative_result_gellert} For any $\epsilon, L > 0$, and $d \in [m]$ with $d \geq \lceil 8 \log(m)L^2/\epsilon^2 \rceil$, and an action set $\cX = \{ \x_i\}_{i=1}^m \subset \mathbb{R}^d$ satisfying $\| \x_i \|_2 = L$ for all $i$ and such that for all $i,j \in [m]$ with  $i \neq j$, $| \x_i^\top \x_j | \leq L^2 \sqrt{\frac{8\log(m)}{d-1}}$.  
\end{lemma}

\begin{proof}
This is simply a slightly modified version of  in \citet[Lemma 3.1]{lattimore2020learning}. 
\end{proof}

\begin{lemma}
\label{lem:misspec_lower_bound}
Let $\epsilon, L >0$ and $m \in \mathbb{N}$. For any $d \in [m]$ with $d \geq \lceil 8 \log(m)L^2/\epsilon^2 \rceil$ consider a finite action set $\cX \subset \mathbb{R}^d$ of size $m$ satisfying $\| \x_i \|_2 = L$ for all $i \in [m]$. Then for any policy $\pi$, there is a parameter $\thetastar$, $\epsilon$-close to a function $f : \cX \rightarrow \mathbb{R}^m$ for which: 
\begin{equation*}
     R_{\pi, (\cX, \thetastar, f)}(T, 1) \geq  \epsilon \sqrt{ \frac{d-1}{8\log(m)}}\frac{\min(T, m-1)}{4}  .
\end{equation*}
Moreover, this parameter $\thetastar$ satisfies $\norm{\thetastar}_2 = \frac{\epsilon}{L} \sqrt{\frac{d-1}{8 \log(m)}}$.

\end{lemma}

\begin{proof}
Observe that by assumption since  $d \geq \lceil 8 \log(m)L^2/\epsilon^2 \rceil$ we have that $\epsilon \geq L \sqrt{\frac{8\log(m)}{d} }$. By \cref{lemma::negative_result_gellert}, we may choose $\cX = \{ \x_i\}_{i=1}^m \subset \mathbb{R}^d $  satisfying:
\begin{enumerate}
    \item $\| \x_i \|_2 = L$ for all $i \in [m]$.
    \item $| \x_i^\top \x_j | \leq L^2 \sqrt{\frac{8\log(m)}{d-1}}$ for all $i \neq j$. 
\end{enumerate}

We now consider a family of $m$ bandit instances indexed by each of the arms $\x_i \in \cX$. For any $\x_i \in \cX$ define $\thetastar_i = \delta \x_i$ with $\delta = \frac{\epsilon}{L^2}\sqrt{ \frac{d-1}{8\log(m)}}$ and define $f_i$ as:
\begin{equation*}
    f_i(\x) = \begin{cases}
           L^2 \delta &\text{if } \x = \x_i\\
            0 &\text{o.w.}
        \end{cases}
\end{equation*}

By definition $f_i$ is $\epsilon-$close to linear since for all $\x \in \cX$ with $\x \neq \x_i$,
\begin{equation*}
    \langle \x, \thetastar_i  \rangle \leq \delta L^2 \sqrt{\frac{8\log(m)}{d-1}}  = \epsilon \quad\text{ and }\quad    \langle \x_i, \thetastar  \rangle = L^2 \delta 
\end{equation*}

Denote by $\x^{(t)}$ the action played by algorithm $\pi$ at time $t$ and define $\tau_i = \max\left\{ t \leq n: \x^{(s)} \neq \x_i \forall s \leq t \right\}$. Then $\mathbb{E}[ R_{\pi, (\cX, \thetastar, f)}(T, 1)  ]  \geq L^2\delta \mathbb{E}_i[\tau_i] $. Where $\mathbb{E}_i$ denotes the expectation under the law of bandit problem $\thetastar_i$ and algorithm $\pi$. Observe that any algorithm $\pi$ that queries arm $\x_j$ with $j \neq i$ more than once before pulling arm $\x_i$ will have a larger $\mathbb{E}_i[\tau_i]$ than one that only queries each arm once before querying arm $\x_i$. This means that in order to lower bound $\mathbb{E}_i[\tau_i]$ we can restrict ourselves to algorithms $\pi$ that do not repeat an arm pull before $\tau_i$. In fact we can assume algorithm $\pi$ behaves the same for all $i \in [m]$. Observe that for such algorithms whenever facing problem $\thetastar_i$ and $t \leq \tau_i$, the law of the rewards is independent of $\x_i$ for all $i \in [m]$. Let $f_0 : \cX \rightarrow \mathbb{R}$ denote the zero function such that $f(\x) =0$ for all $\x \in \cX$ and let $\mathbb{E}_0$ be the expectation under the law of the bandit problem induced by $f_0$ and algorithm $\pi$. Let $T_i$ be the first time that algorithm $A$ encountered arm $i$ when interacting with $f_0$. Observe that since the interactions of $\pi$ with $f_i$ before $\tau_i +1$ and the interactions of $\mathcal{A}$ with $f_0$ before $\mathcal{A}$ pulls $\x_i$ (or the time runs up) are indistinguishable $\mathbb{E}_{i}[\tau_i] = \mathbb{E}_{0}[\min(T, T_i-1)]$. A simple averaging argument shows that
\begin{align*}
   \frac{1}{m} \sum_{i}  R_{\pi, (\cX, \thetastar, f)}(T, 1)     \geq \frac{L^2\delta}{m} \sum_{i=1}^m \mathbb{E}_i[ \tau_i ] = \frac{L^2\delta}{m} \sum_{i=1}^m  \mathbb{E}_0\left[ \min(T,T_i-1)\right] = \frac{L^2\delta}{m} \mathbb{E}_0\left[\sum_{i=1}^m  \min(T,T_i-1)\right]
\end{align*}
Since $\pi$ is assumed to interact with $f_0$ by never pulling the same arm twice, $\{ T_i-1 \}_{i=1}^m= \{ i-1 \}_{i=1}^m$. Using this fact, we can write
\begin{equation*}
  \frac{1}{m} \sum_{i=1}^m  \min(T,T_i-1) = \frac{1}{m}\sum_{i=1}^{\min(m, T)} i-1 + \mathbf{1}(T \leq m-1) T.
\end{equation*}
In order to bound the expression above we analyze two cases, first when $T\leq m$, and second when $T > m$. In the first case 
\begin{equation*}
   \frac{1}{m} \sum_{i=1}^{\min(m, T)} i-1 + \mathbf{1}(T \leq m-1) T= \frac{1}{m}\left(\frac{T(T-1)}{2} + T(m-T). \right)
\end{equation*}

Let's consider two sub-cases. If $T-1 \geq \frac{m}{2}$, then $\frac{1}{m}\left(\frac{T(T-1)}{2} + T(m-T) \right)  \geq \frac{T}{4} $. If $T-1 < \frac{m}{2}$ then $m-T > m-\left(\frac{m}{2} +1 \right) = \frac{m}{2} -1$. And therefore $\frac{1}{m}\left(\frac{T(T-1)}{2} + T(m-T) \right)  \geq \frac{T}{2} - \frac{T}{m} >\frac{T}{2}  - 1 $. It follows that whenever $T \leq m$ (and $T > 1$), then $\frac{1}{m}\left(\frac{T(T-1)}{2} + T(m-T) \right)  \geq \frac{T}{4} $. 

Now let's consider the case when $T > m$. If this holds, 

\begin{equation*}
    \frac{1}{m} \sum_{i=1}^{\min(m, T)} i-1 + \mathbf{1}(T \leq m-1) T = \frac{m-1}{2}.
\end{equation*}
Assembling these facts together we conclude that in all cases,
\begin{equation*}
     \frac{1}{m} \sum_{i=1}^m  \min(T,T_i-1) \geq \frac{\min(T, m-1)}{4}.
\end{equation*}
The result follows by noting this implies there must exist one $\thetastar_i$ such that 
\begin{equation*}
    R_{\pi, (\cX, \thetastar, f)}(T, 1) \geq \epsilon \sqrt{ \frac{d-1}{8\log(m)}}\frac{\min(T, m-1)}{4}.
\end{equation*}
The result follows.
\end{proof}

\subsection{Adversarial Context Lower Bound}

The result follows similarly to the lower bound from \citet[Theorem 1]{han2020sequential}.
\begin{lemma}
\label{lem:adv_context_lower_bound}
For any parallel bandit algorithm policy $\pi$ there exists a sequence of (oblivious) adversarial contexts $\cX_{t,p}$ and parameter $\thetastar$ such that,
\begin{equation}
    R_{\pi, (\cX_{t,p}, \thetastar)} (T, P) \geq \Omega \left(LS P \min(\sqrt{d}, \frac{T}{\sqrt{d}}) \right).
\end{equation}
\end{lemma}
\begin{proof}
    First consider the case where $T \geq d/2$ and without loss of generality assume that $d' = \frac{d}{2}$ is an integer (the odd case is handled identically). First note that trivially there must exist $d'$ batch indices $\{i_1, \hdots, i_{d'} \}$ such that\footnote{again we use $a, b \in [T,P]$ to denote the lexigraphic ordering of time indices $t_{a, b}$}:
    \begin{align}
        \sum_{k=1}^{d'} t_{i_k, 1} - t_{i_{k-1}, 1} \geq P d'
    \end{align}
    since there are $P$ contexts/actions in each batch.

    Now we construct a (random) $\thetastar$ as follows. If $\epsilon_1, \hdots \epsilon_{d'} \in \{-1, +1\}$ are Rademacher r.v.s we set $\thetastar = (\btheta_1, \hdots, \btheta_d)$ with $\thetastar_{2k-1} = \frac{S}{\sqrt{d'}} \Ind[\epsilon_k = -1]$ and $\thetastar_{2k} = \frac{S}{\sqrt{d'}} \Ind[\epsilon_k = +1]$, for all $k \in [d']$. By construction $\norm{\thetastar}_{2} = S$. Now we construct contexts such that up to the $k$th batch the contexts reveal only information about the $\thetastar_{2k-1}$ and $\thetastar_{2k}$ coordinate of $\thetastar$. In particular if $t_{a, b}$ is contained in the $k$th batch ($t_{i_{k-1}, 1} \leq t_{a, b} \leq t_{i_k, 1})$, then let the context set at $t_{a,b}$ be $\{L e_{2k-1}, L e_{2k} \}$, otherwise set both actions to the zero vector $0$.
    
    The regret of the policy can now be lower bounded in each batch. Since in the $k$th round, before entering this batch, the learner has obtained no information regarding whether $\thetastar_{2k-1} = \frac{S}{\sqrt{d'}}$ and $\thetastar_{2k} = \frac{S}{\sqrt{d'}}$, the regret of the policy decomposes over batches. Now note within each batch an incorrect action incurs instantaneous regret $\frac{LS}{\sqrt{d'}}$. Since the policy in the $k$th batch must be fixed before entering the $k$th batch, averaged over the randomness in $\thetastar$ we have that, 
    \begin{align*}
        \mE_{\thetastar}[R_{\pi, (\cX_{t,p}, \thetastar)}(T, P)] \geq 
        \sum_{k=1}^{d'} t_{i_k, 1} - t_{i_{k-1}, 1} \cdot \frac{LS}{2 \sqrt{d'}} \geq \Omega(LSP \sqrt{d})
    \end{align*}
    Hence there must exist a parameter $\thetastar$ such that $R_{\pi, (\cX_{t,p}, \thetastar)}(T, P) \geq \Omega(LSP\sqrt{d})$.
    
    The case where $T < d/2$ can be handled similarly. In this case we choose $d' = T$. We immediately obtain that $\sum_{k=1}^{d'} t_{i_k, 1} - t_{i_{k-1}, 1} = TP$. Now using the same Rademacher construction but with scaling $\frac{L}{\sqrt{d}}$ (and setting coordinates beyond $d'$ to zero), we can conclude 
    \begin{align*}
        \mE_{\thetastar}[R_{\pi, (\cX_{t,p}, \thetastar)}(T, P)] \geq 
        \sum_{k=1}^{d'} t_{i_k, 1} - t_{i_{k-1}, 1} \cdot \frac{LS}{2 \sqrt{d}} \geq \Omega(LSTP/\sqrt{d})
    \end{align*}
    Hence there must exist a parameter $\thetastar$ such that $R_{\pi, (\cX_{t,p}, \thetastar)}(T, P) \geq \Omega(LSP\sqrt{d})$.
    Combining both cases gives the result.

Finally, we bound the minimax regret in setting applicable to both the case with a single global context and the oblivious adversarial setting. Let $\{Le_1, Le_2 \}$ be the action set (context set) during the first parallel deployment. Consider two plausible (uniformly random) settings for $\thetastar \in \{Se_1, Se_2 \} $. Any incorrect action choice incurs in a regret of $LS$. Since before entering this batch, the learner has no information regarding whether $\thetastar = Se_1$ or $\thetastar  = Se_2$, averaging over the randomness in the selection of $\thetastar$ we have that,
\begin{align*}
        \mE_{\thetastar}[R_{\pi, (\cX_{t,p}, \thetastar)}(T, P)] \geq \frac{LSP}{2} \geq \Omega(LSP)
    \end{align*}
which shows there exists a $\thetastar$ for which the lower bound holds as well.
Combining these three lower bounds the result follows.
    
\end{proof}

%% file: app_experiments.tex
\section{Experiment Details}
In this section we provide all of the experimental details for training the bandit algorithms.
\label{sec:app_exp}
\subsection{Hyperparameters}
The synthetic hyperparameters were fixed to the theoretical values. For the randomly initialized neural network experiments, the grid for the regularizer $\lambda$, norm bound $S$, and noise subgaussianity $R$ were: $\{0.01, 0.1, 1.0, 10.0, 100.0\}$ which was selected post-hoc for each experiment. For the superconductor experiments, the grid for all 3 parameters were: $\{0.1, 1.0, 10.0\}$. For the transcription factor binding dataset, the parameters grid was $\lambda = \{1.0, 10.0\}$, $R = \{0.01, 0.1, 1.0, 10.0, 100.0, 1000.0\}$, and $S = \{0.01, 0.1, 1.0, 10.0, 100.0, 1000.0\}$. For $\epsilon$-greedy, the parameter grid was set over $\epsilon = \{0.01, 0.02, \ldots , 0.99\}$ across all relevant experiments.

\begin{figure}[!hbt]
\centering
\begin{minipage}[c]{.31\linewidth}
\includegraphics[width=\linewidth]{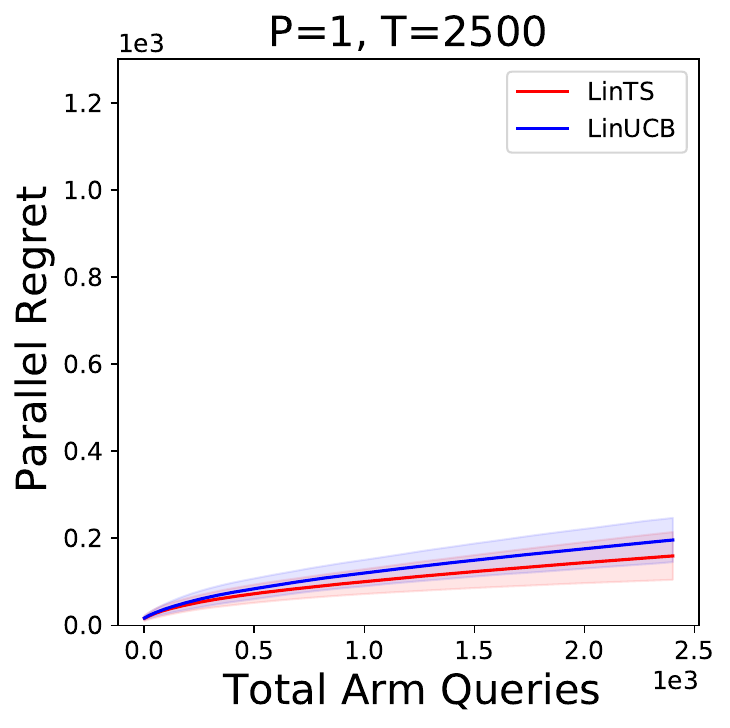}
\end{minipage}
\begin{minipage}[c]{.31\linewidth}
\includegraphics[width=\linewidth]{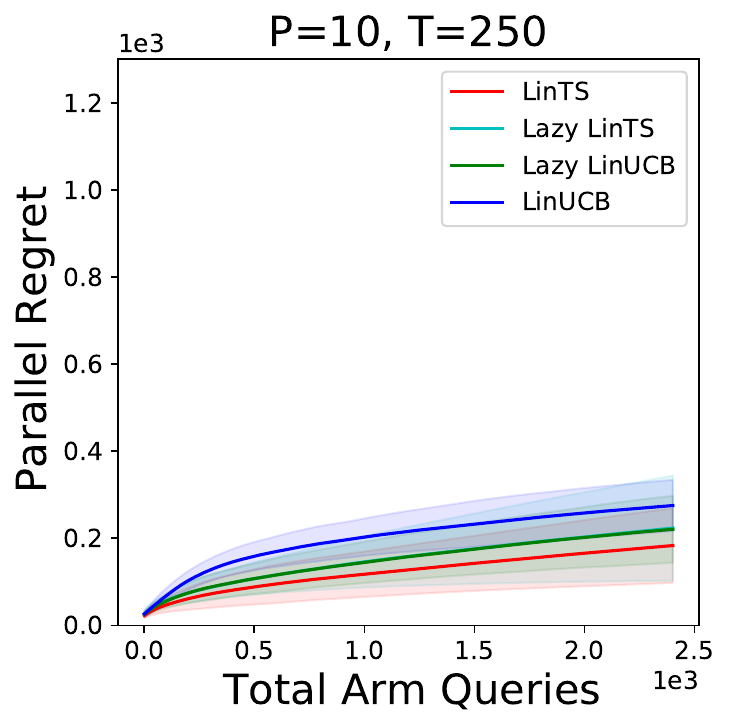}
\end{minipage}
\begin{minipage}[c]{.31\linewidth}
\includegraphics[width=\linewidth]{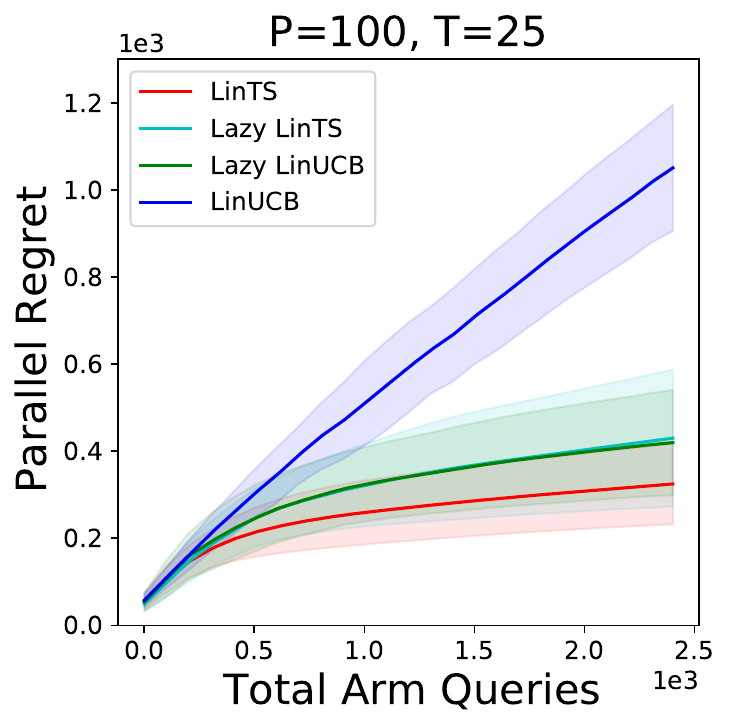}
\end{minipage}
\caption{TFBinding parallel regret with linear features. From left to right: $P=1$, $P=10$, and $P=100$.
}
\label{fig:tfbindinglinear}
\end{figure}

\subsection{Feature Engineering for Random Neural Network}
Two feature sets were considered:
\begin{itemize}
    \item Linear features $x_i$ was encoded as a 14-length feature vector
    \item Linear + Quadratic features where interaction terms of $x_i$ were included in $105$ quadratic features alongside the original $14$ linear features into a $119$ features.
\end{itemize}
The resulting parallel regret plots are shown in \cref{fig:1drandomnn}.

\begin{figure}[!hbt]
\centering
\begin{minipage}[c]{.31\linewidth}
\includegraphics[width=\linewidth]{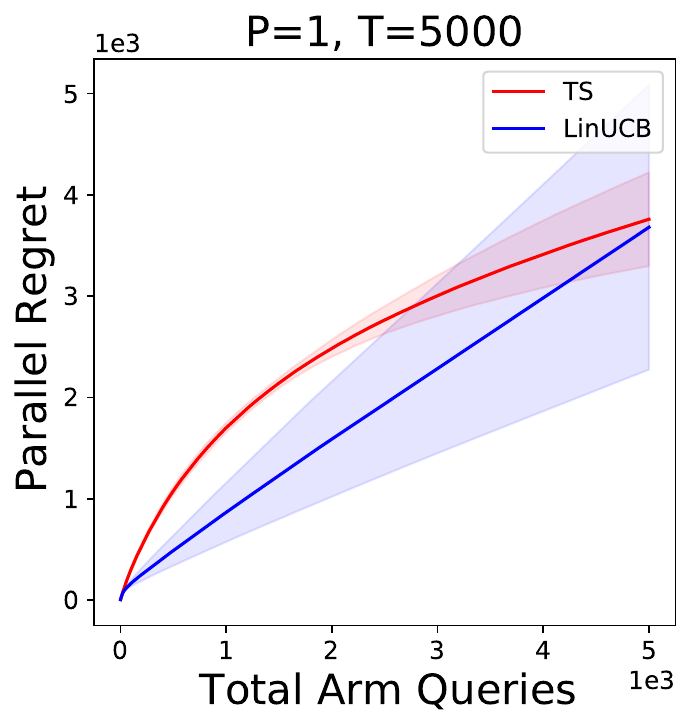}
\end{minipage}
\begin{minipage}[c]{.31\linewidth}
\includegraphics[width=\linewidth]{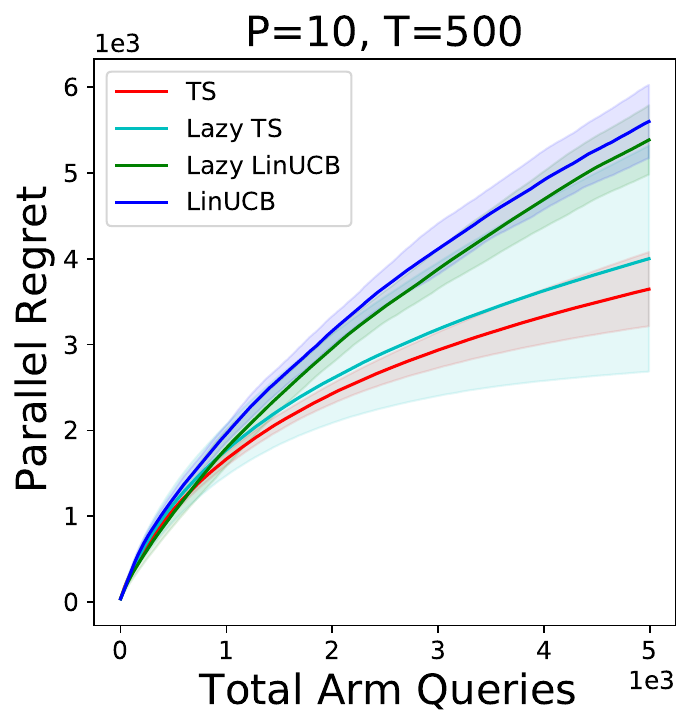}
\end{minipage}
\begin{minipage}[c]{.31\linewidth}
\includegraphics[width=\linewidth]{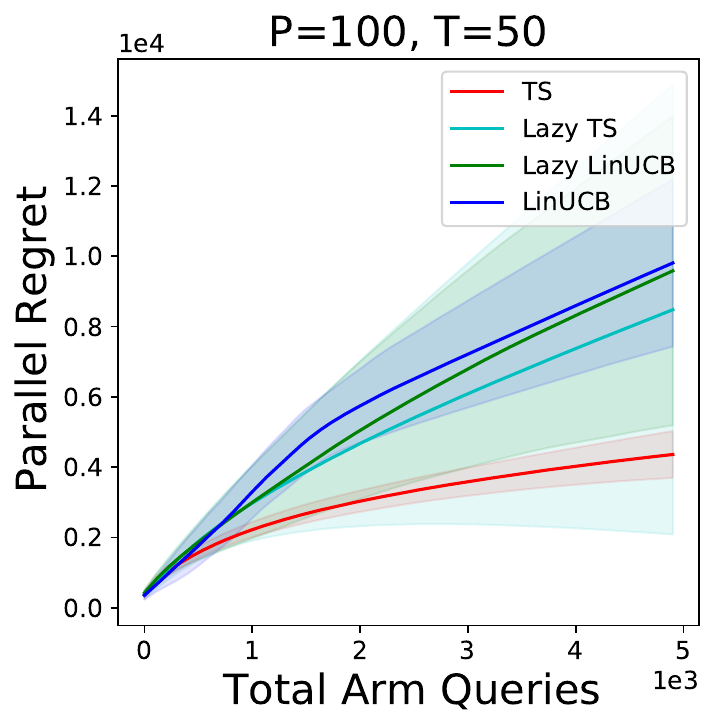}
\end{minipage}
\caption{RandomNN with Linear features. From left to right: $P=1$, $P=10$, and $P=100$.
}
\label{fig:1drandomnn}
\end{figure}

\subsection{Feature Engineering for Transcription Factor Binding}
Three feature sets were considered:
\begin{itemize}
    \item Linear features $\x_i$ was one-hot encoded into a 32-length feature vector.
    \item 250 random ReLU features where a $250 \times 32$ random matrix $\W$ is sampled such that $\W_{ij} \sim \mathrm{Normal}(0,1)$. Then, the feature map was evaluated as:
    $$\phi(\x_i) = \frac{1}{\sqrt{250}} \text{ReLU}\Bigg(\frac{\W \x_i}{\sqrt{32}}\Bigg)$$
    \item Linear + Quadratic features where $\x_i$ was one-hot encode and all 32 linear features and 528 quadratic features were combined into a 560-length vector.
\end{itemize}

The off-line test $R^2$ (without added noise in the $y_i$) on a train set of $TP = 2500$ matching the number of total arm queries yields values $0.15, 0.26, 0.29$ for linear, ReLU, and quadratic respectively. This matches with the number of features and level of expressivity of the model class. However, as shown in \cref{fig:relu250,fig:quadratic} the linear features perform the best followed by the ReLU features, and then the quadratic features. One can gain further insight by examining the off-line test $R^2$ for a smaller training size and see that the larger feature expansions accrue variance making the linear features perform the best. This matches our understanding that model fitting is less statistically efficient in an online setting.

\begin{figure}[!hbt]
\centering
\begin{minipage}[c]{.31\linewidth}
\includegraphics[width=\linewidth]{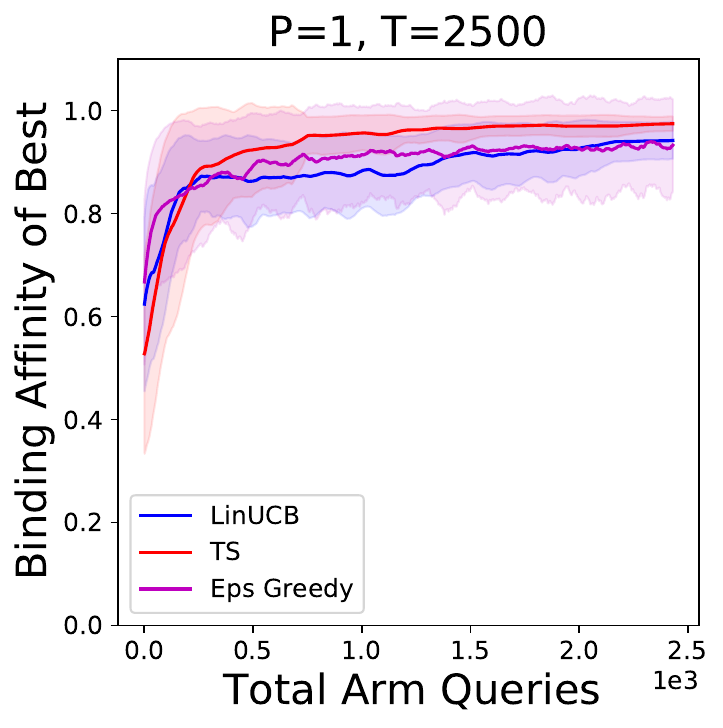}
\end{minipage}
\begin{minipage}[c]{.31\linewidth}
\includegraphics[width=\linewidth]{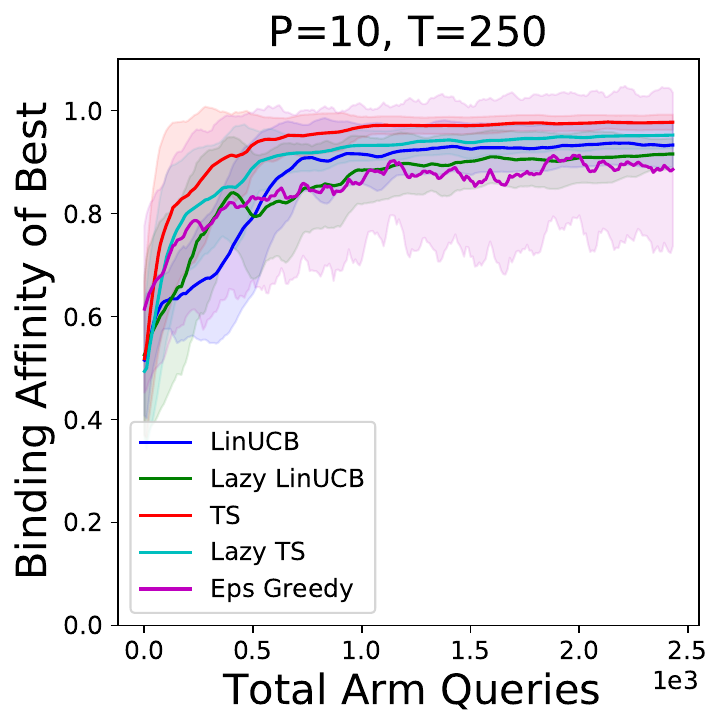}
\end{minipage}
\begin{minipage}[c]{.31\linewidth}
\includegraphics[width=\linewidth]{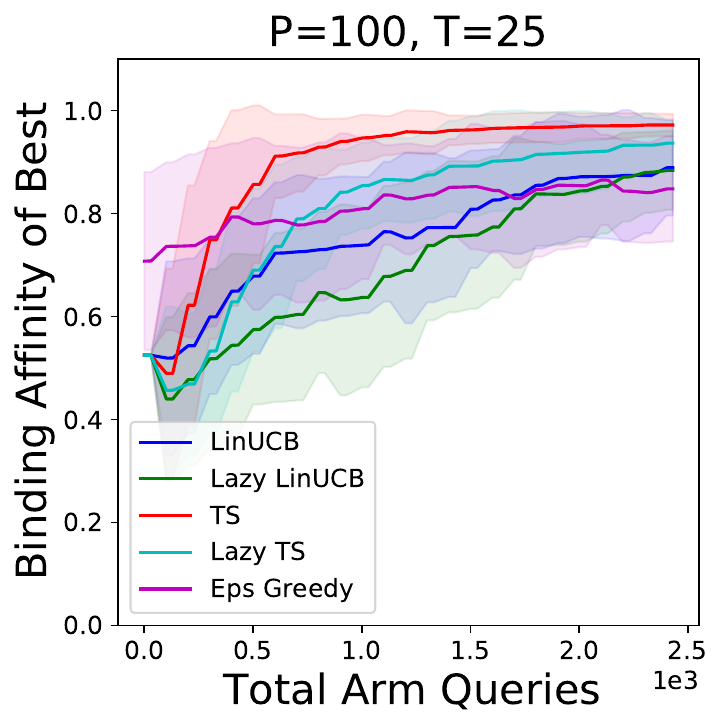}
\end{minipage}
\caption{TFBinding best arm with ReLU features. From left to right: $P=1$, $P=10$, and $P=100$.
}
\label{fig:relu250}
\end{figure}

\begin{figure}[!hbt]
\centering
\begin{minipage}[c]{.31\linewidth}
\includegraphics[width=\linewidth]{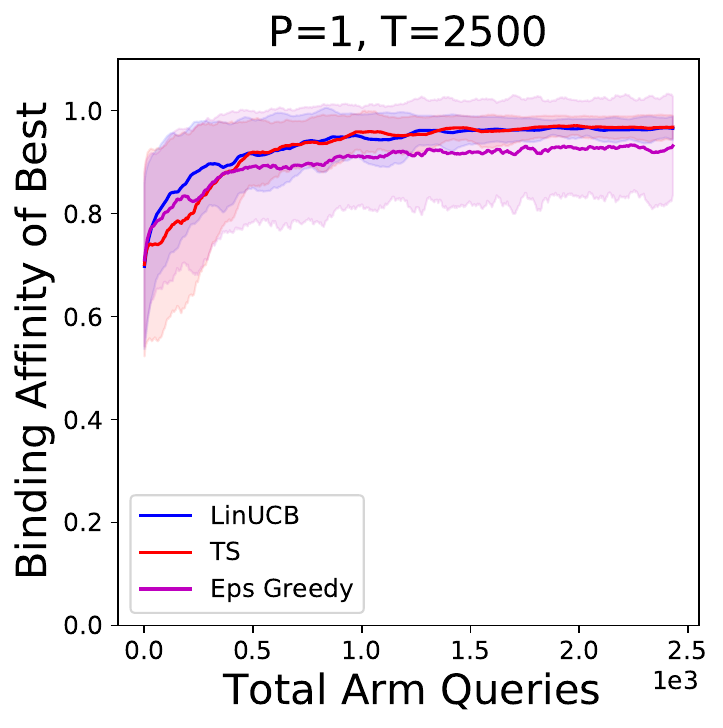}
\end{minipage}
\begin{minipage}[c]{.31\linewidth}
\includegraphics[width=\linewidth]{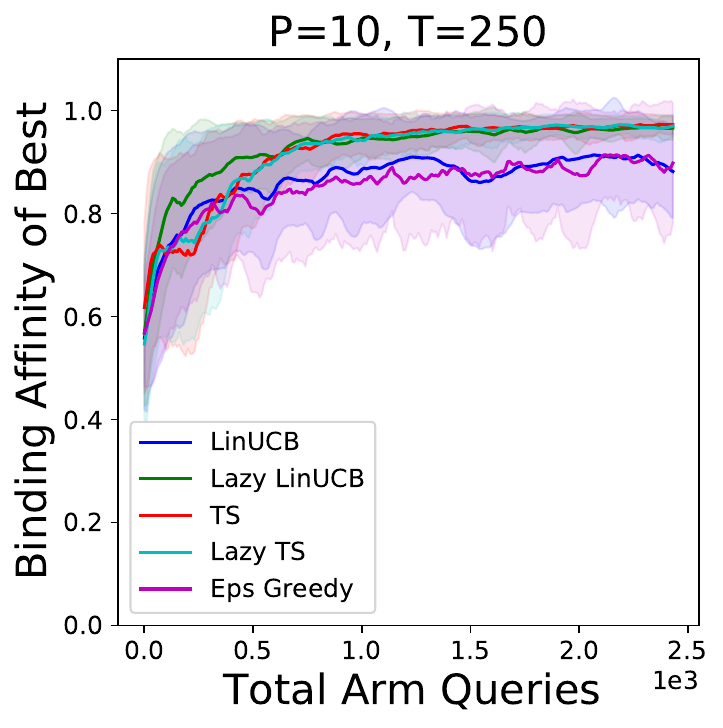}
\end{minipage}
\begin{minipage}[c]{.31\linewidth}
\includegraphics[width=\linewidth]{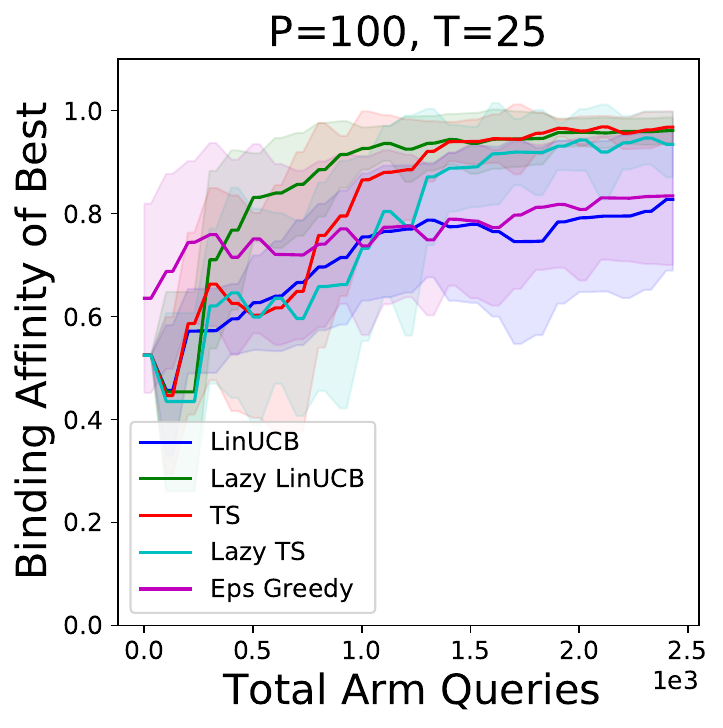}
\end{minipage}
\caption{TFBinding best arm with quadratic features. From left to right: $P=1$, $P=10$, and $P=100$.
}
\label{fig:quadratic}
\end{figure}
